\documentclass[11pt]{article}
\usepackage{amsmath,amsthm,amsfonts,amssymb,mathrsfs,bm}
\usepackage[usenames]{color}
\usepackage{hyperref}
\usepackage{amssymb}
\usepackage{dsfont} 
\usepackage{graphicx}
\usepackage{subcaption}
\usepackage{float}
\usepackage{hyperref}

\usepackage[a4paper,margin=1.15in]{geometry}

\newtheorem{theorem}{Theorem}[section]
\newtheorem{lemma}[theorem]{Lemma}
\newtheorem{corollary}[theorem]{Corollary}
\newtheorem{proposition}[theorem]{Proposition}

\newtheorem{remark}{Remark}[section]

\newtheorem{definition}{Definition}

\renewcommand{\d}{\mathrm{d}}

\newcommand{\R}{\mathbb{R}}
\newcommand{\C}{\mathbb{C}}

\newcommand{\N}{\mathbb{N}}
\newcommand{\E}{\mathbb{E}}

\newcommand{\rmP}{\mathrm{P}}
\newcommand{\K}{\mathrm{K}}
\newcommand{\M}{\mathrm{M}}

\newcommand{\I}{\mathrm{I}}

\newcommand{\X}{\mathcal{X}}

\newcommand{\var}{\mathrm{Var}}

\newcommand{\proj}{\mathrm{Proj}}

\renewcommand{\L}{\mathrm{L}}

\newcommand{\vol}{\mathrm{vol}}

\newcommand{\err}{E_N}

\newcommand{\range}{\mathrm{Range}}

\newcommand{\iid}{\mathrm{i.i.d.}}
\newcommand{\lip}{\mathrm{Lip}}
\newcommand{\diag}{\mathrm{diag}}
\newcommand{\knn}{\mathrm{NN}}

\renewcommand{\div}{\text{div}}
\newcommand{\eps}{\varepsilon}

\DeclareMathOperator{\Var}{Var}

\DeclareMathOperator{\dist}{dist}

\DeclareMathOperator{\Lip}{Lip}
\DeclareMathOperator{\Span}{span}

\definecolor{mygreen}{rgb}{0.1,0.75,0.2}

\makeatletter
\renewcommand*{\@cite@ofmt}{\hbox}
\makeatother

\makeatletter
\renewcommand*{\@cite@ofmt}{\hbox}
\makeatother

\begin{document}
\title{Determinantal sampling on general spaces:\\ fast approximation via diffusion geometry}
\author{
	\small\begin{tabular}{c}
		{Hoang-Son Tran  \thanks{Corresponding author}}\\  
        Department of Mathematics\\
		 National University of Singapore\\ transonsp97@gmail.com
	\end{tabular}
    \and
  \small\begin{tabular}{c}
		{Pranav Gupta}
		 \\  Division of Applied Mathematics\\
		 Brown University\\
         pranav\_gupta@brown.edu
	\end{tabular}
    \and
    \small\begin{tabular}{c}
		{Subhro Ghosh}
		 \\  Department of Mathematics\\
		 National University of Singapore\\
         subhrowork@gmail.com
	\end{tabular}
}
\date{}
\maketitle

\begin{abstract}
Determinantal point processes (DPPs) have recently emerged as a kernel-based alternative to standard independent sampling for constructing efficient minibatches, coresets, and other compact representations of large-scale datasets. In particular, sampling mechanisms based on DPPs are believed to possess better approximation properties compared to classical i.i.d. samplers, even at the scale of the decay exponent on the sample size that characterises the rate. One of the key strengths of DPP based samplers is that they can be deployed over very general spaces, in contrast to more classical sampling methods beyond i.i.d. which tend to be tailored to particular well-structured settings, principally Euclidean spaces.  In this work, we establish explicit rate guarantees for determinantal sampling in spaces that extend far beyond known Euclidean setups, focussing on spectral kernels obtained from eigenspaces of naturally associated Markov diffusion operators, including in particular the Laplacian on manifolds and nearest-neighbour type graphs on data point clouds and the discrete Schrodinger operator on large scale networks. This covers, in particular, the settings of Riemannian manifolds and discrete networks, as well as the problem of batch sampling from discrete datasets. 

In determinantal sampling from compact Riemannian manifolds, we establish sampling rates that automatically pick up the intrinsic dimensionality $d_{\text{int}}$ of the underlying manifold (as opposed to the ambient dimension of the data, which is usually much larger). In the setting of networks, we investigate DPP-based samplers on the celebrated $k$-nearest neighbour graphs, random geometric graphs and large scale weighted networks (without any underpinning manifold structure), and demonstrate a similar improved dependence on the intrinsic dimensionality of the data.  Overall, our approach achieves guarantees comparable to $\big(\text{sample size}\big)^{-\frac{1}{2}-\frac{1}{2d_{\text{int}}}}$ , which match with the rates for well-known determinantal samplers on Euclidean spaces of comparable dimension. In terms of techniques, we connect to the celebrated Weyl's Law for spectra on general spaces, and leverage tools from the theory of Markov diffusion operators and Dirichlet forms, which could be of independent interest in this area. Our approach is applicable to a very wide plethora of ambient spaces beyond Euclidean as soon as these ingredients are available, including e.g.~fractals, leading to guarantees that will be governed by the intrinsic structural constants of such spaces.
\end{abstract}

\newpage

\tableofcontents

\section{Introduction}
Sampling is a fundamental ingredient in modern machine learning with large-scale
models and massive datasets. For instance, stochastic gradient descent (SGD) uses a small random subset of the data, called a \emph{minibatch}, to approximate the gradient of an empirical loss at each update, avoiding the cost of evaluating the full gradient~\cite{bottou2010large,bottou2018optimization}.
In a related vein, \emph{coreset} constructions seek small, weighted representations of a dataset that approximately preserve the objective functions of interest~\cite{feldman2011unified,munteanu2018coresets,TrBaAm19}. 
Both procedures raise a common question: how can we draw a small sample that accurately represents the full dataset for the task of interest? The choice of sampling distribution, and its ability to exploit the structure of the data, are central to this question.

To fix ideas, let $\mathcal X=\{x_1,\ldots,x_N\}$ be a dataset, and $f: \mathcal{X} \to \mathbb{R}$ be a test function. Consider the approximation of the empirical average
\[
    \mathcal L_N(f):=\frac{1}{N}\sum_{i=1}^N f(x_i)
\]
using only $n\ll N$ evaluations of $f$. A common approach is to sample indices $I_1,\ldots,I_n$ \emph{independently} and uniformly from $\{1,\ldots,N\}$, which yields the unbiased estimator $\widehat{\mathcal{L}}_n(f)=n^{-1}\sum_{j=1}^n f(x_{I_j})$. Conditional on the dataset, the corresponding root mean square error (RMSE) is of order $n^{-1/2}$ (see \cite{robert2004monte}). Uniform sampling without replacement has the same $n^{-1/2}$ scale for RMSE when $n$ is small relative to $N$. A central question is whether \emph{structured, dependent} sampling methods (that are yet relatively simple with manageable computational costs) can improve on this benchmark $n^{-1/2}$, and how the achievable improvement depends on the regularity of the objective and the intrinsic geometry of the dataset.

The significance of this question is particularly apparent when processing each selected data point is expensive: computing a gradient contribution in a large neural network entails differentiation through the network~\cite{rumelhart1986backprop,baydin2018autodiff}, while computing likelihood gradients in conditional random fields can require inference in a large graphical model~\cite{lafferty2001crf,sutton2012crf}.
In such situations, even a modest reduction in the minibatch size needed to achieve a prescribed estimation accuracy can yield considerable computational savings. This motivates devoting additional effort to sample selection, particularly when the resulting sampling distribution can be reused across many iterations of a learning algorithm. This trade-off between sampling overhead and variance reduction has been studied, for example, in importance-sampling methods for deep learning~\cite{katharopoulos2018samples}.

\subsection{Determinantal sampling: background and generalities }

\paragraph{Determinantal Probability Measures.}
Determinantal point processes (DPPs) provide a kernel-based approach to this problem. Introduced by Macchi~\cite{Mac75} as a stochastic model, and subsequently developed as tractable models of diversity in machine learning~\cite{KuTa12}, DPPs describe random subsets of a space through determinants of a kernel. Heuristically, in this setting, each sample point is associated with a feature vector, and DPPs favour samples whose feature vectors span a larger volume. This provides a natural mechanism for reducing redundancy among the selected data points. The determinantal structure also brings a substantial toolbox from linear algebra and operator theory to bear on questions of sampling and fluctuation~\cite{HKPV06}.

The main recipe for defining a DPP is to construct its kernel.
For simplicity, we focus on projection kernels, which define projection DPPs. Let $E$ be a Polish space equipped with a reference measure $\mu$, and let $\phi_1,\ldots,\phi_n$ be orthonormal functions in $L^2(\mu)$. We define the kernel
\begin{equation}
\K(x,y):=\sum_{j=1}^n\phi_j(x)\overline{\phi_j(y)}.
    \label{eq:intro-projection-kernel}
\end{equation}
Sample an ordered tuple $(X_1,\ldots,X_n)\in E^n$ according to the probability density
\[
    p(x_1,\ldots,x_n)
    :=\frac{1}{n!}\det\!\left[\K(x_i,x_j)\right]_{i,j=1}^{n}
\]
with respect to the product measure $\mu^{\otimes n}$. The resulting random configuration $\mathcal S_n:=\{X_1,\ldots,X_n\}$ is called the projection DPP on $E$ with kernel $\K$ and reference measure $\mu$. For a more detailed discussion, we refer the reader to \cite{Sos00,HKPV06}. 

\paragraph{Generality of Scope.}
A key strength of determinantal samplers is their generality of scope. Indeed, following the celebrated Macchi-Soshnikov criterion \cite{Mac75,Sos00}, determinantal probability measures can be defined on any Polish space as soon as the kernel $\K$ and the background measure $\mu$ satisfy the analytical condition that $\K$ is a non-negative contraction operator on $L^2(\mu)$. Thus, in principle, determinantal sampling can be undertaken on very general spaces with rather minimal structural requirements, giving it the potential to be a broadly applicable sampling mechanism. This is in contrast to most alternatives that aim to improve upon iid rates, such as sophisticated quadrature methods, which are typically tailored quite specifically to the background space in question.

In general, therefore, DPP sampling can be viewed as a kernelised approach to provide improved sampling accuracy on very general  spaces. However, actual improvement in rates beyond iid is understood only in a rather limited setting, focusing on Euclidean spaces \cite{dpp-sgd, bardenet2024small, tran2026state}. Other available results leverage rather specialised geometric structures such as compact complex or prequantized
symplectic manifolds \cite{lemoine2024monte,kordyukov2026monte}; while these are of mathematical interest, they may be of limited applicability to machine learning problems. 

Our principal goal in this work is to develop concrete, explicit approximation guarantees beyond iid for DPP samplers in very general settings beyond Euclidean spaces with the potential to address sampling scenarios in a broad class of  machine learning applications. We are able to achieve this aim in a wide array of background spaces, including manifold geometries, data point clouds, large-scale discrete networks (without any underpinning manifold structure). In fact, our approach is applicable to a very wide plethora of ambient spaces beyond Euclidean as soon as the fundamental ingredients such as Dirichlet forms and Weyl's Law are available, including e.g. fractals, leading to guarantees that will be governed by the intrinsic structural constants of such spaces. 

\paragraph{Sampling a DPP.}
The algebraic formulation of DPPs also offers computational advantages. DPP sampling algorithms rely on basic operations such as matrix multiplication and orthogonalization, which can be implemented efficiently using optimized numerical linear algebra libraries. In particular, modern GPUs accelerate matrix multiplication by processing matrix blocks in parallel and using specialized tensor units for multiply-and-accumulate operations~\cite{nvidia2023matrix}. 
For repeated sampling from a fixed DPP, it is useful to distinguish the one-time preprocessing cost, including kernel construction and factorization, from the cost of generating each subsequent sample. For a rank-$n$ projection DPP on a ground set of $N$ points, the kernel can be represented by a feature matrix with orthonormal columns. Once this feature matrix is available, basic exact sampling algorithms have a per-sample cost of $O(Nn^2)$~\cite{tremblay2018optimized}. More recently, the exact rejection sampler of Barthelm\'e et al.~\cite{barthelme2023faster} has an expected cost of $O(n^3\log n)$ for each subsequent sample (and an overhead of $O(Nn+n^3\log n)$ for the first sample). 
The tree-based method of Gillenwater et al. \cite{gillenwater2019tree} also permits exact repeated sampling, with a per-sample cost of \(O(n^4\log N)\) after \(O(Nn^2)\) preprocessing.
For fixed-size DPPs, Anari et al.~\cite{anari2025optimal} give approximate repeated sampling with a cost controlled by the output size after an initial preprocessing stage. These developments make determinantal sampling computationally attractive when evaluating the selected items is expensive relative to the sampling overhead; an example being gradient evaluation on a large scale neural network.

\paragraph{Improvement over independent sampling.}
Structured dependent sampling can improve on the independent sampling rate for suitable classes of integrands.
For DPP samplers, such gains depend on selecting kernels that lead to favourable fluctuation properties. In \emph{continuous} settings, asymptotic kernel analysis provides a way to establish these gains. For example, multivariate orthogonal polynomial ensembles yield an RMSE of order $n^{-1/2-1/(2d)}$ for Monte Carlo quadrature in dimension $d$, under appropriate regularity assumptions  \cite{bardenet2020monte}. Related fluctuation estimates are available for harmonic ensembles on spheres and other compact two-point homogeneous manifolds ~\cite{levi2024linear,beltran2016energy,garcia2026}. Additional geometric structure can yield faster rates: on compact complex or prequantized compact symplectic manifolds, suitably constructed DPPs achieve an RMSE of order $n^{-1/2-1/d}$ under appropriate geometric and regularity assumptions, where $d$ denotes the real dimension~\cite{lemoine2024monte,kordyukov2026monte}. Wavelet-based kernels provide a complementary approach to variance reduction that accommodates objectives with low regularity~\cite{tran2026state}.
The framework in~\cite{tran2026state} provides a systematic scheme for discretizing continuous kernels on finite datasets while preserving the corresponding RMSE rates under suitable assumptions. Discretized orthogonal polynomial and wavelet kernels have been used to reduce the variance of SGD minibatch estimators~\cite{dpp-sgd,tran2026state} and to construct weighted coresets with statistical guarantees~\cite{TrBaAm19,bardenet2024small,jaquard2026}.


\subsection{Our contributions: point clouds, manifolds, and networks}

\paragraph{Spectral DPPs on general spaces.}
In this work, we investigate \emph{spectral DPPs}, a family of DPPs constructed from eigenfunctions of natural operators on an underlying space, and establish approximation rates that reflect its diffusion geometry. Our framework encompasses diffusion generators on Riemannian manifolds, graph Laplacians on point clouds in $\mathbb{R}^d$, discrete Schr\"odinger operators on networks, and Markov semigroup generators on more general spaces, such as fractals. We develop a general theory of explicit approximation rates across these settings, providing quantitative support for the idea that DPP sampling can improve on i.i.d. sampling.
While the particular case of Laplacian eigenfunctions has been considered before (see, eg, \cite{katori2022, tremblay2017graphsamplingdeterminantalprocesses, jaquard2026} and the references therein), the broader class of operators associated with symmetric Dirichlet forms, including in particular the discrete Schrodinger operators on large scale networks, are to our knowledge novel kernel categories in the setting of DPPs.  The fact that this improvement in approximation guarantees is driven by a very general, spectral construction further underscores the wide scope of effectiveness for DPP samplers, which is a key objective of this work. 

On manifolds, we use weighted Laplace--Beltrami operators; on point clouds, graph Laplacians built on the observed data; and on networks, graph Schr\"odinger operators with confining potentials. We analyse these constructions within a common framework of nonnegative self-adjoint operators associated with symmetric Dirichlet forms. These operators encode local variation, and their low-frequency eigenspaces provide features for constructing samplers adapted to the underlying geometry.

More precisely, let $\mu$ be a reference measure on the underlying space. We define \emph{spectral DPPs associated to an operator $\L$} on $L^2(\mu)$ by choosing the orthonormal functions $\{\phi_j\}$ in \eqref{eq:intro-projection-kernel} to be a real-valued orthonormal eigenbasis of $\L$ in $L^2(\mu)$, ordered by nondecreasing eigenvalue. In particular, the rank-$n$ projection kernel is given by
\begin{equation}
    \K_n(x,y)=\sum_{j=1}^n\phi_j(x)\phi_j(y).
    \label{eq:intro-spectral-kernel}
\end{equation}
Let $\mathcal{S}_n$ be the associated DPP defined by the kernel $\K_n$ and reference measure $\mu$.
We call $\mathcal{S}_n$ the \emph{spectral DPP of rank $n$ associated to $\L$}.

For a test function $f$, we write 
\begin{equation} \label{eq:intro-linear-statistic}
    \mathcal{S}_n(f):=\sum_{x\in\mathcal{S}_n}f(x)
\end{equation} 
for its linear statistic. 
Linear statistics are central to our study because they are used to construct DPP-based estimators in applications. We focus on bounding their variances, since they lead to controls on the mean squared error of the estimators.
Our main result, Theorem \ref{thm:var-reduce-general}, bounds the variance of linear statistics in terms of the local spectral growth of $\L$, providing a unified treatment of the continuous and discrete models considered below.

\paragraph{Data subsampling and intrinsic dimension.}
An important motivation comes from the manifold hypothesis: high-dimensional observations often possess a much smaller number of effective degrees of freedom, modelled by a distribution concentrated on or near a low-dimensional manifold~\cite{fefferman2016testing}. For instance, images generated by varying the pose of a fixed object may lie in a space with many pixel coordinates while their variability is described by a small number of parameters. The relevant dimension for sampling should then reflect this underlying structure. 

We investigate the setting in which the data-generating distribution is supported on the manifold itself.
In particular, we consider a point cloud $\mathcal X=\{x_1,\ldots,x_N\}\subset\mathbb R^d$ drawn independently from a smooth, strictly positive density on a connected, compact embedded Riemannian manifold $(\mathcal M,g)$ without boundary, of intrinsic dimension $m\geq 2$. From the observed data points, we construct weighted $\varepsilon$-graphs and symmetric $k$-nearest-neighbour graphs using compactly supported distance weights, and form DPPs from the first $n$ eigenvectors of their graph Laplacians. Once the graph parameters are chosen, the construction uses the point cloud without requiring further knowledge of the latent manifold or its eigenfunctions. This is important in practice since the latent manifold and its differential operators are typically unknown. 

We obtain explicit variance bounds that separate the intrinsic sampling rate from the error in approximating the continuum geometry (Theorems \ref{thm:var-reduce-eps} and \ref{thm:var-reduce-knn}). Suitable choices of graph parameters yield, with high probability over the point cloud,
\[
    \frac{1}{n^2}\var[\mathcal S_{n}(f)|\mathcal X]
    \lesssim
    \bigl(\|f\|_{L^\infty(\mathcal X)}^2+
          \operatorname{Lip}_{\mathcal M}(f)^2\bigr)
    \cdot n^{-1-1/m},
\]
simultaneously for Lipschitz $f$ and an explicit range of $n$ growing with $N$ (Corollaries~\ref{cor:epsilon-concrete-scaling} and~\ref{cor:knn-concrete-scaling}). The exponent is governed by $m$, even when the observed data lie in a much higher-dimensional space $\R^d$. The implicit constants retain dependence on the fixed geometry, embedding, and density.

As an application, we introduce an unbiased estimator to approximate the empirical average $\mathcal L_N(f)$,
\[
    \widehat{\mathcal{L}}_{n}(f)
    :=\sum_{x\in \mathcal S_{n}}\frac{f(x)}{\K_{n}(x,x)} \cdot 
\]
Here $\K_{n}$ is the corresponding kernel with respect to the empirical probability measure. Under a fair assumption on the lower bound of the diagonal of $\K_n$, the corresponding RMSE is of the order $O(n^{-1/2-1/(2m)})$ over an explicit range of minibatch sizes $n$ (Theorems~\ref{thm:eps-minibatch} and~\ref{thm:knn-minibatch}), with high probability over the point cloud. 

The recent work of Jaquard and Keriven~\cite{jaquard2026} develops a general theory for transferring statistical properties between discrete and continuous DPPs, including Laplacian DPPs on graphs constructed from data sampled on unknown manifolds. However, they are able to obtain a general relationship between the two settings, while explicit rate improvements for these constructions have remained unresolved. 

Our results provide concrete, explicit improvements in the sample-size exponent for the graph-Laplacian samplers above, with quantitative graph approximation errors and simultaneous control over growing ranges of sample sizes. To our knowledge, these are the first such explicit rate guarantees for both $\varepsilon$-graph and symmetric $k$-nearest-neighbour spectral DPPs on manifold-supported data. They provide a basis for variance reduction in minibatch estimation and for further applications to weighted data summaries.

\paragraph{Quadratures on manifolds and spectral geometry.}
Let $(\mathcal M,g)$ be a compact, connected Riemannian manifold without boundary, of dimension $m\geq 2$. We consider weighted Laplacians on $\mathcal M$, allowing reference measures with strictly positive $C^{1,1}$ densities with respect to Riemannian volume. This class includes the standard Laplace--Beltrami operator and diffusion generators with nonconstant invariant densities.
For spectral DPPs associated to these operators, our Theorem~\ref{thm:var-reduce-manifold} establishes
\begin{equation}
    \frac{1}{n^2}\var[\mathcal S_n(f)]
    \lesssim \bigl(\|f\|_\infty^2+\operatorname{Lip}(f)^2\bigr)
    \cdot n^{-1-1/m}
    \label{eq:intro-manifold-variance}
\end{equation}
for Lipschitz observables $f$. 

As an application, we construct a quadrature rule for integration against a prescribed weight $\omega\in C^1(\mathcal M)$, using the Laplace--Beltrami spectral DPP and inverse-intensity weights. For a Lipschitz function $f:\mathcal M\to\mathbb R$, we consider
\[
    \widehat{\mathcal L}_n(f)
    :=\sum_{x\in S_n}\frac{\omega(x)f(x)}{\K_n(x,x)}.
\]
Our Theorem~\ref{thm:monte-carlo-manifold} shows this estimator is unbiased for $\mathcal L(f):=\int_{\mathcal M}f\omega\d\vol_g$ and satisfies
\begin{equation}
    \operatorname{RMSE}\!\left(\widehat{\mathcal L}_n(f)\right)
    =O\!\left(n^{-1/2-1/(2m)}\right).
    \label{eq:intro-rmse}
\end{equation}
This RMSE matches the orthogonal-polynomial quadrature rate of~\cite{bardenet2020monte}. However, the OPE-based quadrature guarantees in that work assume that the target measure has a sufficiently regular density with respect to Lebesgue measure on a Euclidean domain. Our construction instead works intrinsically with the target measure $\omega\,\d\vol_g$, which is singular with respect to ambient Lebesgue measure when $\mathcal M$ is embedded in a higher-dimensional Euclidean space. Since the sampling distribution depends on neither $f$ nor $\omega$, the same nodes can be reused for multiple integration tasks. 

\paragraph{Large  scale discrete networks.}
For discrete networks (without any underpinning manifold structure), our framework uses graph connectivity to construct samplers without requiring ambient coordinates. This setting includes observations attached to vertices of social or interaction graphs. We consider connected, locally finite graphs with uniformly bounded degree. On the subgraph induced by the vertices in the ball $B_R$ centred at a fixed root $o$, we define the graph Schr\"odinger operator
\[
    \L_R=\Delta_R+V,\qquad V(x):=\dist(o,x)^\beta,
    \qquad \beta>0,
\]
where $\Delta_R$ is the graph Laplacian operator on $B_R$, and $\dist$ is the graph distance.
The confining potential $V$ incorporates distance from the root into the spectral construction. Under the volume-growth assumption
\[
    |B_R|=C_{\mathrm{vol}}R^m+O(R^{m-\gamma}),
    \qquad m>1,\quad 0<\gamma\leq 1,\quad \text{as }R\to\infty,
\]
the associated spectral DPP satisfies
\begin{equation}
    \frac{1}{n^2}\var[\mathcal S_{n}(f)]
    \leq C\|f\|_\infty^2 n^{-1-\min\{\beta,\gamma\}/m}
    \label{eq:intro-network-variance}
\end{equation}
for every bounded observable $f$, uniformly in the truncation radius $R$ and the sample size $1\leq n\leq |B_R|$ (Theorem~\ref{thm:var-reduce-network}). Here $m$ denotes the volume-growth exponent. The bound controls fluctuations of linear statistics about their DPP expectation and requires no regularity of $f$ beyond boundedness. It shows how bounded degree, confinement, and volume growth yield sublinear variance, extending the spectral approach to a class of sparse networks.
In a Euclidean or manifold type network, $\gamma$ would be 1, whereas a reasonable choice of $\beta$ is 2 (corresponding to a quadratic potential $V$); this would lead to an RMSE bound of $n^{-\frac{1}{2}-\frac{1}{2m}}$, in line with our results on manifolds and geometric graphs.

\paragraph{Rate improvements.}
Compared with the independent-sampling benchmark $n^{-1/2}$, our RMSE bound gains an additional factor of $n^{-1/(2m)}$.  Although this gain decreases as $m$ increases, dimension dependence is also inherent in optimal randomized integration rates for isotropic smoothness classes; for example, for integration over $[0,1]^m$, the optimal worst-case RMSE is of order $n^{-1/2-1/m}$ over the Lipschitz functions, 
(see \cite{Novak1988Deterministic,novak-krieg,extensible-grids,ullrich2017} and the references therein). This, in particular, indicates that the $O(m^{-1})$ increment in the exponent on $n$ is a fact of nature, rather than a limitation of the method or proof techniques. Moreover, $m$ is the \emph{intrinsic dimension} of the data, which can remain small or moderate even when the observations lie in a much higher-dimensional ambient space $\mathbb{R}^d$.

\paragraph{PAC guarantees.}
These variance bounds also lead to high-probability approximation guarantees. In \cite{bardenet2024small}, the authors establish concentration inequalities for DPP linear statistics in general settings, whose deviation bounds depend explicitly on the variance. Together with suitable
control of the inverse-intensity weights and the complexity of the query class, these inequalities yield uniform approximation guarantees for weighted coresets. Our results provide the variance estimates
needed to apply this approach to graph spectral DPPs, with a gain governed by the intrinsic dimension of the data. This connects our variance analysis to \emph{Probably Approximately Correct} (PAC) coreset
guarantees: with prescribed confidence, the weighted sample approximates the full-data objective uniformly over the query class.

\subsection{Structure of the paper}

The rest of the paper is organized as follows. Section \ref{sec:main-results} states our main results: the general variance reduction mechanism of spectral DPPs (Subsection \ref{sec:general}) and its applications to manifolds (Subsection \ref{sec:monte-carlo-manifold}), random geometric graphs (Subsection \ref{sec:minibatch-lowdim}), and large networks (Subsection \ref{sec:large-network}). Section \ref{sec:experiments} discusses the computational cost of the spectral construction and presents numerical experiments on geometric point clouds and a minibatch construction for image classification. We present the main proof ideas in Section \ref{sec:proof-idea}, with the complete proofs deferred to the Appendixes ~\ref{proof:general}, \ref{proof:manifold}, \ref{proof:geometric-graph}, \ref{proof:network}. Appendix \ref{app:general} collects background on DPPs and Dirichlet forms.

\section{Main results} \label{sec:main-results}

\subsection{Spectral DPPs: variance control via diffusion geometry} \label{sec:general}

Let $\mathcal X$ be a locally compact separable metric space equipped with a positive Radon measure $\mu$ of full support. We write $L^2(\mu):=L^2(\mathcal X,\mu)$ for the real Hilbert space of square-integrable functions on $\mathcal X$. Let $\L$ be a densely defined nonnegative self-adjoint operator on $L^2(\mu)$, and let
\[\mathcal E(f,g)
:=
\big\langle \L^{1/2}f,\L^{1/2}g\big\rangle,
\qquad
\mathcal D(\mathcal E):=\mathcal D(\L^{1/2}),
\]
be its associated symmetric form.
We impose the following assumptions:
    \begin{enumerate}
\item[(A1)] $(\mathcal E,\mathcal D(\mathcal E))$ is a regular symmetric Dirichlet form;
\item[(A2)] the non-killing part of $\mathcal E$ admits a carr\'e du champ  $\Gamma(\cdot,\cdot)$ with respect to $\mu$;
\item[(A3)] $\L$ has compact resolvent.
\end{enumerate}

By (A3), the spectrum of $\L$ is discrete. More precisely, there exist a finite or countable index set $\mathbb I$, an orthonormal basis $\{\phi_j\}_{j\in\mathbb I}$ of $L^2(\mu)$, and eigenvalues $0\leq \lambda_1\leq\lambda_2\leq\cdots$,
listed with multiplicity, such that $\L\phi_j=\lambda_j\phi_j$ for $j\in\mathbb I$. In what follows, we will fix such an orthonormal basis $\{\phi_j\}_{j \in \mathbb{I}}$. 
When $\mathbb I=\mathbb N$, one has $\lambda_j\to\infty$ as $j\to\infty$. 
Let $n_{\L}(\lambda)$ be the eigenvalue counting function of $\L$, given by
\begin{equation*}
n_{\L}(\lambda)
:=
\#\{k\in \mathbb{I}:\lambda_k\leq\lambda\},
\qquad
\lambda\in\mathbb R,
\end{equation*}
with the convention that $n_{\L}(\lambda)=0$ for $\lambda<0$.

\begin{definition}[Local spectral counting function] \label{def:loc-spec-counting}
    The local spectral counting function associated with $\L$ is defined by
    \[\mathcal R_{\L}(n):= \max\Big \{n_{\L}(1),
    \max_{\substack{0 \le p \le \lfloor\sqrt\lambda_n\rfloor-1\\ p \in \mathbb{Z}}} 
    \big \{n_{\L}\big ((\sqrt{\lambda_{n}}-p)^{2} \big ) - n_{\L}\big ((\sqrt{\lambda_{n}}-p-1)^{2} \big )\big \} \Big \},\]
    where the maximum over an empty set is understood to be zero. 
\end{definition}

For $n \in \mathbb{N}$ with $n \le \dim L^2(\mu)$ if $\dim L^2(\mu)<\infty$, let $\mathcal{S}_n$ be the rank $n$ spectral DPP associated to $\L$.
Our main theorem is the following:
\begin{theorem} \label{thm:var-reduce-general}
Let $\L$ be an operator satisfying Assumptions (A1)--(A3).
Let $\mathcal S_n$ denote the spectral DPP of rank $n$ associated with $\L$. Let 
\[\mathcal{A}:=\{f\in \mathcal{D}(\L^{1/2}): \|f\|_{\infty}<\infty, \|\Gamma(f)\|_{\infty}<\infty\},\]
where $\|\cdot \|_{\infty}$ denotes the essential supremum norm in $L^{\infty}(\mu)$, and $\Gamma(f):=\Gamma(f,f)$.
Then for every $f\in \mathcal{A}$
\[ \var[\mathcal S_n(f)]  \le \Big (\|f\|_{\infty}^2+\frac{\pi^2}{6}\|\Gamma(f)\|_{\infty} \Big) \mathcal R_{\L}(n).
\]
\end{theorem}

Theorem \ref{thm:var-reduce-general} separates the two inputs governing the variance of linear statistics of spectral DPPs: the regularity of $f$ via the carr\'e du champ, and the spectral geometry of $\L$, encoded by the local spectral counting function $\mathcal R_{\L}(n)$. In the following sections, we will see how to control $\mathcal{R}_{\L}(n)$ in concrete examples. The variance reduction phenomenon occurs when $\mathcal{R}_{\L}(n)$ grows sublinearly in $n$. 

\begin{remark}
    To illustrate the resulting order of the variance, suppose that, for some $\beta>0, C>0 $ and $0<\eta \le \beta$,
\[
n_{\L}(\lambda)
=
C\lambda^\beta
+
O\bigl(\lambda^{\beta-\eta}\bigr)
\qquad\text{as }\lambda\to\infty.
\]
Such an asymptotic is known as a \emph{Weyl law with remainder}. For instance, for a Laplace-type operator on a smooth compact $m$-dimensional manifold, the standard Weyl law yields $\beta=m/2$ with remainder of order $O(\lambda^{(m-1)/2})$, which corresponds to $\eta=1/2$. Then, one can show
\[
\mathcal R_{\L}(n)
=
O\bigl(\lambda_n^{\beta-\min\{1/2,\eta\}}\bigr)
=
O\bigl(n^{1-\min\{1/(2\beta),\eta/\beta\}}\bigr).
\]
Applying Theorem \ref{thm:var-reduce-general} yields
$\var[\mathcal S_n(f)]=
O\bigl(n^{1-\min\{1/(2\beta),\eta/\beta\}}\bigr)$.
\end{remark}

    We next establish variance bounds through a variant of the local spectral counting function $\mathcal{R}_{\L}(n)$. For $\alpha>0$, define
    \[\mathcal R_{\L}^{(\alpha)}(n):= \max\Big \{n_{\L}(1),
    \max_{\substack{0 \le p \le \lfloor\lambda_n^{\alpha}\rfloor-1\\ p \in \mathbb{Z}}} 
    \big \{n_{\L}\big ((\lambda_{n}^{\alpha}-p)^{1/\alpha} \big ) - n_{\L}\big ((\lambda_{n}^{\alpha}-p-1)^{1/\alpha} \big )\big \} \Big \},\]
    where the inner maximum is taken to be zero when its index set is empty. This definition recovers $\mathcal{R}_{\L}(n)$ when $\alpha=1/2$ and allows freedom in choosing the spectral counting scale. The choice $\alpha=1/2$ is natural for Laplacian-type operators on manifolds and their Weyl laws. In other settings, stronger commutator estimates may be available for larger exponents, leading to sharper variance bounds. The discrete networks considered in Section \ref{sec:large-network} provide one such example.

\begin{theorem} \label{thm:var-reduce-general-alpha}
    Let $\M_f$ be the operator given by $\M_f g(x)=f(x)g(x)$. 
    We have:
    \begin{enumerate}
        \item[(i)] if the commutator $[\L^{\alpha},\M_f]:= \L^{\alpha}\M_f - \M_f \L^{\alpha}$ can be extended to a bounded operator on $L^2(\mu)$ for some $\alpha>0$, then 
        \[ \var[\mathcal{S}_n(f)] \le \Big (\|f\|_{\infty}^2 + \frac{\pi^2}{6} \|[\L^{\alpha},\M_f]\|^2_{L^2\to L^2}\Big ) \mathcal{R}_{\L}^{(\alpha)}(n);\]
        \item[(ii)] for $0<\alpha<1/2$ and $f\in \mathcal{A}$, the commutator $[\L^{\alpha},\M_f]$, initially defined on $\mathcal{D}(\L^{1/2})$, can be extended to a bounded operator on $L^2(\mu)$ with
        \[ \|[\L^{\alpha},\M_f]\|_{L^2 \to L^2} \le \frac{C}{1-2\alpha}\big (\|f\|_{\infty}^2 + \|\Gamma(f)\|_{\infty}\big )^{1/2},\]
        where $C>0$ is a universal constant.
    \end{enumerate}
\end{theorem}

Theorem \ref{thm:var-reduce-general-alpha} (i) translates commutator control into a variance bound through the corresponding spectral counting function. 
Theorem \ref{thm:var-reduce-general-alpha} (ii) 
establishes fractional commutator estimates that are of independent interest in the analysis of Markov generators. These estimates quantify the error incurred when fractional powers of the generator $\L$ are commuted with multiplication by a function of bounded energy density, providing a tool for localizing fractional operators.
In the Euclidean setting, classical commutator bounds can be obtained using powerful analytical tools such as pointwise heat-kernel estimates or pseudodifferential calculus. It is interesting that we can obtain such quantitative control under general Dirichlet-form assumptions (which accommodate both continuous and discrete settings) by using Dirichlet form theory directly, without requiring a smooth geometric structure or relying on such powerful tools.

\begin{remark}
\begin{itemize}
    \item[(i)] To keep the presentation concise and focused on the main ideas, we defer the necessary Dirichlet-form background to Appendix \ref{app:general}. Nevertheless, we emphasize that our framework is broad: the assumptions imposed on the underlying Dirichlet space are fairly general, and we even allow a nontrivial killing part. In particular, the framework covers our main examples such as diffusion operators on manifolds and weighted graph Laplacians, as well as certain generators on more abstract state spaces. This includes, for example, fractals, wherein existing constructions of Markov semigroup generators on fractals and Weyl's Law on fractals would result in variance bounds governed by the space's intrinsic structural constants.
    \item[(ii)]  The notion of carr\'e du champ will be introduced in Appendix \ref{app:general}. However, to convey the intuition behind (A2), we can view the carr\'e du champ $\Gamma(f):=\Gamma(f,f)$ as the pointwise energy density of $f$, analogous to a squared gradient. For instance, for the Laplace--Beltrami operator, $\Gamma(f)=|\nabla f|^2$, whereas for the Laplacian on a weighted graph with symmetric edge weights $w_{xy}$,
     \[ \Gamma(f)(x) =  \frac{1}{2\mu(x)}\sum_y w_{xy}|f(x)-f(y)|^2. \]
    In both cases, $\Gamma(f)$ quantifies the local variation of $f$.
\end{itemize}
\end{remark}

\subsection{Spectral DPPs on manifolds}\label{sec:monte-carlo-manifold}
\subsubsection{Variance reduction for spectral DPPs on manifolds}

Let  $(\mathcal M,g)$ be a connected, compact Riemannian manifold without boundary of dimension $m\ge 2$. 
Let $\vol_g$ denote the  Riemannian volume measure, and we assume for simplicity that  $\vol_g(\mathcal{M})=1$.
Let $\mu(\d x) = \rho(x) \vol_g(\d x)$ be a probability measure on $\mathcal M$. We assume that
\begin{equation}\label{eq:assumption-rho}
    0 < \rho_{\min} \le \rho(x) \le \rho_{\max} <\infty \quad,\quad \rho \in C^{1,1}(\mathcal M).
\end{equation}

Let $\Delta_g=\div_g(\nabla)$ be the Laplace-Beltrami operator on $(\mathcal{M},g)$.
We consider the
nonnegative self-adjoint operator $\L$ on $L^2(\mu)$ associated with the Dirichlet form
\[
\mathcal E(u,v)
=
\int_{\mathcal M}\langle\nabla u,\nabla v\rangle_g\d\mu,
\qquad
\mathcal D(\mathcal E)=H^1(\mathcal M).
\]
By integration by parts, $\L$ is given by
\[ \L f = -\rho^{-1}\div_g(\rho\nabla f) =
-\Delta_g f-\langle\nabla\log\rho,\nabla f\rangle_g, \qquad \forall f\in C^{\infty}(\mathcal{M}).\]

It is known that the form $(\mathcal{E},H^1(\mathcal{M}))$ is regular and strongly local, with the carr\'e du champ $\Gamma(u,v) = \langle \nabla u , \nabla v \rangle_g$ (cf. \cite{fukushima2011dirichlet}). Moreover, $\L$ has compact resolvent. Thus, assumptions (A1)--(A3) are satisfied in this setting. Let $0 =\lambda_1 \le  \lambda_2\leq\cdots$ be eigenvalues of $\L$ listed with multiplicity. We fix a real orthonormal eigenbasis $\{\phi_j\}_{j \ge 1}$ of $L^2(\mu)$ such that $\L\phi_j=\lambda_j\phi_j$ for every $j$.
Let $\mathcal{S}_n$ be the corresponding spectral DPP of rank $n$ associated to $\L$.

\begin{theorem}\label{thm:var-reduce-manifold}
    Let $f:\mathcal{M} \rightarrow \R$ be a Lipschitz function with Lipschitz constant $\Lip(f)$. Then
    \[ \var[\mathcal{S}_n(f)] \le C \big (\|f\|_{\infty}^2 +   \Lip(f)^{2} \big ) \cdot n^{1-1/m},\]
    where $C>0$ is a constant depending only on $(\mathcal{M},g)$ and $\rho$.
\end{theorem}

    Theorem~\ref{thm:var-reduce-manifold} gives a variance upper bound of order $n^{1-1/m}$, where $m$ is the dimension of the underlying manifold. This matches the known variance bounds for multivariate OPEs \cite{bardenet2020monte} and harmonic ensembles on $m$-dimensional spheres \cite{levi2024linear}. Our result applies
    to general compact Riemannian manifolds and for weighted diffusion generators. Its scope includes general geometries and nonconstant densities, while the exponent remains governed by the intrinsic dimension $m$. 


\subsubsection{Application to Monte Carlo integration}

Let $\omega\in C^1(\mathcal{M})$ be the target weight. Let $f:\mathcal{M} \to \R$ be a Lipschitz function. We aim to approximate the integral
\[ \mathcal{L}_{\omega}(f) := \int_{\mathcal{M}} f(x)\omega(x)\vol_g(\d x),\]
by estimators of the form
\[ \widehat{\mathcal{L}}_n(f) = \frac{1}{n} \sum_{i=1}^n w_if(X_i),\]
where $\mathcal{S}=\{X_1,\ldots,X_n\}$ is a finite set of samples, and $\{w_i\}_{1\le i \le n}$ are suitable weights.

We use spectral DPPs on $\mathcal{M}$ for the Monte Carlo quadrature. In particular,
our quadrature rule is as follows. 
For a sample size $n\in \N$, let $\mathcal{S}_n$ be the spectral DPP of rank $n$ associated with the operator $\L=-\Delta_g$. In particular, $\rho=1$ and $\mu=\vol_g$ in this setting. We choose the nodes to be sample points in $\mathcal{S}_n$, and the weight to be $n \omega(x)/\K_n(x,x)$. Here $\K_n$ is the kernel of $\mathcal{S}_n$ (with reference measure $\mu$) 
\[ \K_n(x,y) = \sum_{k=1}^n \phi_k(x)\phi_k(y),\]
where $\phi_1,\phi_2,\ldots$ are orthonormal eigenfunctions of $\Delta_g$. Since $\mathcal{M}$ is connected, we can choose $\phi_1=1$.
Thus, for each $f: \mathcal{M} \to \R$, the corresponding estimator is given by
\begin{equation*}
    \widehat{\mathcal{L}}_n(f) :=\frac{1}{n}\sum_{x\in \mathcal{S}_n}\frac{n\omega(x)}{\K_n(x,x)}f(x).
\end{equation*}

\begin{theorem} \label{thm:monte-carlo-manifold}
There exists $C>0$, depending only on
$(\mathcal M,g)$, such that for every Lipschitz function $f:\mathcal{M} \to \R$
\[ \mathbb E\left[\big|\widehat{\mathcal{L}}_{n}(f)- \mathcal{L}_{\omega}(f)\big|^2\right]^{1/2}
\le
C\big(\|\omega\|_\infty+\|\nabla \omega\|_{\infty}\big)  \big(\|f\|_\infty+\Lip(f)\big) n^{-1/2-1/(2m)} \cdot
\]
\end{theorem}

Theorem~\ref{thm:monte-carlo-manifold} gives $\operatorname{RMSE}(\widehat{\mathcal{L}}_n(f))=O(n^{-1/2-1/(2m)})$ for Lipschitz $f$. As noted earlier, this matches the quadrature rate established for multivariate orthogonal polynomial ensembles in Euclidean domains~\cite{bardenet2020monte}. Our result extends this guarantee to Lipschitz integrands on general compact Riemannian manifolds, with a sampling construction defined directly from the intrinsic geometry and requiring no global Euclidean parametrization. In particular, the convergence rate depends on the intrinsic dimension $m$, even when the manifold is embedded in a much higher-dimensional ambient space. Thus, we accommodate curved, low-dimensional geometries while retaining the convergence exponent available in Euclidean spaces of the same intrinsic dimension.


\subsection{Spectral DPPs on random geometric graphs} \label{sec:minibatch-lowdim}
In this section, we study spectral DPPs constructed on a random geometric graph on a finite data set. The data set is a random point cloud
\[\mathcal{X}:=\{x_1,\ldots,x_N\}\subset \R^d,
    \qquad \text{where }
    x_1,\ldots,x_N\stackrel{\mathrm{i.i.d.}}{\sim}\mu.\]
Here $\mu$ is an unknown probability measure on $\R^d$. We focus on the singular setting $\mu=\rho\,\mathrm{vol}_g$, where $\vol_g$ is the Riemannian volume form of a manifold $(\mathcal{M},g)$ and
\[\rho\in C^\infty(\mathcal M),\qquad  0<\rho_{\min}\leq \rho\leq\rho_{\max}<\infty.\]
We assume that $(\mathcal M,g)$ is a compact, orientable, connected, smooth Riemannian manifold without boundary, of intrinsic dimension $m\ge 2$, isometrically embedded in
$\mathbb R^d$. We identify $\mathcal M$ with its image and write $|\cdot|$ for the Euclidean norm. Let $K_{\mathcal{M}}>0$ be an upper bound for the absolute sectional curvature of $\mathcal{M}$, and let $R_{\mathcal{M}}>0$
be the reach of the embedded manifold.
We denote by $i_{\mathcal{M}}>0$ the injectivity radius of $(\mathcal{M},g)$.

Given the data set $\mathcal{X}$, we first construct a weighted undirected graph on \(\mathcal{X}\) by
assigning symmetric weights
$w_{ij}=w_{ji}\geq0$ for $1\leq i,j\leq N$.
Its raw weighted graph Laplacian is given by
\begin{equation*}
\widetilde{\L}_N u(x_i)
 :=\frac1N\sum_{j=1}^N
   w_{ij}\bigl(u(x_i)-u(x_j)\bigr).
\end{equation*}
We introduce a deterministic parameter $h>0$ (possibly depending on
$N$) called the \emph{rescaling parameter} of the graph. Given $h$, we define the
rescaled graph Laplacian by
\begin{equation}
 \L_N u(x_i)
 :=\frac{c_{\mathsf G}}{Nh^{m+2}}
   \sum_{j=1}^Nw_{ij}\bigl(u(x_i)-u(x_j)\bigr)
 =c_{\mathsf G}h^{-(m+2)}\widetilde \L_N u(x_i),
 \label{eq:rescaled-graph-laplacian-four}
\end{equation}
where \(c_{\mathsf G}>0\) is a deterministic normalization constant associated
with the graph construction. The rescaling parameter $h$ is the appropriate normalization such that $\L_N$ approximates a diffusion generator on $\mathcal{M}$ as $N\to \infty$.

We equip the data set $\mathcal{X}$ with the empirical measure $\mu_N:=\frac1N\sum_{i=1}^N\delta_{x_i}$, and denote by \(L^2(\mu_N)\) the space of real-valued functions on \(\mathcal{X}\) with
\[ \langle u,v\rangle_{L^2(\mu_N)} :=\frac1N\sum_{i=1}^N u(x_i)v(x_i).\]
It is easy to check that $\L_N$ is a self-adjoint operator on $L^2(\mu_N)$. Moreover, it satisfies Assumptions (A1)--(A3) in Section \ref{sec:general}. For each $1\le n\le N$, we denote by $\mathcal{S}_{n,N}$ the rank-$n$ spectral DPP associated to $\L_N$.

In what follows, we fix a non-increasing function \(\eta:[0,\infty)\to[0,\infty)\), supported in
\([0,1]\) and Lipschitz on \([0,1]\), with \(\eta(1/2)>0\). We normalize
$\eta$ so that $\int_{\mathbb R^m}\eta(|z|)\,\mathrm dz=1,$
and set
\[
    \sigma_\eta:=\int_{\mathbb R^m}z_1^2\eta(|z|)\,\mathrm dz.
\]
The edge weights in the two graph constructions below are defined using
\(\eta\).

\paragraph{$\varepsilon$-graphs.}

Let \(\varepsilon>0\). We join \(x_i\) and \(x_j\) whenever
\(|x_i-x_j|\leq\varepsilon\), and assign the weight
\[
    w_{ij}^\varepsilon
    :=\eta\!\left(\frac{|x_i-x_j|}{\varepsilon}\right).
\]
We define
\begin{equation}
 \L_N^\varepsilon u(x_i)
 :=\frac{1}{N\varepsilon^{m+2}}
   \sum_{j=1}^Nw_{ij}^\varepsilon
   \bigl(u(x_i)-u(x_j)\bigr).
 \label{eq:epsilon-laplacian-four}
\end{equation}
This is \eqref{eq:rescaled-graph-laplacian-four} with rescaling parameter
\(h=\varepsilon\) and normalization constant \(c_{\mathsf G}=1\).

\paragraph{Symmetric $k$-NN graphs.}

Let $k\in \mathbb{N}$ such that $1\le k \le N-1$. 
Let \(\varepsilon_k(x_i)\) denote the Euclidean distance from \(x_i\) to its
\(k\)-th nearest neighbour and set
\[
    r_k(x_i,x_j)
    :=\max\{\varepsilon_k(x_i),\varepsilon_k(x_j)\}.
\]
The symmetric \(k\)-NN graph joins \(x_i\) and \(x_j\) whenever
\(|x_i-x_j|\leq r_k(x_i,x_j)\). We define
\begin{equation}
 \L_N^{k}u(x_i)
 :=\frac1N\left(\frac{N\omega_m}{k}\right)^{1+2/m}
   \sum_{j=1}^N
   \eta\!\left(\frac{|x_i-x_j|}{r_k(x_i,x_j)}\right)
   \bigl(u(x_i)-u(x_j)\bigr).
 \label{eq:knn-laplacian-four}
\end{equation}
This is an instance of
\eqref{eq:rescaled-graph-laplacian-four} with rescaling parameter \(h=(k/N)^{1/m}\)
and normalization constant
\(c_{\mathsf G}=\omega_m^{1+2/m}\), where $\omega_m$ denotes the volume of the unit ball in $\R^m$.

\begin{remark}[Convention on constants]
Throughout this section, \((\mathcal M,g)\), its embedding in \(\mathbb R^d\),
\(\rho\), and \(\eta\) are fixed. We call a constant \emph{structural} if it
depends only on these fixed data. This includes dependence on fixed geometric data of \(\mathcal M\), on the density \(\rho\) and on the kernel \(\eta\). Dependence on any additional
fixed parameter will be indicated explicitly.

The notation \(A\lesssim B\) means
\(A\leq CB\) for some structural constant $C>0$, and \(A\asymp B\) means both \(A\lesssim B\) and
\(B\lesssim A\), with structural comparison constants.  If $a$ and $b$ are additional parameters, we write $A\lesssim_{a,b}B$ or $A\asymp_{a,b}B$ when the implicit constants may also depend on $a$ and $b$.
\end{remark}


\subsubsection{Variance reduction for spectral DPPs on geometric graphs} \label{subsec:graph-main-results}

\paragraph{$\eps$-graphs}
Let $\L_N^{\eps}$ be defined as in \eqref{eq:epsilon-laplacian-four}.
We list the eigenvalues of \(\L_N^{\eps}\) as
 \[0=\lambda_{1}(\L_N^{\eps})\leq\lambda_{2}(\L_N^{\eps})\leq\cdots\leq\lambda_{N}(\L_N^{\eps}),\] 
 repeated according to multiplicity. We fix an orthonormal basis $\{\phi_{j,N}^{\eps}\}_{j=1}^N$ of $L^2(\mu_N)$ such that $\phi_{1,N}^{\eps} = 1$ and $\L_N^{\eps} \phi_{j,N}^{\eps}= \lambda_j(\L_N^{\eps}) \phi_{j,N}^{\eps}$ for every $1\le j \le N$.
For each $1\le n\le N$, let $\mathcal{S}_{n,N}^{\eps}$ be the rank-$n$ spectral DPP associated to $\L_N^{\eps}$, defined by the kernel
\[\K_{n,N}^{\eps}(x,y) = \sum_{j=1}^n \phi_{j,N}^{\eps}(x) \phi_{j,N}^{\eps}(y).\]
\begin{theorem}    \label{thm:var-reduce-eps}
 Let $\eps>0$ be such that 
 \[\eps < \frac{1}{2}\min \{1,i_{\mathcal{M}}, K_{\mathcal{M}}^{-1/2},R_{\mathcal{M}}/2\}.\]
 Let $\theta$ and $\tilde{\delta}$ be positive parameters satisfying 
  \begin{enumerate}
        \item[(i)] $N^{-1/m} \le \tilde{\delta} \le \eps/4$;
        \item[(ii)] $C(\theta + \tilde{\delta}) \le  \rho_{\min}/2$, where $C$ is a structural constant.
    \end{enumerate}
    Then, there exist positive structural constants $C_1, C_2$ such that:
    with probability at least $1-C_1N \exp(-C_1^{-1}N\theta^2\tilde{\delta}^m)$, we have 
    \begin{eqnarray*}
        \frac{1}{n^2}\var\big [\mathcal{S}_{n,N}^{\eps}(f)|\mathcal{X} \big] &\lesssim& \big(\|f\|^2_{L^{\infty}(\mathcal{X})} +\lip_{\mathcal{M}}^2(f)\big) \cdot
    n^{-1-\frac1m} \\
    &\times& \Big [1 + \Big (\frac{\tilde{\delta}}{\eps}+\eps +\theta\Big ) n^{\frac{1}{m}} + \eps n^{\frac2m}\Big]
    \end{eqnarray*}
   for every $f:\mathcal{M} \to \R$ with Lipschitz constant $\lip_{\mathcal{M}}(f)$, and for every $n$ satisfying $\lambda_n(\L_N^{\eps}) \ge 1$ and 
    \[\frac{\tilde{\delta}}{\eps} + \eps + \theta + \eps\cdot n^{1/m}\le C_2.\]
\end{theorem}

A concrete choice of parameters for $\varepsilon$-graphs yields the following corollary:

\begin{corollary}
\label{cor:epsilon-concrete-scaling}
Fix $r>0$. There exists a constant $A_r>0$ depending only on $r$
and the structural data, a structural integer $n_0\geq 1$, and
$N_0(r)\in\mathbb{N}$ such that the following holds.
\\
For $N\geq N_0(r)$, set
\[
\varepsilon_N
:=
\left(\frac{\log N}{N}\right)^{\frac{1}{2m+2}},
\qquad
\widetilde{\delta}_N
:=
\varepsilon_N^2,
\qquad
\theta_N
:=
A_r\varepsilon_N.
\]
Then, with probability at least
$ 1-C_rN^{-r}, $
one has
\[
\frac{1}{n^2}
\var
\left[
\mathcal S^{\varepsilon_N}_{n,N}(f)| \mathcal X
\right]
\lesssim_r
\left(
\|f\|_{L^\infty(\mathcal X)}^2
+
\operatorname{Lip}_{\mathcal{M}}(f)^2
\right) \cdot
n^{-1-1/m}
\]
for every Lipschitz function $f:\mathcal{M}\to\mathbb{R}$ and every
\[
n_0
\leq n
\leq
\varepsilon_N^{-m/2}
=
\left(
\frac{N}{\log N}
\right)^{\frac{m}{4m+4}}.
\]
Here $C_r>0$ and the implicit constant in $\lesssim_r$ may depend on
$r$ and the structural data, but not on $N$ or $n$.
\end{corollary}

\paragraph{$k$-NN graphs}

Let $\L_N^k$ be defined as in \eqref{eq:knn-laplacian-four}.
We list the eigenvalues of \(\L_N^{k}\) as
 \[0=\lambda_{1}(\L_N^{k})\leq\lambda_{2}(\L_N^{k})\leq\cdots\leq\lambda_{N}(\L_N^{k}),\] 
 repeated according to multiplicity. We fix an orthonormal basis $\{\phi_{j,N}^{k}\}_{j=1}^N$ of $L^2(\mu_N)$ such that $\phi_{1,N}^k =1$ and  $\L_N^{k} \phi_{j,N}^{k}= \lambda_j(\L_N^{k}) \phi_{j,N}^{k}$ for every $1\le j \le N$.
For each $1\le n\le N$, let $\mathcal{S}_{n,N}^{k}$ be the rank-$n$ spectral DPP associated to $\L_N^{k}$, defined by the kernel
\[\K_{n,N}^{k}(x,y) = \sum_{j=1}^n \phi_{j,N}^{k}(x) \phi_{j,N}^{k}(y).\]

\begin{theorem}   \label{thm:var-reduce-knn}
 Let $k$ be such that 
 \[(k/N)^{1/m} <C_{\rho}\min \{1,i_{\mathcal{M}}, K_{\mathcal{M}}^{-1/2},R_{\mathcal{M}}/2\}\]
 where $C_{\rho}$ is a constant that depends only on the density $\rho$. 
 Let $\theta$ and $\tilde{\delta}$ be positive parameters satisfying 
  \begin{enumerate}
        \item[(i)] $N^{-1/m} \le \tilde{\delta} \le c_{\rho}(k/N)^{1/m}$, where $c_{\rho}$ depends only on $\rho$ ;
        \item[(ii)] $C(\theta + \tilde{\delta}) \le  \rho_{\min}/2$, where $C$ is a structural constant.
    \end{enumerate}
    Then, there exist positive structural constants $C_1, C_2$ such that:
    with probability at least $1-C_1N \exp(-C_1^{-1}N\theta^2\tilde{\delta}^m)$, we have 
    \begin{eqnarray*}
        \frac{1}{n^2}\var\big [\mathcal{S}_{n,N}^{k}(f)\big|\mathcal{X} \big ]&\lesssim& \big(\|f\|^2_{L^{\infty}(\mathcal{X})} +\lip_{\mathcal{M}}^2(f)\big)
    \cdot n^{-1-\frac1m} \\
    &\times&
    \Big [1 + \Big (\frac{\tilde{\delta}}{(k/N)^{1/m}}+(k/N)^{1/m} +\theta\Big ) n^{\frac{1}{m}} + (k/N)^{1/m} n^{\frac2m}\Big]
    \end{eqnarray*}
   for every $f:\mathcal{M} \to \R$ with Lipschitz constant $\lip_{\mathcal{M}}(f)$, and for every $n$ satisfying $\lambda_n(\L_N^{k}) \ge 1$ and 
    \[\frac{\tilde{\delta}}{(k/N)^{1/m}}+(k/N)^{1/m} +\theta + (k/N)^{1/m} n^{\frac1m}\le C_2.\]
\end{theorem}

A concrete choice of parameters for $k$-NN graphs yields

\begin{corollary}
\label{cor:knn-concrete-scaling}
Fix $r>0$. There exists a constant $A_r>0$ depending only on $r$
and the structural data, a structural integer $n_0\geq 1$, and
$N_0(r)\in\mathbb{N}$ such that the following holds.
\\
For $N\geq N_0(r)$, set
\[
k_N
:=
\left\lceil
N^{\frac{m+2}{2m+2}}
(\log N)^{\frac{m}{2m+2}}
\right\rceil,
\qquad
h_N
:=
\left(\frac{k_N}{N}\right)^{1/m} \asymp
\left(\frac{\log N}{N}\right)^{\frac{1}{2m+2}},
\]
and let
\[
\widetilde{\delta}_N
:=
h_N^2,
\qquad
\theta_N
:=
A_r h_N.
\]
Then, with probability at least
$
1-C_rN^{-r},
$
one has
\[
\frac{1}{n^2}
\var
\left[
\mathcal S^{k_N}_{n,N}(f) \big |\mathcal X
\right]
\lesssim_r
\left(
\|f\|_{L^\infty(\mathcal X)}^2
+
\operatorname{Lip}_{\mathcal{M}}(f)^2
\right) \cdot 
n^{-1-1/m}
\]
for every Lipschitz function $f:\mathcal{M}\to\mathbb{R}$ and every
\[ n_0 \leq n \leq h_N^{-m/2} \asymp
\left(\frac{N}{\log N}\right)^{\frac{m}{4m+4}}.\]
Here $C_r>0$ and the implicit constants in $\lesssim_r$
may depend on $r$ and the structural data, but not on $N$ or $n$.
\end{corollary}

Earlier work on graph spectral DPPs (see \cite{tremblay2017graphsamplingdeterminantalprocesses}, \cite{jaquard2026}) left the question of explicit improvements in the variance rate over independent sampling unresolved. To the best of our knowledge, our results above provide the \emph{first explicit better-than-i.i.d. variance rates} for graph Laplacian spectral DPPs constructed from finite samples on a manifold. In particular, Theorems~\ref{thm:var-reduce-eps} and~\ref{thm:var-reduce-knn} establish such improvements for $\varepsilon$-graphs and symmetric $k$-nearest-neighbour graphs, respectively, with explicit control of the graph approximation error. Under suitable parameter choices, Corollaries~\ref{cor:epsilon-concrete-scaling} and~\ref{cor:knn-concrete-scaling} yield variance bounds of $O(n^{-1-1/m})$ improving on the standard $n^{-1}$ rate of independent sampling. These bounds hold with high probability, simultaneously over an explicit range of subsample sizes $n$ that grows with the data size $N$, with rates governed by the intrinsic dimension $m$. Thus, both graph constructions translate the geometry of an observed point cloud into a provable statistical advantage, without requiring knowledge of the underlying manifold.

\subsubsection{Application to minibatch sampling}

As an application, we construct minibatch estimators of empirical average losses with better approximation rates than independent sampling. 
Given the data set $\mathcal{X}=\{x_1,\ldots,x_N\}$ and $f:\mathcal{M} \to \R$
a Lipschitz function, let 
\[\mathcal{L}_N(f):= \frac{1}{N} \sum_{i=1}^N f(x_i).\] 
We approximate $\mathcal{L}_N(f)$ by estimators of the form
$\mathcal{L}_{\mathcal{S}}(f) := \sum_{j=1}^n w(y_j)f(y_j)$,
where $\mathcal{S}=\{y_1,\ldots, y_n\} \subset \mathcal{X}$ is a minibatch and $w:\mathcal{X}\to \R$ is a suitable weight. 

\paragraph{$\eps$-graphs}
For $\eps$-graphs, we consider the estimator
\begin{equation}
    \mathcal{L}_{n,N}^{\eps}(f):= \sum_{x\in \mathcal{S}_{n,N}^{\eps}}\frac{f(x)}{\K_{n,N}^{\eps}(x,x)} = \mathcal{S}_{n,N}^{\eps}\Big (\frac{f}{\K_{n,N}^{\eps}|_{\diag}}\Big ),
\end{equation}
where $\K_{n,N}^{\eps}|_{\diag}(x):= \K_{n,N}^{\eps}(x,x)$ for $x\in \mathcal{X}$.

\begin{theorem} \label{thm:eps-minibatch}
Fix $r>0$ and $\kappa>0$. There exists a constant $C_r>0$ depending only on $r$
and the structural data, a structural integer $n_0\geq 1$, and
$N_0(r)\in\mathbb{N}$ such that the following holds.
For $N\geq N_0(r)$, set
\[ \varepsilon_N := \left(\frac{\log N}{N}\right)^{\frac{1}{2m+2}},
\qquad I_N:= \Big \{n \in \N: n_0 \le n \le \Big (\frac{N}{\log N}\Big )^{\frac{m}{(m+1)(m+5)}}\Big \}.
\]
Then, with probability at least $1-C_rN^{-r} -p_{N,\kappa}$, we have
\[
\mathbb{E}\Big [\big |\mathcal{L}_{n,N}^{\eps_N}(f) - \mathcal{L}_N(f)\big |^2 \big |\mathcal{X} \Big ]
\lesssim_{r,\kappa}
\left(
\|f\|_{L^\infty(\mathcal X)}^2
+
\operatorname{Lip}_{\mathcal{M}}(f)^2
\right) \cdot
n^{-1-1/m},
\]
simultaneously for every $n\in I_N$ and for every Lipschitz function $f:\mathcal{M} \to \R$. 
Here $p_{N,\kappa} \in [0,1]$ is such that
\begin{equation*}
    \mathbb{P}\Big ( \frac{\K_{n,N}^{\eps_N}(x,x)}{n} \ge \kappa , \: \forall x\in \mathcal{X}, \forall n\in I_N\Big ) \ge 1-p_{N,\kappa}.
\end{equation*}
\end{theorem}

\paragraph{$k$-NN graphs}
For $k$-NN graphs, we consider the estimator
\begin{equation} 
    \mathcal{L}_{n,N}^{k}(f):= \sum_{x\in \mathcal{S}_{n,N}^{k}}\frac{f(x)}{\K_{n,N}^{k}(x,x)} = \mathcal{S}_{n,N}^{k}\Big (\frac{f}{\K_{n,N}^{k}|_{\diag}}\Big ),
\end{equation}
where $\K_{n,N}^{k}|_{\diag}(x):= \K_{n,N}^{k}(x,x)$ for $x\in \mathcal{X}$.

\begin{theorem} \label{thm:knn-minibatch}
Fix $r>0$ and $\kappa>0$. There exists a constant $C_r>0$ depending only on $r$
and the structural data, a structural integer $n_0\geq 1$, and
$N_0(r)\in\mathbb{N}$ such that the following holds.
For $N\geq N_0(r)$, set
\[
k_N
:=
\left\lceil
N^{\frac{m+2}{2m+2}}
(\log N)^{\frac{m}{2m+2}}
\right\rceil,
\qquad
h_N
:=
\left(\frac{k_N}{N}\right)^{1/m} \asymp
\left(\frac{\log N}{N}\right)^{\frac{1}{2m+2}},
\]
and $I_N:= \{n \in \N: n_0 \le n \le h_N^{-2m/(m+5)} \}$.
Then, with probability at least $1-C_rN^{-r} -p_{N,\kappa}$, we have
\[
\mathbb{E}\Big [\big |\mathcal{L}_{n,N}^{k_N}(f) - \mathcal{L}_N(f)\big |^2 \big |\mathcal{X} \Big ]
\lesssim_{r,\kappa}
\left(
\|f\|_{L^\infty(\mathcal X)}^2
+
\operatorname{Lip}_{\mathcal{M}}(f)^2
\right)
\cdot n^{-1-1/m},
\]
simultaneously for every $n\in I_N$ and for every Lipschitz function $f:\mathcal{M} \to \R$. Here $p_{N,\kappa} \in [0,1]$ is such that
\begin{equation*}
    \mathbb{P}\Big ( \frac{\K_{n,N}^{k_N}(x,x)}{n} \ge \kappa , \: \forall x\in \mathcal{X}, \forall n\in I_N\Big ) \ge 1-p_{N,\kappa}.
\end{equation*}
\end{theorem}

    We briefly comment on the term $p_{N,\kappa}$ appearing in Theorems \ref{thm:eps-minibatch} and \ref{thm:knn-minibatch}. Let
    \[\K_{n,N}(x,y)=\sum_{j=1}^n\phi_{j,N}(x)\phi_{j,N}(y),\qquad \K_n^\star(x,y)=\sum_{j=1}^n\phi_j^\star(x)\phi_j^\star(y)\] 
    denote the discrete kernel and
    the corresponding continuum spectral projection kernel, respectively.
    Quantitative spectral convergence results (see \cite{calder2022improved,calder2022lipschitz,dunson2021spectral}) establish convergence of graph Laplacian eigenvectors $\phi_{j,N}$ towards eigenfunctions $\phi_j^{\star}$ of continuum diffusion operators, under appropriate hypotheses and choices of eigenbases.  Furthermore, the local Weyl law (see Lemma \ref{lm:continuum-kernel-local-weyl}) gives $\K_n^\star(x,x)\ge c_\star n$ uniformly in $x$ and $n$. This suggests that, for sufficiently small fixed $0<\kappa<c_\star$, the analogous lower bound for the discrete kernel should hold with probability tending to one, allowing $p_{N,\kappa}\to0$. 

    However, establishing such bounds and obtaining quantitative decay rates for $p_{N,\kappa}$ constitute a natural problem in the spectral approximation of graph Laplacians and manifold learning, which is of interest in its own right and beyond the scope of this work. Thus, we treat it as an assumption and leave this question to future research.



\subsection{Spectral DPPs on large networks} \label{sec:large-network}
Let $G=(\mathcal{X},E)$ be an infinite, connected, locally finite, simple graph, and let $d$ denote the graph distance. Fix a vertex $o \in \mathcal{X}$, called the \emph{root}. We define 
\[ r(x):= d(o,x), \qquad B_R:=\{x\in \mathcal{X}: r(x) \le R\},\qquad R\ge 0.\]
Local finiteness implies that $B_R$ is finite for every $R<\infty$.
We denote by $G_R$ the subgraph of $G$ induced by $B_R$. Then $G_R$ is a finite, connected graph, and we set $N_R:=|G_R|$.

Fix $\beta>0$. We define a killing potential $V:\mathcal{X} \to [0,\infty)$ by 
\begin{equation}
    V(x):= r(x)^{\beta}, \qquad x\in \mathcal{X}.
\end{equation}
For each $R>0$, we consider the graph Laplacian on $B_R$ given by
\[ \Delta_R u(x):= \sum_{y\in B_R, y\sim x} (u(x)-u(y)), \qquad x\in B_R.\]
Let $\mu_R:= N_R^{-1} \sum_{x\in B_R}\delta_x$.
On $L^2(\mu_R)$, we define 
\begin{equation}
    \L_R := \Delta_R + \M_{V|_{B_R}} 
\end{equation}
where $\M_{V|_{B_R}}: L^2(\mu_R)\to L^2(\mu_R)$ denotes the linear operator given by multiplication by the function $V|_{B_R}$. Then $\L_R$ is a positive semidefinite, self-adjoint operator on $L^2(\mu_R)$. We list the eigenvalues (repeated according to multiplicity) of $\L_R$ as
\[ \lambda_{R,1} \le \lambda_{R,2} \le \cdots \le \lambda_{R,N_R}.\]
For $1\le n \le N_R$, we let $\mathcal{S}_{n,R}$ be the rank $n$ spectral DPP associated to $\L_R$.

\begin{theorem} \label{thm:var-reduce-network}
 We assume that:
    \begin{itemize}
        \item[(i)] $G$ has uniformly bounded
degree, i.e., there exists $D<\infty$ such that $\deg(x)\le D$ for every $x\in \mathcal{X}$;
        \item[(ii)] $|B_R| = C_{\vol}R^m + O(R^{m-\gamma})$ as $R\to \infty$, for some $m>1, 0<\gamma \le 1$ and $C_{\vol}>0$.
    \end{itemize}
Let $\eta:= \min \{\beta, \gamma\}/m$. Then, there exists $C>0$ independent of $n, R$ and $f$ such that
\[ \var[\mathcal{S}_{n,R}(f)] \le C\|f\|_{\infty}^2 n^{1-\eta}, \qquad \forall 1\le n \le N_R\]
 for every $R\in \N$ and for every bounded $f:B_R \to \R$.
\end{theorem}

Thus, beyond the geometric networks considered in the previous section, Theorem~\ref{thm:var-reduce-network} applies to \emph{sparse} networks. It identifies another mechanism for variance reduction: bounded degree, confinement, and quantitative volume growth together yield sublinear variance for every bounded observable, without requiring additional regularity. The bound is uniform in the truncation radius $R$ and therefore remains valid as the finite graphs exhaust the infinite network. The explicit gain exponent $\min\{\beta,\gamma\}/m$ reveals how the strength of confinement and the precision of the volume-growth asymptotic jointly determine the extent of variance reduction.

\section{Computational costs and experiments}\label{sec:experiments}
We numerically validate the variance reduction predicted by the graph-spectral DPP
construction for both $\varepsilon$-graphs and $k$-nearest-neighbour graphs. The
goal is to test whether minibatches sampled from graph-Laplacian spectral DPPs
exhibit variance decay governed by the intrinsic dimension of the data, and to
compare their performance with standard uniform subsampling.

Given a point cloud $X=\{x_1,\ldots,x_N\}$, we construct either an unweighted
symmetric $\varepsilon$-graph or a symmetrized $k$-NN graph, and let
$L=D-W$ be the corresponding unnormalized graph Laplacian. For each minibatch
size $n$, we compute the first $n$ eigenvectors of $L$ and use them to define a
rank-$n$ projection DPP on the data set. If $\Psi_n$ denotes the matrix whose
columns are these eigenvectors, then the DPP kernel is $K_n=\Psi_n\Psi_n^\top$. We use the
estimator $N^{-1}\sum_{i\in S} f(x_i)/\pi_i$, where  $\pi_i=K_n(i,i)$ is the inclusion probability. As a baseline, we use iid sampling, which is denoted by the estimator \textit{uniform}. Our experiments demonstrate improved performance for the DPP samplers across all tasks. Although our theory establishes improvements only when the intrinsic dimension $m \geq 2$, we experimentally observe improved performance even when $m = 1$.

Sampling a spectral DPP of size $n\ll N$ requires only the first $n$
eigenvectors of the graph Laplacian. These can be computed using sparse
iterative eigensolvers, such as the locally optimal block preconditioned
conjugate gradient method (LOBPCG) \cite{knyazev2001toward}.
Once the graph has been constructed, let $s_N$ denote the number of
nonzero entries of its Laplacian. With identity preconditioning, LOBPCG with $T_\text{eig}$ iterations requires
$O\bigl(T_{\text{eig}}(s_Nn+Nn^2)\bigr)$
arithmetic operations and $O(s_N+Nn)$ storage. The two terms account
for sparse matrix--block products and orthogonalization together with
the projected eigenproblems. The iteration count depends on the
requested accuracy and spectral separation. 

For a symmetric
$k$-nearest-neighbor graph, $s_N=O(Nk)$, giving a per-iteration cost
$O\bigl(N(kn+n^2)\bigr)$. Plugging in $k \asymp N^{\frac{m+2}{2m+2}}(\log N)^{\frac{m}{2m+2}}$, we obtain a per-iteration cost of 
\[
O\bigl(N^{\frac{3m+4}{2m+2}}(\log N)^{\frac{m}{2m+2}}n+Nn^2\bigr).
\]

Additionally, we incur the preprocessing cost of constructing the graph from the given point cloud, which comes out to be $O(N^2d)$ for $k-$nearest neighbours graphs using the brute-force method. 

The construction of the graph and computation of the eigenvectors of the Laplacian constitute preprocessing steps that can be reused
for arbitrarily many samples from the same spectral DPP. Sampling
operates directly on the $N\times n$ eigenvector matrix, without
forming the full $N\times N$ projection kernel. Thus, for sparse graphs
and moderate sample sizes, the computational requirements remain
manageable, especially when the preprocessing cost is amortized
over repeated minibatch sampling.

In particular, the repeated DPP sampler in \cite{barthelme2023faster} has an expected cost of $O(Nn + n^3\log n)$ for the first draw, followed by a cost of $O(n^3\log n)$ for subsequent draws.

\subsection{Variance reduction for point clouds on the circle and sphere}
Our first experiment involves subsampling points lying on two elementary manifolds. The first is the unit circle
$\mathbb{T}^1\subset\mathbb{R}^2$, on which the points are sampled either uniformly or from the one parameter family of 
densities $\rho_a(t)=1+a\cos(2\pi t), \ \  0<a < 1$, and tested against the Lipschitz function
$f(t)=|2t-1|$. The second is the unit sphere $S^2\subset\mathbb{R}^3$, on which the points are sampled
either uniformly or from the one parameter family of 
densities proportional to $1+az, \ \ 0 < a < 1$, where $z$ is
the height coordinate, and tested against $f(x)=\|x-e_3\|_2$, the Euclidean distance
from the north pole. These test functions are Lipschitz, but not smooth.

For each point cloud, graph type, and minibatch size, we estimate the conditional
variance of the estimator by repeated DPP or iid samples, and then average this
variance over several independent point clouds. In the reported experiments we use
$N=10^4$ data points, minibatch sizes
$n\in\{5,7,10,12,15,17,20,22,25\}$, four independent point clouds, and 150
sampling repetitions per cloud. For the circle we take $\varepsilon=0.08$ and
$k=40$, while for the sphere we take $\varepsilon=0.20$ and $k=45$. The empirical
rates are obtained by fitting a line to the log-log plot of variance against $n$.

On both the circle and the sphere, the performance of both the graph-based DPPs is approximately in line with predicted rates for both uniform and non-uniform densities.

\begin{figure}[H]
    \centering

    \begin{minipage}{0.48\textwidth}
        \centering
        \includegraphics[width=\linewidth]{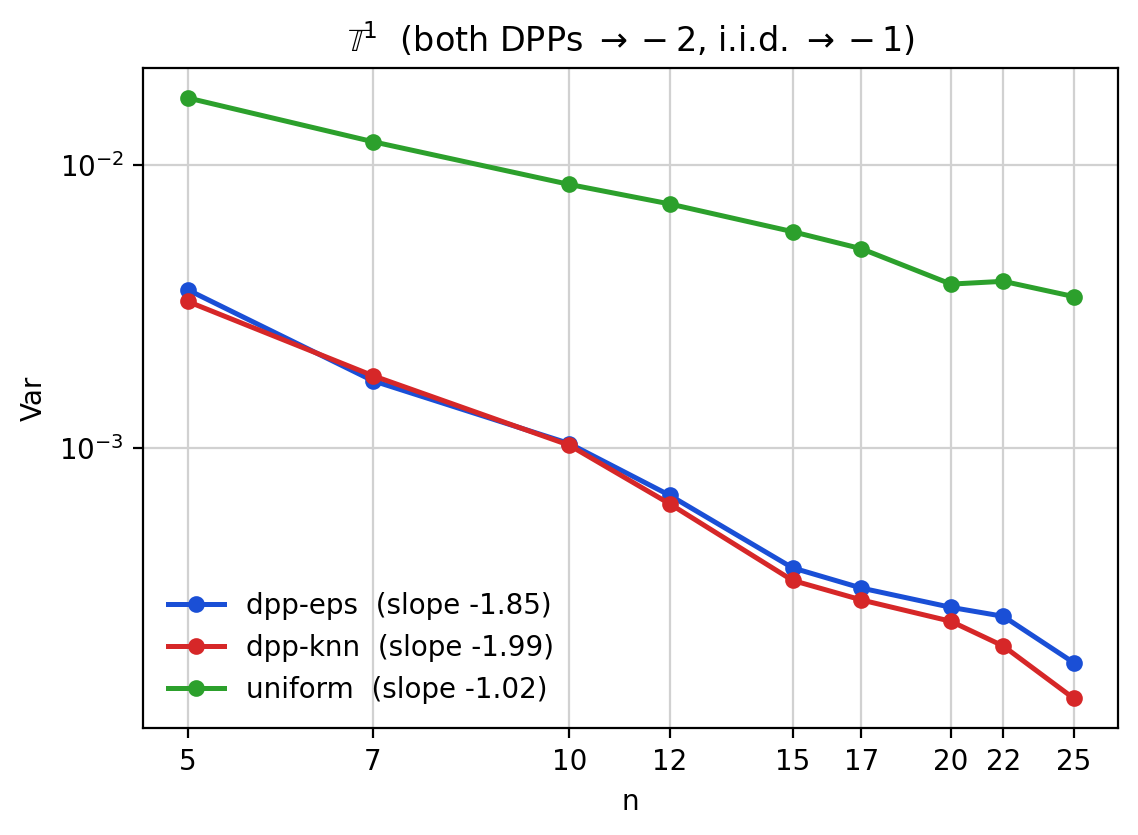}
        \caption{circle with uniform measure}
    \end{minipage}
    \hfill
    \begin{minipage}{0.48\textwidth}

        \centering
        \includegraphics[width=\linewidth]{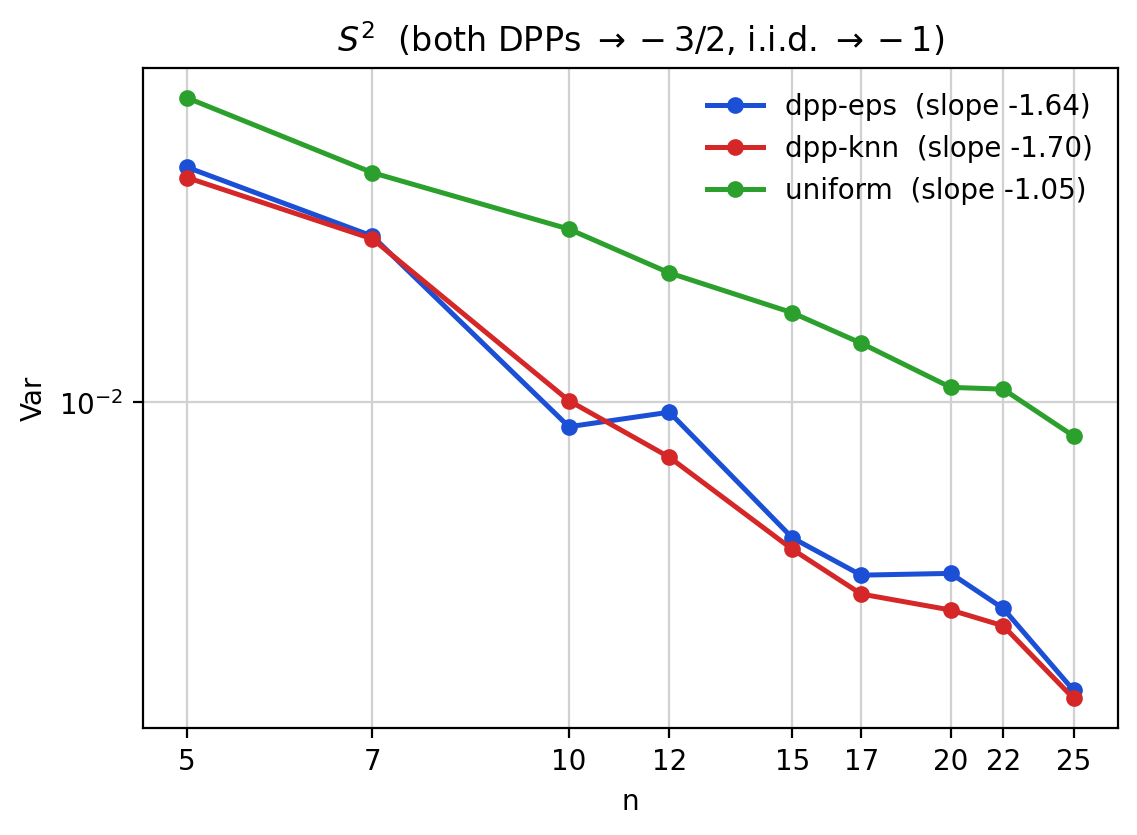}
        \caption{sphere with uniform measure}
    \end{minipage}
\end{figure}

\begin{figure}[H]
    \begin{minipage}{0.48\textwidth}
        \centering
        \includegraphics[width=\linewidth]{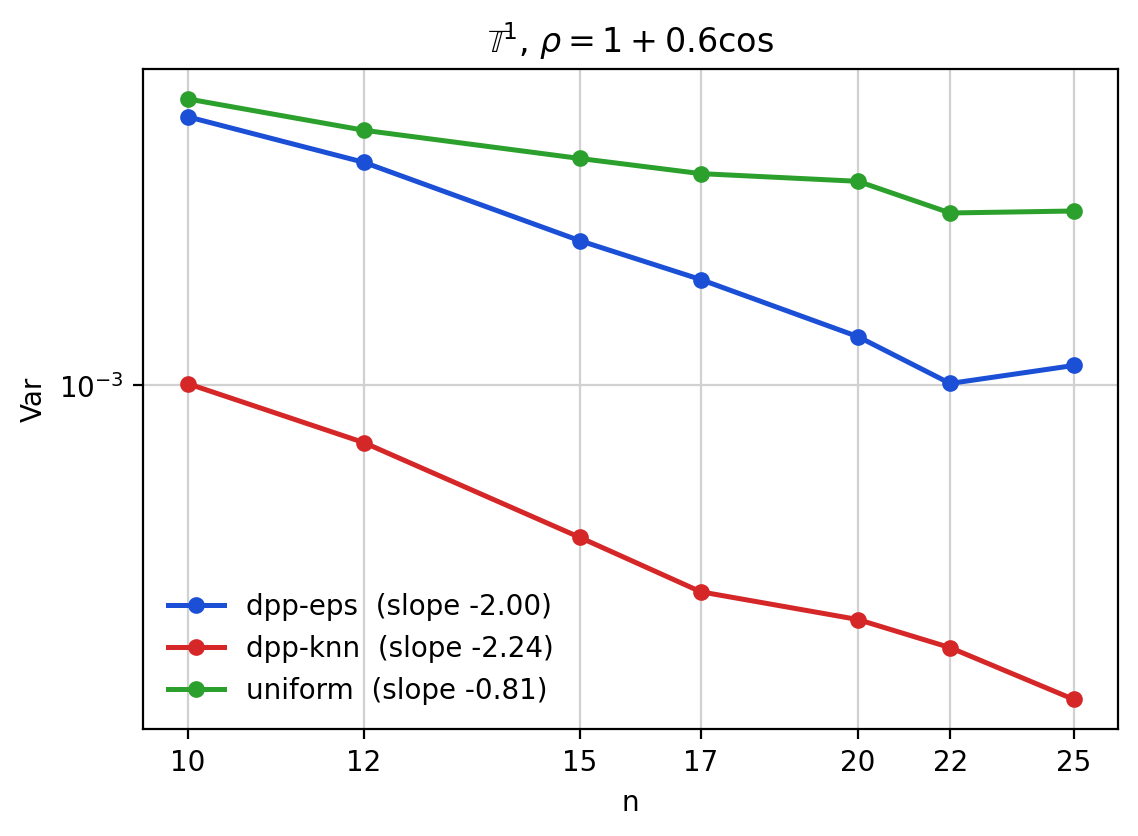}
        \caption{circle with $\rho = 1 + 0.6\cos(2\pi t)$}
    \end{minipage}
    \hfill
    \begin{minipage}{0.48\textwidth}
        \centering
        \includegraphics[width=\linewidth]{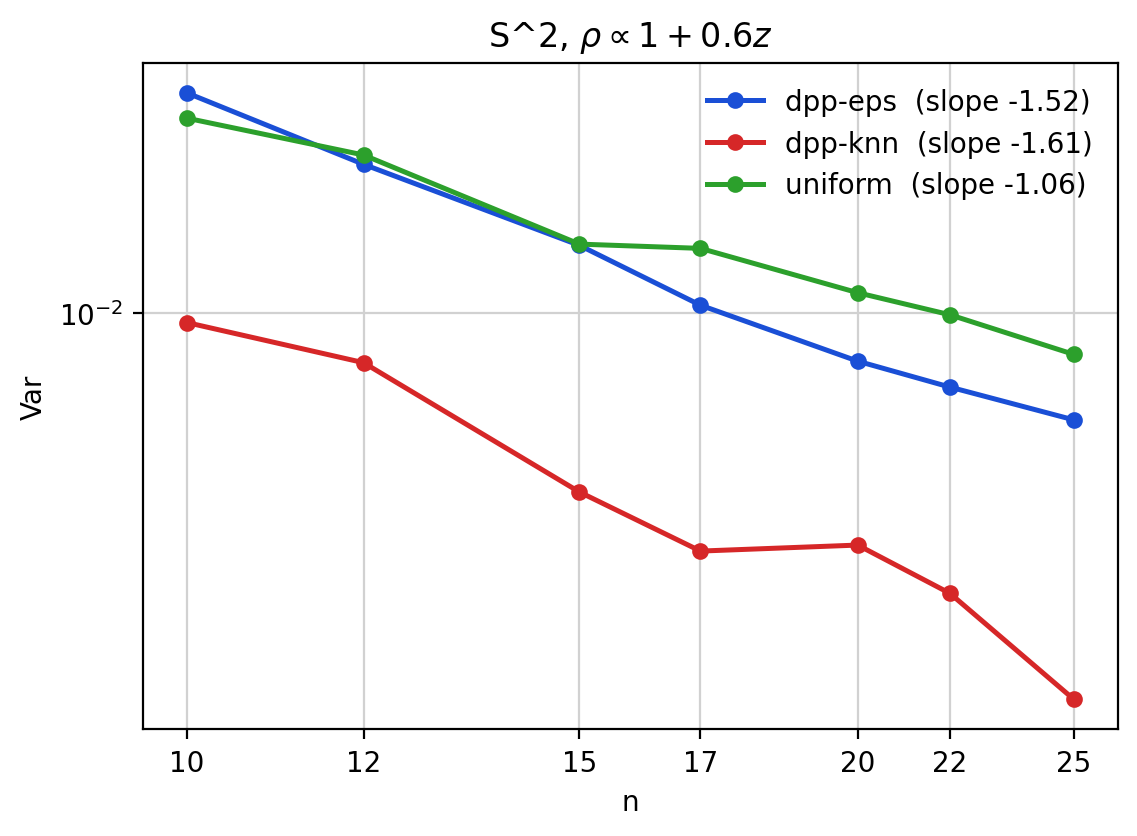}
        \caption{sphere with $\rho \propto 1 + az$}
    \end{minipage}
    \label{fig:variance-decay-four-functions}
\end{figure}

Next, we examine how the variance constant changes with the density parameter
$a\in\{0,0.3,0.6,0.9\}$ . Fitting the model $\mathbb{E}_{\mathcal X}\!\left[\operatorname{Var}\!\left(\mathcal L^g_{n,N}(f)\mid\mathcal X\right)\right] \approx C_g(a)n^{-2}$
on the circle and $\mathbb{E}_{\mathcal X}\!\left[\operatorname{Var}\!\left(\mathcal L^g_{n,N}(f)\mid\mathcal X\right)\right] \approx C_g(a)n^{-3/2}$ on the sphere, for $g\in\{\varepsilon,k\}$, and plotting $C_g(a)/C_g(0)$ against $a$, we find that the constant for
$\varepsilon$-graph increases substantially. In contrast, the $k$-NN constants remain roughly the same, supporting the conclusion of Remark \ref{rmk:density-eps-knn}.

\begin{figure}[H]
    \centering
    \begin{minipage}{0.48\textwidth}
        \centering
        \includegraphics[width=\linewidth]{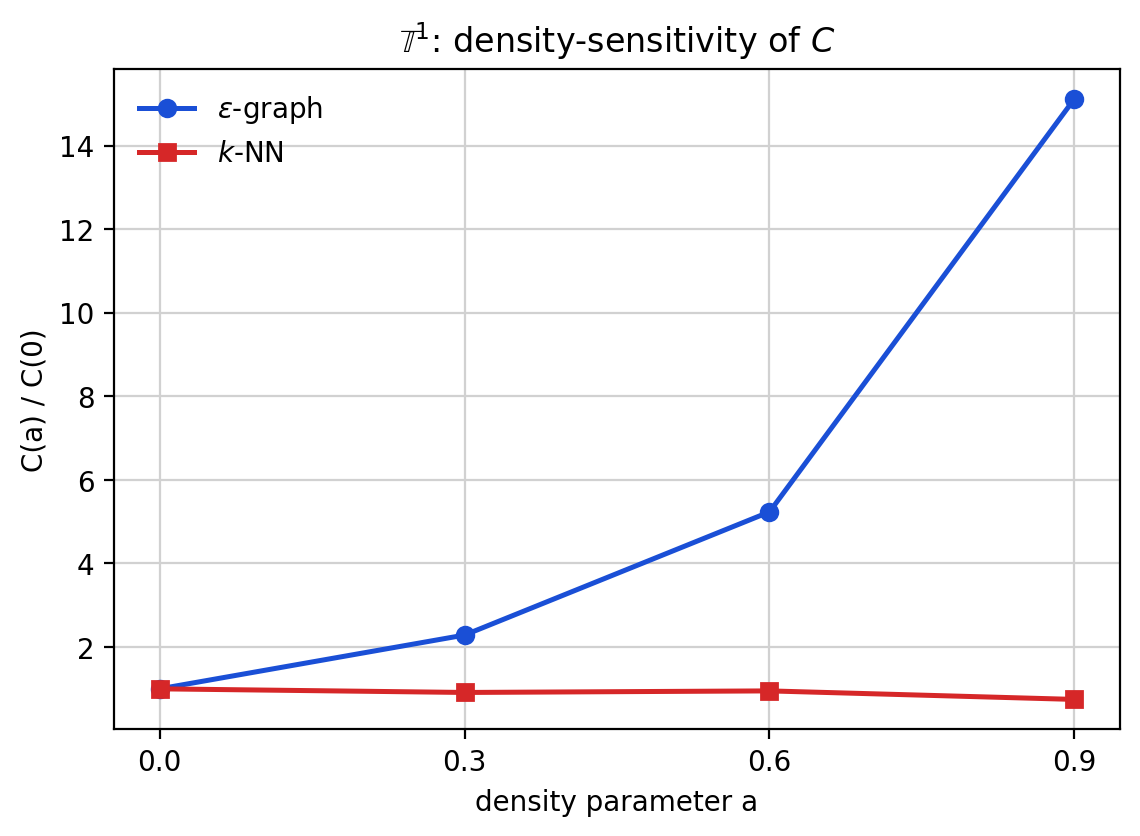}
        \caption{density-sensitivity on the circle}
    \end{minipage}
    \hfill
    \begin{minipage}{0.48\textwidth}

        \centering
        \includegraphics[width=\linewidth]{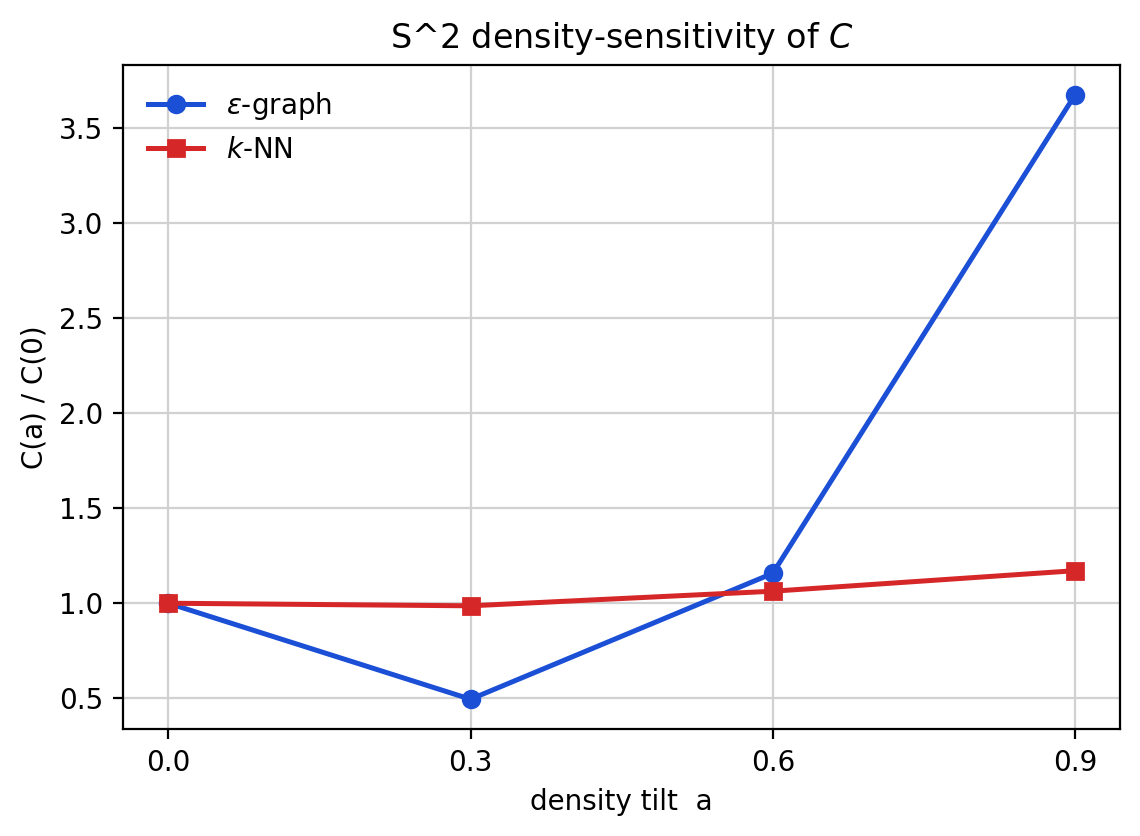}
        \caption{density sensitivity on the sphere}
    \end{minipage}
\end{figure}

\subsection{Minibatch sampling for image classification}

We next consider a classification task with low-dimensional
structure in a high-dimensional ambient space. We select one
image of each digit $4$ and $9$ from the MNIST dataset and generate
$2500$ training images and $750$ test images per class by
independently rotating the selected images through angles
sampled uniformly from $[0,2\pi)$. Rotated images are resampled
to $28\times28$ pixels and represented as vectors in
$\mathbb R^{784}$. Each class therefore samples a one-parameter
rotation family. Training and test images use independent
rotation angles.

For the two graph-based DPPs, we construct a separate graph within each class and sample
independent DPP minibatches of size $16$ from each, giving
a total minibatch size of $32$. The IID baseline uses the
same class allocation. We initialize $k=24$ and choose
$\varepsilon$ as the median distance to the $24$th nearest
neighbour, increasing these parameters if needed for
within-class connectivity.

We train a linear support vector machine using the Pegasos algorithm \cite{shalev2007pegasos}.
Writing $z_i=(x_i,1)$ and $y_i\in\{-1,+1\}$, the objective is
\[
    f(\theta)
    =
    \frac{\lambda}{2}\|\theta\|_2^2
    +
    \frac{1}{N}\sum_{i=1}^{N}
    \bigl(1-y_i\langle\theta,z_i\rangle\bigr)_+,
    \qquad N=5000.
\]

We use the estimator 
\[
    \frac{\lambda}{2}\|\theta\|_2^2
    +
    \frac{1}{n}\sum_{i=1}^{n}
    \bigl(1-y_i\langle\theta,z_i\rangle\bigr)_+
\]
for the IID sampler, and 
\[
    \frac{\lambda}{2}\|\theta\|_2^2
    +
    \frac{1}{N}\sum_{i \in \mathcal{S}_n}
    \frac{\bigl(1-y_i\langle\theta,z_i\rangle\bigr)_+}{\pi_i}
\]
for the two DPP samplers. 

All methods start from $\theta_0=0$ and run for $T=100$
iterations with $\lambda=0.1$, step size
$\eta_t=1/(\lambda t)$. We use the parameter error $\|\theta_t - \theta_{\text{ref}}\|$ as the metric to evaluate the performance of our sampler. The reference parameter $\theta_{\mathrm{ref}}$ is obtained by solving the full-data SVM dual using L-BFGS-B, with a verified primal-dual gap of at most $10^{-6}$. The parameter error is plotted on the $y-$axis, against the number of iterations on the $x-$axis. The performance is averaged over 100 independent trials for each sampler, and we plot error bars corresponding to $\pm 2$ standard errors.

We see a considerable and statistically significant improvement over the IID baseline with each of the two DPP samplers. 

\begin{figure}[H]
    \centering
    \begin{minipage}{0.48\textwidth}
        \centering
        \includegraphics[width=\linewidth]{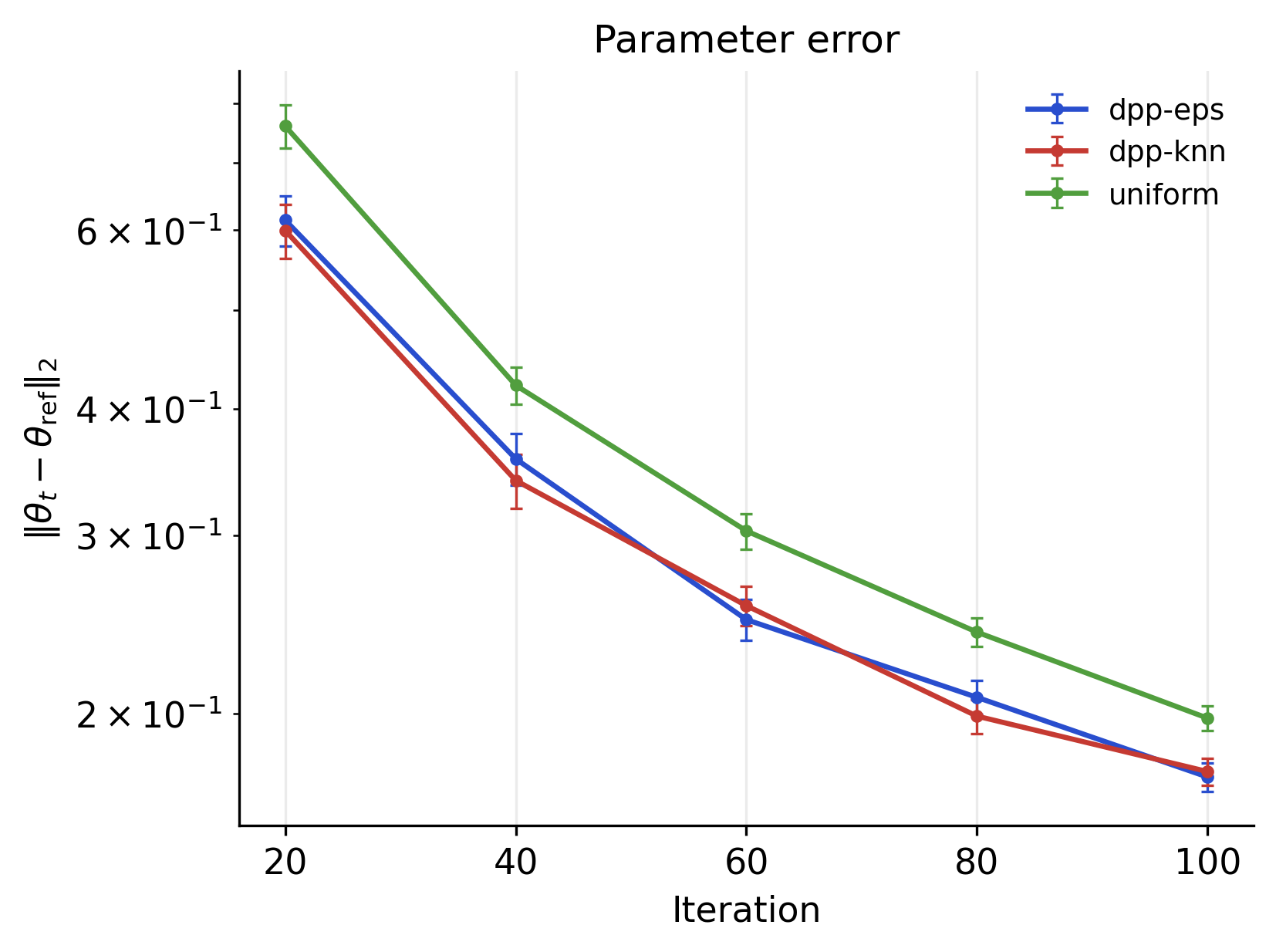}
        \caption{Parameter error for Pegasos}
    \end{minipage}
\end{figure}

\section{Main ideas and key proof ingredients} \label{sec:proof-idea}
Let $\L$ be an operator on $L^2(\mathcal X,\mu)$ satisfying assumptions (A1)--(A3), and let $\{\phi_j\}_j$ be an orthonormal eigenbasis, with the corresponding eigenvalues $\lambda_j\geq 0$ arranged in nondecreasing order. Denote by $\mathcal S_n$ the associated rank-$n$ spectral DPP on $\mathcal X$, and let $f:\mathcal{X} \to \R$ be a bounded observable. A traditional approach to controlling the variance of linear statistics
uses the exact formula
\[
    \var[\mathcal S_n(f)]
    =\frac12\iint_{\mathcal X^2}
    |f(x)-f(y)|^2|\K_n(x,y)|^2\,\mu(\d x)\mu(\d y),
\]
where $\K_n(x,y)=\sum_{j=1}^n\phi_j(x)\phi_j(y)$ is the kernel of
$\mathcal S_n$. A careful asymptotic analysis of the kernel may yield
the desired estimates. However, this approach often relies on additional
differential structure on $\mathcal X$ and detailed information about
$\K_n$, making it difficult to apply in discrete settings such as
point clouds or networks.

We take a different route.
Our starting point is the following operator identity for the variance: 
\begin{equation}
    \var[\mathcal{S}_n(f)]
    =\bigl\|(\mathrm{I}-\Pi_n)\M_f \Pi_n\bigr\|_{\mathrm{HS}}^2,
    \label{eq:intro-variance-identity}
\end{equation}
where $\Pi_n$ is the orthogonal projection onto $\Span\{\phi_1,\ldots,\phi_n\}$, $\mathrm I$ is the identity on $L^2(\mu)$, $\M_f$ is multiplication by $f$, and $\|\cdot\|_{\mathrm{HS}}$ denotes the Hilbert--Schmidt norm. Thus, the variance quantifies the extent to which multiplication by $f$ carries functions outside the selected low-frequency subspace. Variance analysis based on spectral-projector commutators and
semiclassical estimates was investigated on Euclidean spaces by Deleporte and
Lambert~\cite{DeleporteLambert2024}.

Expanding \eqref{eq:intro-variance-identity}, we obtain
\begin{equation} \label{eq:intro-variance-expansion}
    \var[\mathcal{S}_n(f)] = \sum_{k=1}^n \sum_{j>n} \big |\langle \phi_j,f\phi_k\rangle \big |^2.
\end{equation}
These interactions can be controlled through the \emph{commutators}  $[\L^{\alpha},\M_f]:= \L^{\alpha} \M_f - \M_f \L^{\alpha}$, where $\alpha>0$ is a suitably chosen parameter. To elaborate, using $\L^{\alpha} \phi_j = \lambda_j^{\alpha} \phi_j$, for $\lambda_j\ne \lambda_k$  we have 
\begin{equation} \label{eq:intro-variance-commutator}
    |\langle \phi_j,f\phi_k\rangle | = \frac{1}{|\lambda_j^{\alpha}-\lambda_k^{\alpha}|}\left | \langle \phi_j, [\L^{\alpha},\M_f] \phi_k \rangle \right |.
\end{equation}
Thus, if $\lambda_k$ lies well below the cutoff $\lambda_n$, it is separated from every $\lambda_j$ with $j>n$. This suggests splitting the sum according to whether $\lambda_k$ lies near the cutoff or farther below it: the first contribution is controlled by counting eigenvalues in a short spectral interval, and the second by commutator estimates. 

This commutator approach requires an $L^2(\mu)$ operator-norm bound for $[\L^\alpha,\M_f]$. When $\L$ is the Laplace--Beltrami operator on a smooth compact Riemannian manifold, the natural exponent is $\alpha=1/2$, and classical commutator estimates from pseudodifferential analysis yield
\[
    \|[\L^{1/2},\M_f]\|_{L^2\to L^2}
    \lesssim \|\nabla f\|_\infty,
    \qquad f\in C^1.
\]
For general operators on more abstract spaces, analogous estimates are harder to establish because tools such as pseudodifferential calculus are not readily available. Using Dirichlet form theory, we prove that every admissible observable $f$ satisfies
\begin{equation} \label{eq:commutator-alpha}
    \|[\L^\alpha,\M_f]\|_{L^2\to L^2}
    \lesssim
    \frac{1}{1-2\alpha}
    \bigl(\|f\|_\infty^2+\|\Gamma(f)\|_\infty\bigr)^{1/2},
    \qquad 0<\alpha<1/2.
\end{equation}
Here $\Gamma(f)$ is the energy density of $f$, generalizing the squared gradient $|\nabla f|^2$ associated with the Laplace--Beltrami operator. This estimate applies across continuous and discrete settings and requires neither pseudodifferential nor Fourier integral operator calculus, making it of independent interest.

However, \eqref{eq:commutator-alpha} does not cover the endpoint $\alpha=1/2$, and choosing $\alpha$ close to this endpoint introduces a logarithmic loss in the resulting variance bound. To avoid this loss, we estimate the variance directly using the associated Dirichlet form. Specifically, we prove that
\[
    \var[\mathcal{S}_n(f)]
    \le
    \big (\|f\|_\infty^2
    +\frac{\pi^2}{6}\|\Gamma(f)\|_\infty\big )
    \mathcal R_{\L}(n),
\]
where $\mathcal R_{\L}(n)$ is the local spectral counting function associated to $\L$, introduced in Definition~\ref{def:loc-spec-counting}. Thus, for each fixed admissible observable $f$, this bound yields variance reduction whenever $\mathcal R_{\L}(n)$ grows sublinearly in $n$. The remaining task is therefore to establish such growth estimates in the concrete settings considered below.

On manifolds, the required spectral estimates follow from Weyl's law with a remainder and local spectral-function estimates from ~\cite{hormander1968spectral,rivera2022weighted,huang2022pointwise}, which yield
$\mathcal R_{\L}(n)\lesssim n^{1-1/m}$.
Local Weyl estimates also control the diagonal of the spectral kernel, providing bounds on the weights used in the estimator in Theorem~\ref{thm:monte-carlo-manifold}.
For random geometric graphs, quantitative comparisons between graph Laplacians and continuum diffusion operators~\cite{calder2022improved} transfer the required spectral estimates, with explicit approximation errors as the sample size grows. On bounded-degree networks, stronger commutator estimates combine with the counting law for the confining potential to yield~\eqref{eq:intro-network-variance}. Together, these applications illustrate how spectral geometry and Dirichlet forms govern determinantal fluctuations, suggesting a route to variance reduction beyond the sampling models considered here.

\section{Acknowledgements and AI Statement}
\subsection{AI Statement}
In this work, we have used standard generative AI tools to support our work. The authors supplied the research directions and detailed proof ideas, and standard generative AI tools were used for exploratory simulations, literature search, calculations,  the development and simplification of some parts of the proofs and typesetting. In particular, the authors proposed the central ideas of constructing kernels from Markov semigroup generators (including diffusion generators and discrete Schr\"odinger operators), using diffusion geometry and Dirichlet form theory to derive variance bounds, and drawing on the relevant analytical literature. The authors have verified all steps, and take responsibility for the final content of this work.
\subsection{Acknowledgements}
Subhro Ghosh was supported in part  the NUS Dean’s Chair Associate Professorship E-146-00-0037-01 and the Singapore MOE grants A-8002014-00-00 and A-8003802-00-00. Hoang Son Tran was supported by the Singapore Ministry of Education grants A-8003576-00-00 and A-8002488-00-00.

\bibliographystyle{plain} 
\bibliography{bibliography,stats} 

\newpage 

\appendix

\section{Background} \label{app:general}
\subsection{Background on determinantal point processes} \label{subsec:background-dpp}

Let $E$ be a Polish space equipped with a reference measure $\mu$. A simple point process $\mathcal{S}$ on $E$ is a
random locally finite subset of $E$. We say that $\mathcal{S}$ is
a \emph{determinantal point process} (DPP) with respect to $\mu$ if there exists a
measurable function $K:E^2\to \C$ such that, for
every $k\in \N$ and every nonnegative measurable function
$f:E^k\to \R$,
\begin{equation*}
    \E\left[
        \sum_{\neq} f(x_{i_1},\ldots,x_{i_k})
    \right]
    =
    \int_{E^k}
        f(x_1,\ldots,x_k)
        \det\!\bigl[K(x_i,x_j)\bigr]_{1\le i,j\le k}
        \,\mathrm d\mu^{\otimes k}(x_1,\ldots,x_k).
\end{equation*}
Here the sum on the left-hand side is over all ordered $k$-tuples of pairwise
distinct points of $\mathcal{S}$. The function $K$ is called a
\emph{kernel} of the process, and $\mu$ is called the
\emph{reference measure}; see, for instance, \cite{Mac75,Lyo03,HKPV06}.

A particularly important class is given by \emph{projection DPPs}. Let
$n\ge 1$, and let $\phi_1,\ldots,\phi_n$ be orthonormal functions in
$L^2(\mu)$. Define
\begin{equation*}
    K(x,y)
    :=
    \sum_{k=1}^n \phi_k(x)\overline{\phi_k(y)} .
\end{equation*}
As an integral kernel, $K$ represents the orthogonal projection from
$L^2(\mu)$ onto the subspace $\mathrm{span}\{\phi_1,\ldots,\phi_n\}$. The DPP associated with $(K,\mu)$ is called a \emph{projection DPP} of
rank $n$.

Projection DPPs have a useful finite-dimensional probabilistic representation.
If \(\mathcal{S}\sim \mathrm{DPP}(K,\mu)\), where \(K\) is a projection kernel
of rank \(n\), then \(\mathcal{S}\) contains exactly \(n\) points almost surely.
Moreover, writing \(\mathcal{S}=\{X_1,\ldots,X_n\}\), the points may be sampled
by first drawing an exchangeable ordered \(n\)-tuple \((X_1,\ldots,X_n)\) with
density
\begin{equation*}
    p(x_1,\ldots,x_n)
    =
    \frac{1}{n!}
    \det \bigl(K(x_i,x_j)\bigr)_{1\le i,j\le n}
\end{equation*}
with respect to \(\mu^{\otimes n}\), and then forgetting the order.

For further background
on DPPs, we refer to \cite{HKPV06,Lyo03,KuTa12}.

\subsection{Background on Dirichlet form theory}
\label{subsec:Dirichlet-space}
\subsubsection{Basic notions} 

We briefly recall basic notions in Dirichlet form theory and Markov semigroup theory. For a detailed discussion, we refer to the books \cite{bakry2014analysis, fukushima2011dirichlet, bouleau2010dirichlet}.

Let $\mathcal X$ be a locally compact separable metric space equipped with a positive Radon measure $\mu$ of full support. We work with real-valued functions and denote by $L^2(\mu)$ the Hilbert space of square-integrable functions, equipped with the inner product 
\[ \langle f,g\rangle := \int_{\mathcal X} f g \, \d\mu . \] 
We denote by $\|\cdot\|_2$ the corresponding norm.

Let $\L$ be a densely defined, non-negative self-adjoint operator on $L^2(\mu)$ such that the semigroup $\rmP_t := e^{-t\L}, t\ge 0$ is sub-Markovian: for every $f\in L^2(\mu)$ satisfying $0\le f(x)\le 1$ for $\mu$-a.e. $x\in \mathcal{X}$, one has 
\[ 0\le \rmP_t f(x)\le 1, \qquad \mu\text{-a.e. } x\in \mathcal{X}, \text{ for every }t\ge 0. \] 

Let $\mathcal D(\L)$ and $\mathcal D(\L^{1/2})$ denote the domains of $\L$ and $\L^{1/2}$, respectively. Since $\L$ is non-negative and self-adjoint, one has 
$\mathcal D(\L)\subset \mathcal D(\L^{1/2})$. 
The symmetric Dirichlet form associated with $\L$ is given by
\[ \mathcal{E}(f,g) := \langle \L^{1/2}f,\L^{1/2}g\rangle, \qquad \mathcal D(\mathcal E):=\mathcal D(\L^{1/2}). \] 
In particular, for $f,g\in \mathcal D(\L)$, one has
$\mathcal{E}(f,g) = \langle f,\L g\rangle = \langle \L f,g\rangle $. 
From now on, we only discuss symmetric Dirichlet forms.

\begin{remark} 
Our sign convention is that $\L$ is non-negative. Equivalently, the infinitesimal generator of the semigroup is $-\L$, which is non-positive in the usual convention of Markov semigroup theory. 
\end{remark} 

\subsubsection{Regular Dirichlet forms and Beurling--Deny decomposition} \label{app:beurling-deny}

Let $C_c(\mathcal{X})$ denote the space of real-valued continuous functions on $\mathcal X$ with compact support. A Dirichlet form $(\mathcal{E},\mathcal{D}(\mathcal{E}))$ is called \emph{regular} if $\mathcal{D}(\mathcal{E})\cap C_c(\mathcal{X})$ is dense in both $C_c(\mathcal{X})$ with respect to the uniform norm and in $\mathcal{D}(\mathcal{E})$ with respect to the norm 
\[\|f\|_{\mathcal{E}_1}:=(\mathcal{E}(f,f)+\|f\|_2^2)^{1/2}.\]
For $f,g \in \mathcal{D}(\mathcal{E})\cap C_c(\mathcal{X})$, a regular Dirichlet form $\mathcal{E}$ admits the  following Beurling--Deny decomposition
\[\mathcal{E}(f,g)= \mathcal{E}^c(f,g) + \mathcal{E}^j(f,g) + \mathcal{E}^k(f,g),\]
where:
\begin{enumerate}
    \item[(i)] $\mathcal{E}^c$ is called the \emph{strongly local} part, satisfying $\mathcal{E}^c(f,g)=0$ whenever $f$ is constant on a neighborhood of the support of $g$;
    \item[(ii)] $\mathcal{E}^j$ is called the \emph{jump} part, given by
    \[\mathcal{E}^j(f,g) = \frac{1}{2}\iint_{\mathcal{X}^2\setminus \diag}(f(x)-f(y))(g(x)-g(y))J(\d x,\d y) \]
    where $\diag:=\{(x,x):x\in \mathcal{X}\}$ and $J$ is a symmetric positive Radon measure on $\mathcal{X}^2\setminus \diag$, called the \emph{jump measure};
    \item[(iii)] $\mathcal{E}^k$ is called the \emph{killing} part, given by
    \[ \mathcal{E}^k(f,g) = \int_{\mathcal X}f(x)g(x)\kappa(\d x),\]
    where $\kappa$ is a positive Radon measure on $\mathcal{X}$, called the \emph{killing measure}.
\end{enumerate} 
We set $\mathcal{E}^{(0)}:= \mathcal{E}^c + \mathcal{E}^j$, referred to as the \emph{non-killing} part of $\mathcal{E}$.

\begin{remark}
    For simplicity of notation, from now on we identify functions in $\mathcal{D}(\mathcal{E})$ occurring in the Beurling--Deny representation with their quasi-continuous versions. We remark that for $v\in \mathcal{D}(\mathcal{E}) \cap L^{\infty}(\mu)$, its quasi-continuous version satisfies $|v|\le \|v\|_{\infty}$ quasi-everywhere.
\end{remark}

\subsubsection{Energy measures and carr\'e du champ} \label{app:carreduchamp}

For $f,g \in \mathcal{D}(\mathcal{E}) \cap C_c(\mathcal{X})$, the energy measure $\nu_{f,g}$ is defined by
\begin{equation*}
    \int_{\mathcal{X}}h \d\nu_{f,g} := \frac{1}{2} \big ( \mathcal{E}(f,hg) +\mathcal{E}(hf,g) - \mathcal{E}(fg,h) \big ), \quad  h\in \mathcal{D}(\mathcal{E}) \cap C_c(\mathcal{X}).
\end{equation*}
We say that $\mathcal{E}$ admits a carr\'e du champ (with respect to $\mu$) if $\nu_{f,g}\ll \mu$ for all such $f$ and $g$. We define 
\[\Gamma(f,g):= \frac{\d \nu_{f,g}}{\d \mu}\]
and extend this to all $\mathcal{D}(\mathcal{E})$ in the usual way.
When $f=g$, we write $\Gamma(f):=\Gamma(f,f)$ and $\mathcal{E}(f):=\mathcal{E}(f,f)$ for simplicity. It is well-known that $\Gamma$ satisfies
\[ \Gamma(f)\ge 0\quad,\quad|\Gamma(f,g)|^2 \le \Gamma(f)\Gamma(g),\qquad \mu\text{--almost everywhere.}\]

More generally, when $\mathcal{E}$ admits the Beurling--Deny decomposition $\mathcal{E}=\mathcal{E}^c +\mathcal{E}^j+\mathcal{E}^k$, we denote by $\mathcal{E}^{(0)}:=\mathcal{E}^c + \mathcal{E}^j$ the non-killing part of $\mathcal{E}$. 
Although $\mathcal E^{(0)}$, $\mathcal E^c$, and $\mathcal E^j$ need not be closed on $\mathcal D(\mathcal E)$, their mutual energy measures are well defined through the Beurling--Deny decomposition. 
In particular, for $\bullet\in \{(0),c,j\}$ and $f,g\in\mathcal{D}(\mathcal{E})\cap C_c(\mathcal{X})$, let $\nu^\bullet_{f,g}$ be given by
\begin{equation}  \label{eq:energy-measure}
   \int_{\mathcal X}h\d\nu^\bullet_{f,g}
:=
\frac12\big(
\mathcal E^\bullet(f,hg)
+\mathcal E^\bullet(hf,g)
-\mathcal E^\bullet(fg,h)
\big), \quad h\in\mathcal{D}(\mathcal{E})\cap C_c(\mathcal{X}). 
\end{equation}
For these component forms, the functionals on the right-hand side are represented by signed Radon measures, and the diagonal measures $\nu^\bullet_f:=\nu^\bullet_{f,f}$ are nonnegative.
We say that $\mathcal{E}^{\bullet}$ admits a carr\'e du champ (with respect to $\mu$) if $\nu^\bullet_{f,g}\ll \mu$ for all such $f$ and $g$.

\begin{lemma} \label{lm:app-carreduchamp}
    Let $(\mathcal{E},\mathcal{D}(\mathcal{E}))$ be a regular Dirichlet form with the Beurling--Deny decomposition $\mathcal{E}=\mathcal{E}^c+\mathcal{E}^j+\mathcal{E}^k$. If the non-killing part $\mathcal{E}^{(0)}:= \mathcal{E}^c+\mathcal{E}^j$ admits a carr\'e du champ $\Gamma$ with respect to $\mu$, then $\mathcal{E}^{\bullet}$ also admits a carr\'e du champ $\Gamma^{\bullet}$ with respect to $\mu$ for $\bullet \in \{c,j\}$. Moreover, for $f,g \in \mathcal{D}(\mathcal{E})$
    \[ \Gamma(f,g) = \Gamma^c(f,g) + \Gamma^j(f,g) \qquad \mu\text{--almost everywhere}\]
    and 
    \[ \mathcal{E}^{(0)}(f,g)=\int \Gamma(f,g) \d\mu,\quad \mathcal{E}^{\bullet}(f,g) =\int \Gamma^{\bullet}(f,g)\d\mu, \quad \bullet \in \{c,j\}. \]
\end{lemma}

\begin{proof}
    Let $f, g \in \mathcal{D}(\mathcal{E}) \cap C_c(\mathcal{X})$, by linearity 
    \[\nu^{(0)}_{f,g}=\nu^c_{f,g}+\nu^j_{f,g}.\]
    In particular, $\nu^{(0)}_f = \nu^c_f+\nu^j_f$. Since these measures are nonnegative, and $\nu^{(0)}_f \ll \mu$, it follows that $\nu^{c}_f$ and $\nu^j_f$ are absolutely continuous with respect to $\mu$.

     Let $\bullet \in\{c,j\}$. 
    By polarization, we have 
    \[ \nu^{\bullet}_{f,g} = \frac{1}{4}(\nu^{\bullet}_{f+g}-\nu^{\bullet}_{f-g}).\]
    This implies $ \nu^{\bullet}_{f,g} \ll \mu$. Thus, $\mathcal{E}^{\bullet}$ admits a carr\'e du champ $\Gamma^{\bullet}$ with respect to $\mu$, where
    \[\Gamma^{\bullet}(f,g) = \frac{\d \nu^\bullet_{f,g}}{\d\mu} \cdot\]
    By linearity, we have $\Gamma = \Gamma^c+\Gamma^j$. It remains to prove the integral identities.
    
     For $\mathcal{E}^c$, in the definition of the energy measure in \eqref{eq:energy-measure}, choosing $h$ equal to $1$ on a neighborhood of the supports of $f$ and $g$ yields
    \[ \int_{\mathcal{X}} \Gamma^c(f,g)\d \mu = \mathcal{E}^c(f,g),\]
    where $\mathcal{E}^c(fg,h)=0$ due to the strong locality of $\mathcal{E}^c$. 
    
    For $\mathcal{E}^j$, using \eqref{eq:energy-measure} and expanding, one has 
    \[ \int_{\mathcal{X}} h\Gamma^j(f,g)\d \mu = \frac{1}{2}\iint_{\mathcal{X}^2\setminus\diag}h(x)(f(x)-f(y))(g(x)-g(y))J(\d x, \d y).\]
    This implies
    \[ \int_{\mathcal{X}} \Gamma^j(f,g)\d \mu = \mathcal{E}^j(f,g).\]
    By linearity, we also obtain
    \[ \int_{\mathcal{X}} \Gamma(f,g)\d \mu = \mathcal{E}^{(0)}(f,g),\]
    as desired.

    Finally, since $\mathcal{D}(\mathcal{E}) \cap C_c(\mathcal{X})$ is dense in $\mathcal{D}(\mathcal{E})$ in $\mathcal{E}_1$-norm, these identities extend by standard approximation from $\mathcal{D}(\mathcal{E}) \cap C_c(\mathcal{X})$ to $\mathcal{D}(\mathcal{E})$. Hence all the identities hold for arbitrary $f,g\in \mathcal{D}(\mathcal{E})$. The lemma follows.
\end{proof}

\section{Proofs from Subsection \ref{sec:general}} \label{proof:general}
Throughout this section, let $\L$ be an operator satisfying Assumptions (A1)--(A3) in subsection \ref{sec:general}. Let $\{\phi_j\}_j$ be an orthonormal basis of $L^2(\mu)$ consisting of eigenfunctions of $\L$, ordered by nondecreasing eigenvalue $\lambda_j\ge 0$. 
Let 
\[H_n:= \Span \{\phi_1,\ldots,\phi_n\} \subset L^2(\mu).\] Let $H_n^\perp$ be the orthogonal complement of $H_n$ in $L^2(\mu)$.

\subsection{Proof of Theorem \ref{thm:var-reduce-general}}
We begin with some technical lemmas.

\begin{lemma} \label{lm:bounded-multiplier}
    Let $f\in \mathcal{A}$. Then for every $u\in \mathcal{D}(\mathcal{E})$, we have  $fu \in \mathcal{D}(\mathcal{E})$, and
    \[\mathcal{E}(fu) \le 2 \|f\|^2_{\infty} \mathcal{E}(u) + 2\|\Gamma(f)\|_{\infty} \|u\|_2^2.\]
\end{lemma}

\begin{proof}
  We first prove the estimate for $u\in \mathcal{D}(\mathcal{E})\cap L^{\infty}(\mu)$. 
  Since $\mathcal D(\mathcal E)\cap L^\infty(\mu)$ is an algebra, and $f,u\in\mathcal D(\mathcal E)\cap L^\infty(\mu)$, we have $fu\in\mathcal D(\mathcal E)$. 
  
  By the Beurling--Deny decomposition (cf. Appendix \ref{app:beurling-deny}), the Dirichlet form $\mathcal{E}$ admits the following decomposition $\mathcal{E}=\mathcal{E}^c+\mathcal{E}^j+\mathcal{E}^k$, where $\mathcal{E}^c$ denotes the strongly local part, $\mathcal{E}^j$ denotes the jump part, and $\mathcal{E}^k$ denotes the killing part of $\mathcal{E}$. Let $\mathcal{E}^{(0)}:= \mathcal{E}^c + \mathcal{E}^j$ be the non-killing part of $\mathcal{E}$. By (A2), $\mathcal{E}^{(0)}$ admits a carr\'e du champ $\Gamma(\cdot,\cdot)$ with respect to $\mu$. Consequently, $\mathcal{E}^c$ (respectively, $\mathcal{E}^j)$ also admits a carr\'e du champ $\Gamma^c$ (respectively, $\Gamma^j$) with respect to $\mu$. Moreover, we have
  \[ \Gamma(f,g) = \Gamma^c(f,g) + \Gamma^j(f,g),\qquad \text{$\mu$-a.e.},\forall f,g \in \mathcal{D}(\mathcal{E})\]
  and 
    \[ \mathcal{E}^{(0)}(f,g)=\int \Gamma(f,g) \d\mu,\quad \mathcal{E}^{\bullet}(f,g) =\int \Gamma^{\bullet}(f,g)\d\mu, \quad \bullet \in \{c,j\}. \]
   We refer to Lemma \ref{lm:app-carreduchamp} for this fact and to Appendix \ref{app:carreduchamp} for more properties of carr\'e du champ operators.

   Returning to the proof, we first consider the strongly local part $\mathcal{E}^c$. Due to the strongly local property, $\Gamma^c$ satisfies the Leibniz rule
  $\Gamma^c(fu,v) = f \Gamma^c(u,v)+u\Gamma^c(f,v)$. 
  In particular, applying this twice yields
  \begin{eqnarray*}
      \Gamma^c(fu) = f^2\Gamma^c(u) + u^2\Gamma^c(f)+2fu\Gamma^c(f,u) \le 2f^2 \Gamma^c(u) + 2u^2 \Gamma^c(f) 
  \end{eqnarray*}
  where we used the Cauchy--Schwarz inequality $|\Gamma^c(f,u)|^2\le \Gamma^c(f)\Gamma^c(u)$. 
  Integrating both sides gives
  \[\mathcal{E}^c(fu)=\int_{\mathcal{X}} \Gamma^c(fu)\d \mu \le 2 \|f\|^2_{\infty} \mathcal{E}^c(u)  +2\int_{\mathcal{X}}|u(x)|^2\Gamma^c(f)(x)\mu(\d x). \]
  Next, let us consider the jump term $\mathcal{E}^j$. We have
  \[ \mathcal{E}^j(v) = \frac{1}{2}\iint_{\mathcal{X}^2\setminus \mathrm{diag}} |v(x)-v(y)|^2 J(\d x,\d y),\]
  where $J$ is the symmetric jump measure. We note that
  \[|f(x)u(x)-f(y)u(y)|^2 \le 2|f(x)|^2|u(x)-u(y)|^2 + 2|u(y)|^2|f(x)-f(y)|^2.\]
  Thus,
  \begin{eqnarray*}
  \mathcal{E}^j(fu) &\le& 2\|f\|^2_{\infty} \mathcal{E}^j(u) + \iint_{\mathcal{X}^2\setminus \mathrm{diag}} |u(x)|^2|f(x)-f(y)|^2J(\d x,\d y)   \\
  &=& 2\|f\|^2_{\infty} \mathcal{E}^j(u) +2\int_{\mathcal{X}}|u(x)|^2\Gamma^j(f)(x)\mu(\d x).
  \end{eqnarray*}
  Finally, for the killing term $\mathcal{E}^k$, we have
  \[ \mathcal{E}^k(fu) = \int|\tilde{f} \tilde{u}|^2\d \kappa  \le \|f\|_{\infty}^2 \mathcal{E}^k(u),\]
  where $\kappa$ is the killing measure, and $\tilde{f}$ (respectively, $\tilde{u}$) denotes the quasi-continuous version of $f$ (respectively, $u$). We remark that for a function $v\in \mathcal{D}(\mathcal{E}) \cap L^{\infty}(\mu)$, its quasi-continuous version satisfies $|\tilde{v}|\le \|v\|_{\infty}$ quasi-everywhere.

  Combining all the ingredients above, we have
  \begin{eqnarray*}
       \mathcal{E}(fu) &=& \mathcal{E}^c(fu) + \mathcal{E}^j(fu) + \mathcal{E}^k(fu) \\
       &\le& 2\|f\|_{\infty}^2 \mathcal{E}(u) + 2\int_{\mathcal{X}} |u(x)|^2 \Gamma(f)(x)\mu(\d x) \\
       &\le&  2\|f\|_{\infty}^2 \mathcal{E}(u) + 2\|\Gamma(f)\|_{\infty} \|u\|_2^2
  \end{eqnarray*}
  for every $u\in\mathcal{D}(\mathcal{E})\cap L^{\infty}(\mu)$, as desired. 
  
  For a general $u\in\mathcal D(\mathcal E)$, let $u_n:=(-n)\vee(u\wedge n)$. Then $u_n\in\mathcal D(\mathcal E)\cap L^\infty(\mu)$ and
$u_n\to u$ in the $\mathcal E_1$-norm. Applying the preceding
estimate to $u_n-u_m$, we see that $(fu_n)_n$ is Cauchy in the
$\mathcal E_1$-norm. Since $fu_n\to fu$ in $L^2(\mu)$, the closedness
of the form implies that $fu\in\mathcal D(\mathcal E)$. Passing to
the limit gives
\[
\mathcal E(fu)
\le
2\|f\|_\infty^2\mathcal E(u)
+
2\|\Gamma(f)\|_\infty\|u\|_2^2.
\]
\end{proof}

\begin{lemma} \label{lm:LMf-bound}
    Let $f\in \mathcal{A}$ and $\mathcal{B}_f(w_1,w_2):=\mathcal{E}(fw_1,w_2)-\mathcal{E}(w_1,fw_2)$.
    Then 
    \[ |\mathcal{B}_f(w_1,w_2)| \le \|\Gamma(f)^{1/2}\|_{\infty}(\|w_1\|_2 \|\L^{1/2}w_2\|_2 + \|w_2\|_2 \|\L^{1/2}w_1\|_2 ),\quad \forall w_1,w_2\in\mathcal{D}(\mathcal{E}).\]
\end{lemma}

\begin{proof} 
    Since $f\in \mathcal{A}$, $fu\in \mathcal{D}(\mathcal{E})$ for every $u\in\mathcal{D}(\mathcal{E})$ by Lemma \ref{lm:bounded-multiplier}. Thus, $\mathcal{B}_f$ is well-defined.
    We claim that for $w_1,w_2\in \mathcal{D}(\mathcal{E})\cap L^{\infty}(\mu)$, one has
    \begin{eqnarray*}
       \mathcal{B}_f(w_1,w_2)
       = \int \big ( w_1\Gamma(f,w_2) -  w_2\Gamma(f,w_1) \big )\d\mu.
    \end{eqnarray*}
    
    Indeed, we write $\mathcal{E}=\mathcal{E}^c+\mathcal{E}^j+\mathcal{E}^k$ for the Beurling--Deny decomposition. For the strongly local term, we have 
    \[\mathcal{E}^c(fw_1,w_2)-\mathcal{E}^c(w_1,fw_2)=\int \big (\Gamma^c(fw_1,w_2)-\Gamma^c(w_1,fw_2)\big)\d\mu .\]
    By the Leibniz rule, one has $\Gamma^c(fw_1,w_2)=f\Gamma^c(w_1,w_2)+w_1\Gamma^c(f,w_2)$. Thus,
    \[\mathcal{E}^c(fw_1,w_2)-\mathcal{E}^c(w_1,fw_2)=\int \big (w_1\Gamma^c(f,w_2)-w_2\Gamma^c(f,w_1)\big)\d\mu .\]  
    For the jump term, expanding $ (f(x)w_1(x)-f(y)w_1(y))(w_2(x)-w_2(y))$ yields
    \begin{eqnarray*}
        w_1(x)(f(x)-f(y))(w_2(x)-w_2(y)) + f(y)(w_1(x)-w_1(y))(w_2(x)-w_2(y)),
    \end{eqnarray*}
    which implies
    \begin{eqnarray*}
        \mathcal{E}^j(fw_1,w_2)-\mathcal{E}^j(w_1,fw_2) &=& \frac{1}{2}\iint_{\mathcal{X}^2\setminus\mathrm{diag}}w_1(x)(f(x)-f(y))(w_2(x)-w_2(y))J(\d x,\d y)\\
        &-& \frac{1}{2}\iint_{\mathcal{X}^2\setminus\mathrm{diag}}w_2(x)(f(x)-f(y))(w_1(x)-w_1(y))J(\d x,\d y) \\
        &=& \int \big (w_1\Gamma^j(f,w_2)-w_2\Gamma^j(f,w_1)\big)\d\mu .
    \end{eqnarray*}
    Finally for the killing term, we have $\mathcal{E}^k(fw_1,w_2)-\mathcal{E}^k(w_1,fw_2) = 0$. Since $\Gamma = \Gamma^c+\Gamma^j$, the claim follows.
    
   Thus, for every $w_1,w_2\in \mathcal{D}(\mathcal{E}) \cap L^{\infty}(\mu)$, we have
    \begin{eqnarray*}
        |  \mathcal{B}_f(w_1,w_2) |&=& \Big |\int  \big (w_1 \Gamma(f,w_2) - w_2\Gamma(f,w_1) \big )\d \mu\Big |\\
        &\le& \int \big ( |w_1| \Gamma(f)^{1/2} \Gamma(w_2)^{1/2} + |w_2|\Gamma(f)^{1/2}\Gamma(w_1)^{1/2} \big )\d\mu \\
        &\le& \|\Gamma(f)^{1/2}\|_{\infty}\Big (\|w_1\|_2 \big (\int \Gamma(w_2)\d\mu\big )^{1/2} + \|w_2\|_2 \big (\int \Gamma(w_1)\d\mu\big )^{1/2}\Big ) \\
        &=& \|\Gamma(f)^{1/2}\|_{\infty}\Big (\|w_1\|_2 \mathcal{E}^{(0)}(w_2)^{1/2} + \|w_2\|_2 \mathcal{E}^{(0)}(w_1)^{1/2}\Big ) \\
        &\le& \|\Gamma(f)^{1/2}\|_{\infty}\Big (\|w_1\|_2 \mathcal{E}(w_2)^{1/2} + \|w_2\|_2 \mathcal{E}(w_1)^{1/2}\Big ) \\
        &=& \|\Gamma(f)^{1/2}\|_{\infty}\Big (\|w_1\|_2 \|\L^{1/2}w_2\|_2 + \|w_2\|_2 \|\L^{1/2}w_1\|_2\Big ).
    \end{eqnarray*}
    By a standard truncation argument, this inequality extends to all $w_1,w_2\in \mathcal{D}(\mathcal{E})$. The lemma follows.
\end{proof}

\begin{lemma} \label{lm:far-lambda-k}
    For $k\le n$ such that $\lambda_k<\lambda_n$, 
    \[ \|\proj_{H_n^{\perp}}(f\phi_k)\|_2 \le \frac{\|\Gamma(f)\|_{\infty}^{1/2}}{\sqrt{\lambda_n}-\sqrt{\lambda_k}} \cdot \]
\end{lemma}

\begin{proof}
 Let $k\le n$ be such that $\lambda_k<\lambda_n$. For $j>n$, one has $\lambda_j \ge \lambda_n > \lambda_k$. Thus, we can define 
 \begin{equation}
     \psi_k := \sum_{j>n} \frac{\langle f\phi_k,\phi_j\rangle}{\lambda_j -\lambda_k} \phi_j.
 \end{equation}
 It is easy to check that $\psi_k \in \mathcal{D}(\L)$ and
 $(\L - \lambda_k) \psi_k = \proj_{H_n^{\perp}}(f\phi_k)$. Moreover, one has 
\[\|\psi_k\|_2 \le \frac{1}{\lambda_n-\lambda_k} \|\proj_{H_n^\perp}(f\phi_k)\|_2,\]
 and
 \[\|\L^{1/2}\psi_k\|_2^2= \sum_{j>n}\frac{\lambda_j}{(\lambda_j-\lambda_k)^2}|\langle f\phi_k,\phi_j\rangle|^2 \le \frac{\lambda_n}{(\lambda_n-\lambda_k)^2}\|\proj_{H_n^\perp}(f\phi_k)\|_2^2.\]
 Here we used the fact that the function $x\mapsto x/(x-\lambda_k)^2$ is nonincreasing on $(\lambda_k,\infty)$.

 By Lemma \ref{lm:bounded-multiplier}, both $f\phi_k$ and $f\psi_k$ belong to the form domain. Thus
 \begin{eqnarray*}
     \mathcal{E}(f\phi_k,\psi_k) -\mathcal{E}(\phi_k,f\psi_k) &=& \langle f\phi_k , \L \psi_k\rangle - \langle \L\phi_k ,f\psi_k\rangle \\
     &=& \langle f\phi_k , \L \psi_k\rangle - \langle f\phi_k, \lambda_k \psi_k\rangle \\
     &=&\langle f\phi_k, \proj_{H_n^{\perp}}(f\phi_k)\rangle = \|\proj_{H_n^{\perp}}(f\phi_k)\|_2^2.
 \end{eqnarray*}
 By Lemma \ref{lm:LMf-bound}, 
 \begin{eqnarray*}
     |\mathcal{E}(f\phi_k,\psi_k) -\mathcal{E}(\phi_k,f\psi_k) | &\le& \|\Gamma(f)\|_{\infty}^{1/2}\big (\|\phi_k\|_2\|\L^{1/2}\psi_k\|_2+\|\psi_k\|_2\|\L^{1/2}\phi_k\|_2\big )\\
     &\le&  \|\Gamma(f)\|_{\infty}^{1/2}\Big (\frac{\sqrt{\lambda_n}}{\lambda_n-\lambda_k}+\frac{\sqrt{\lambda_k}}{\lambda_n-\lambda_k}\Big )\|\proj_{H_n^\perp}(f\phi_k)\|_2 \\
     &=& \frac{ \|\Gamma(f)\|_{\infty}^{1/2}}{\sqrt{\lambda_n}-\sqrt{\lambda_k}}\|\proj_{H_n^\perp}(f\phi_k)\|_2.
 \end{eqnarray*}
 The lemma follows.
\end{proof}

We conclude with the proof of our main result.

\begin{proof}[Proof of Theorem \ref{thm:var-reduce-general}]
    
    The variance formula for projection DPPs gives
    \[ \var[\mathcal{S}_n(f)] = \sum_{k=1}^n \| \proj_{H_n^{\perp}}(f\phi_k)\|^2_2 .\]
    Recall that
      \[\mathcal R_{\L}(n):= \max\Big \{n_{\L}(1),
    \max_{\substack{0 \le p \le \lfloor\sqrt\lambda_n\rfloor-1\\ p \in \mathbb{Z}}} 
    \big \{n_{\L}\big ((\sqrt{\lambda_{n}}-p)^{2} \big ) - n_{\L}\big ((\sqrt{\lambda_{n}}-p-1)^{2} \big )\big \} \Big \}.\]
    For $0\le p \le \lfloor\sqrt{\lambda_n}\rfloor$, define
    \[I_p:= \big\{1\le k \le n: \sqrt{\lambda_n}-p-1 < \sqrt{\lambda_k} \le \sqrt{\lambda_n}-p \big \}.\]
    We have $|I_p|\le \mathcal{R}_{\L}(n)$ by definition of $\mathcal{R}_{\L}(n)$.
    If $k\in I_p$ with $p>0$, by Lemma \ref{lm:far-lambda-k} one has
    \[ \|\proj_{H_n^\perp}(f\phi_k)\|_2 \le \frac{\|\Gamma(f)\|_{\infty}^{1/2}}{p}\cdot\]
    If $k\in I_0$, one has $\|\proj_{H_n^\perp}(f\phi_k)\|_2 \le \|f\phi_k\|_2\le \|f\|_{\infty}$. Therefore
    \begin{eqnarray*}
        \var[\mathcal{S}_n(f)] &=& \sum_{p=0}^{\lfloor \sqrt{\lambda_n}\rfloor} \sum_{k\in I_p} \|\proj_{H_n^\perp}(f\phi_k)\|_2^2 \\
        &\le& \Big (\|f\|_{\infty}^2 + \sum_{p \ge 1} \frac{\|\Gamma(f)\|_{\infty}}{p^2}\Big ) \mathcal{R}_{\L}(n) \\
        &\le& \Big (\|f\|_{\infty}^2 + \frac{\pi^2}{6}\|\Gamma(f)\|_{\infty}\Big ) \mathcal{R}_{\L}(n)
    \end{eqnarray*}
    as desired.
\end{proof}

\subsection{Proof of Theorem \ref{thm:var-reduce-general-alpha} (i)}

\begin{proof}[Proof of Theorem \ref{thm:var-reduce-general-alpha} (i)]
    The variance formula for projection DPPs gives
    \[ \var[\mathcal{S}_n(f)] = \sum_{k=1}^n \| \proj_{H_n^{\perp}}(f\phi_k)\|^2_2 .\]
    Let $\alpha>0$ and $f$ be such that
    $[\L^{\alpha},\M_f]$ is a bounded operator on $L^2(\mu)$.
    We claim that 
    \[\|\proj_{H_{n}^\perp}(f\phi_k)\|_2 \le\frac{1}{\lambda_n^{\alpha}-\lambda_k^{\alpha}} \|[\L^{\alpha},\M_f]\|_{L^2\rightarrow L^2} , \quad \text{for every } k \text{ such that } \lambda_k<\lambda_n. \]
    To prove the claim, we observe that
    \[H_n^\perp =\overline{\Span \{\phi_{m}: m \ge n+1\}} \subset \overline{\Span\{\phi_m: \lambda_m \ge \lambda_{n}\}} =:H_{\ge \lambda_{n}}.\] 
    Thus, for every $1\le k\le n$, we have
    $ \|\proj_{H_n^\perp}(f\phi_k)\|_2 \le \|\proj_{H_{\ge \lambda_{n}}}(f\phi_k)\|_2$. 
Since $\L^{\alpha}\phi_k = \lambda_k^{\alpha} \phi_k$, we have
    \[ [\L^{\alpha},\M_f] \phi_k = \L^{\alpha}(f\phi_k) - f\L^{\alpha}\phi_k = \L^{\alpha}(f\phi_k) - \lambda_k^{\alpha} f\phi_k = (\L^{\alpha} - \lambda_k^{\alpha})(f\phi_k).\]
    Applying $\mathbf{1}_{[\lambda_{n}^{\alpha}, \infty)}(\L^{\alpha})$ to both sides gives
    \begin{eqnarray*}
         \mathbf{1}_{[\lambda_{n}^{\alpha}, \infty)}(\L^{\alpha})([\L^{\alpha},\M_f] \phi_k )
    &=& \mathbf{1}_{[\lambda_{n}^{\alpha}, \infty)}(\L^{\alpha})((\L^{\alpha} - \lambda_k^{\alpha})(f\phi_k)) \\
    &=& (\L^{\alpha} - \lambda_k^{\alpha})(\mathbf{1}_{[\lambda_{n}^{\alpha}, \infty)}(\L^{\alpha}))(f\phi_k).
    \end{eqnarray*}
    For $k$ such that $\lambda_k  < \lambda_{n}$, $\L^{\alpha} - \lambda_k^{\alpha}$ is invertible on $\range(\mathbf{1}_{[\lambda_{n}^{\alpha}, \infty)}(\L^{\alpha}))$. 
    Thus, we can apply $(\L^{\alpha} - \lambda_k^{\alpha})^{-1}$ to both sides to get
    \[\mathbf{1}_{[\lambda_{n}^{\alpha}, \infty)}(\L^{\alpha})(f\phi_k) 
    = (\L^{\alpha}-\lambda_k^{\alpha})^{-1}(\mathbf{1}_{[\lambda_{n}^{\alpha}, \infty)}(\L^{\alpha})([\L^{\alpha},\M_f] \phi_k )).\]
    Therefore,
    \begin{eqnarray*} 
         \|\proj_{H_n^\perp}(f\phi_k)\|_2&\le& \|\mathbf{1}_{[\lambda_{n}^{\alpha}, \infty)}(\L^{\alpha}))(f\phi_k) \|_2 \\
         &=& \|(\L^{\alpha}-\lambda_k^{\alpha})^{-1}(\mathbf{1}_{[\lambda_{n}^{\alpha}, \infty)}(\L^{\alpha})([\L^{\alpha},\M_f] \phi_k ))\|_2 \notag \\
        &\le& \|(\L^{\alpha}-\lambda_k^{\alpha})^{-1}\mathbf{1}_{[\lambda_{n}^{\alpha}, \infty)}(\L^{\alpha})\|_{L^2\rightarrow L^2} \|[\L^{\alpha},\M_f]\|_{L^2 \rightarrow L^2} \|\phi_k \|_2  \notag \\
        &\le& \frac{1}{\lambda_{n}^{\alpha} - \lambda_k^{\alpha}} \|[\L^{\alpha},\M_f]\|_{L^2 \rightarrow L^2}.
    \end{eqnarray*}
    The claim follows.

    For each $p\in \N \cup \{0\}$, we define the following set of indices
    \[ I_p := \{1\le k\le n: \lambda_n^{\alpha}-p-1<\lambda_k^{\alpha} \le \lambda_n^{\alpha}-p\}.\]
    The sets $I_0,\ldots,I_{\lfloor\lambda_n^{\alpha}\rfloor} $ form a partition of $\{1,\ldots,n\}$.
     For $k\in I_0$, we have 
     \[\|\proj_{H_n^{\perp}}(f\phi_k)\|_2 \le \|f\phi_k\|_2\le  \|f\|_{\infty} \|\phi_k\|_2 = \|f\|_{\infty}.\] 
      For $k\in I_p$ with $p\ge 1$, by the claim above, we have 
      \[\|\proj_{H_{n}^\perp}(f\phi_k)\|_2 \le \frac{\|[\L^{\alpha},\M_f]\|_{L^2 \rightarrow L^2}}{\lambda_n^{\alpha}-\lambda_k^{\alpha}} \le \frac{\|[\L^{\alpha},\M_f]\|_{L^2 \rightarrow L^2}}{p} \cdot\] 
      Thus,
    \begin{eqnarray*}
        \var[\mathcal{S}_n(f)]
        &\le& \|f\|^2_{\infty} |I_0| + \|[\L^{\alpha},\M_f]\|_{L^2 \rightarrow L^2}^2 \sum_{1 \le p \le \lfloor\lambda_n^{\alpha}\rfloor}\frac{|I_p|}{p^{2}} \\
        &\le& \Big (\|f\|_{\infty}^2+\frac{\pi^2}{6}\|[\L^{\alpha},\M_f]\|_{L^2 \rightarrow L^2}^2\Big ) \max_{0 \le p \le \lfloor\lambda_n^{\alpha}\rfloor} |I_p|.
    \end{eqnarray*}   
    For $0\le p \le \lfloor\lambda_n^{\alpha}\rfloor-1$, it is easy to see that 
    \[ |I_p| \le n_{\L}\big ((\lambda_{n}^{\alpha}-p)^{1/\alpha} \big ) - n_{\L}\big ((\lambda_{n}^{\alpha}-p-1)^{1/\alpha} \big ).\]
   For $p=\lfloor\lambda_n^{\alpha}\rfloor$, we have $|I_p|\le n_{\L}(1)$. 
  Thus $|I_p|\le \mathcal R_{\L}^{(\alpha)}(n)$ for every $0\le p \le \lfloor \lambda_n^{\alpha}\rfloor$.
  Therefore, 
   \[
        \var[\mathcal{S}_n(f)] \le \Big (\|f\|_{\infty}^2+\frac{\pi^2}{6}\|[\L^{\alpha},\M_f]\|_{L^2 \rightarrow L^2}^2\Big ) \mathcal R_{\L}^{(\alpha)}(n). \]
\end{proof}

\subsection{Proof of Theorem \ref{thm:var-reduce-general-alpha} (ii)}

We begin with some technical lemmas.

\begin{lemma} \label{lm:Ptf-bound}
     Let $\rmP_t =e^{-t\L}, t \ge 0$ and let $f\in \mathcal{A}$. We have
    \[ \|[\rmP_t, \M_f]\|_{L^2 \rightarrow L^2} \le \|\Gamma(f)^{1/2}\|_{\infty} 2\sqrt{t}, \quad \forall t>0.\]
\end{lemma}

\begin{proof}
    Fix $g,h \in L^2(\mu)$ and $t>0$, we define $\psi(s):= \langle f\rmP_s g, \rmP_{t-s}h\rangle $ for $s\in [0,t]$. 
    Strong continuity of $(\rmP_s)$ and boundedness of $\M_f$ imply that $\psi(s)$ is continuous on $[0,t]$, with $\psi(0)=\langle fg, \rmP_t h\rangle $ and $\psi(t) = \langle f\rmP_t g , h\rangle$. 

    For $s\in (0,t)$, we have $\partial_s \rmP_s = -\L\rmP_s$ and $\partial_s \rmP_{t-s} = \L\rmP_{t-s}$. This implies
    \[ \psi'(s) = -\langle f \L\rmP_s g, \rmP_{t-s} h \rangle + \langle f\rmP_s g , \L \rmP_{t-s}h \rangle = \mathcal{B}_f(\rmP_s g,\rmP_{t-s}h). \]
    Applying Lemma \ref{lm:LMf-bound} for $w_1 = \rmP_s g, w_2 = \rmP_{t-s}h$ gives
    \[ |\psi'(s)| \le \|\Gamma(f)^{1/2}\|_{\infty} \big (\|\rmP_s g\|_2 \|\L^{1/2}\rmP_{t-s}h\|_2 + \|\rmP_{t-s}h\|_2 \|\L^{1/2}\rmP_s g\|_2 \big ). \]
    We note that $\|\rmP_{s}\|_{L^2\rightarrow L^2} \le 1$, and 
    \[ \|\L^{1/2}\rmP_s\|_{L^2\rightarrow L^2} \le  \sup_{\lambda\ge 0} \sqrt{\lambda}e^{-\lambda s} = (2es)^{-1/2}.\]
    This implies 
    \[ |\psi'(s)| \le  \|\Gamma(f)^{1/2}\|_{\infty} (2e)^{-1/2}\big (s^{-1/2}+(t-s)^{-1/2} \big ) \|g\|_2\|h\|_2, \quad \forall s\in (0,t).\]
    Now note that $\psi(t)-\psi(0) = \langle \M_f\rmP_tg,h\rangle - \langle \rmP_t \M_f g,h\rangle = -\langle [\rmP_t,\M_f]g,h\rangle $, which implies
    \[ |\langle [\rmP_t,\M_f] g,h\rangle| = |\psi(t)-\psi(0)| \le \int_0^t|\psi'(s)|\d s \le  \|\Gamma(f)^{1/2}\|_{\infty} 2\sqrt{2t}e^{-1/2}\|g\|_2\|h\|_2  .\]
    Since $e\ge2$, we deduce that 
    \[ |\langle [\rmP_t,\M_f] g,h\rangle|  \le  \|\Gamma(f)^{1/2}\|_{\infty} 2\sqrt{t}\|g\|_2\|h\|_2  .\]
\end{proof}

\begin{lemma} \label{lm:Balakrishnan-repn}
    For $\alpha \in (0,1)$ and $u\in \mathcal{D}(\L^{\alpha})$
    \[\L^{\alpha} u = \frac{\alpha}{\gamma(1-\alpha)} \lim_{\substack{r\to 0^+ \\ R \to \infty}} \int_r^R(\I - \rmP_t)u t^{-1-\alpha} \d t \qquad\text{ in } L^2(\mu), \]
    where $\gamma(z)$ denotes the Euler gamma function, given by 
    \[ \gamma(z):=\int_0^{\infty} u^{z-1}e^{-u} \d u, \qquad z\in \C \text{ with } \mathrm{Re}(z)>0.\]
\end{lemma}

\begin{proof}
    By integration by parts, we have the following identity
    \[\lambda^\alpha = c_{\alpha}\int_0^\infty (1-e^{-t\lambda})t^{-1-\alpha}\d t, \quad \lambda \ge 0\]
    where $c_{\alpha}:=\alpha( \int_0^\infty s^{-\alpha} e^{-s} \d s)^{-1}=\alpha/\gamma(1-\alpha)$. 
    For $0<r<R<\infty$, the Bochner integral $\int_r^R (\I - \rmP_t)u t^{-1-\alpha} \d t $ converges in $L^2(\mu)$. By the spectral theorem, it equals $m_{r,R}(\L)u$, where
    \[ m_{r,R}(\lambda) := \int_r^R(1-e^{-t\lambda})t^{-1-\alpha} \d t.\]
    As $r\rightarrow 0^+$ and $R\rightarrow\infty$, $\frac{\alpha}{\gamma(1-\alpha)}m_{r,R}(\lambda)$ converges to $\lambda^{\alpha}$ pointwise. Since the integrand is non-negative, $0\le \frac{\alpha}{\gamma(1-\alpha)}m_{r,R}(\lambda) \le \lambda^{\alpha}$. Thus, we have 
    \[ \Big | \frac{\alpha}{\gamma(1-\alpha)}m_{r,R}(\lambda) - \lambda^{\alpha}\Big |^2 \le 4\lambda^{2\alpha},\]
    which is integrable against the spectral measure $\d\nu_u$, since $u\in \mathcal{D}(\L^{\alpha})$. By the dominated convergence theorem in $L^2(\d\nu_u)$, the lemma follows.
\end{proof}

We conclude with the proof of Theorem \ref{thm:var-reduce-general-alpha} (ii).

\begin{proof}[Proof of Theorem \ref{thm:var-reduce-general-alpha} (ii)]
   For $\alpha \in (0,\tfrac{1}{2})$, we define the commutator $[\L^{\alpha},\M_f]$ acting on $\mathcal{D}(\L^{1/2})$ by
\begin{equation*}
    [\L^{\alpha},\M_f]u := \L^{\alpha}(\M_f u) - \M_f(\L^{\alpha}u) =  \L^{\alpha}(fu) - f\L^{\alpha}u, \quad u \in \mathcal{D}(\mathcal{E}).
\end{equation*}
By Lemma \ref{lm:bounded-multiplier}, $fu\in \mathcal{D}(\mathcal{E})=\mathcal{D}(\L^{1/2})$. Note that $\mathcal{D}(\L^{\alpha}) \supset \mathcal{D}(\L^{1/2})$. Thus, $[\L^{\alpha},\M_f]$ is well-defined on $\mathcal{D}(\L^{1/2})$ for $\alpha \in(0,1/2)$. 

   Let $g\in \mathcal{D}(\mathcal{E})$. By Lemma \ref{lm:Balakrishnan-repn}, with $c_{\alpha}:=\alpha/\gamma(1-\alpha)$, we have the following with limits in $L^2(\mu)$:
   \begin{eqnarray*}
   \L^{\alpha} g &=& c_{\alpha} \lim_{\substack{r\to 0^+ \\ R \to \infty}} \int_r^R(g - \rmP_tg) t^{-1-\alpha} \d t , \quad \text{and}   \\
    \L^{\alpha}(fg) &=& c_{\alpha} \lim_{\substack{r\to 0^+ \\ R \to \infty}} \int_r^R(fg - \rmP_t(fg)) t^{-1-\alpha} \d t.
   \end{eqnarray*}
   The multiplier $\M_f$ is bounded on $L^2(\mu)$, so it commutes with these limits and with finite-interval Bochner integrals, namely
   \[ f\L^{\alpha}g = c_{\alpha} \lim_{\substack{r\to 0^+ \\ R \to \infty}} \int_r^R(fg - f\rmP_tg) t^{-1-\alpha} \d t.\]
   Subtracting these terms yields
   \begin{equation}
        [\L^{\alpha},\M_f]g = -c_{\alpha}  \lim_{\substack{r\to 0^+ \\ R \to \infty}} \int_r^R [\rmP_t,\M_f]g t^{-1-\alpha} \d t
   \end{equation}
    Note that the limit on the right is a Bochner integral on $(0,\infty)$ by Lemma \ref{lm:Ptf-bound} and the trivial bound $\|[\rmP_t,\M_f]\|_{L^2\rightarrow L^2} \le 2 \|f\|_{\infty}$, we have
   \[\|[\rmP_t,\M_f]g\|_{2}  \le 
   \min \big (2\|\Gamma(f)^{1/2}\|_{\infty} \sqrt{t}, 2 \|f\|_{\infty}\big ) \|g\|_2 .\]
   For convenience, we set
   \[C_1:=2\|\Gamma(f)^{1/2}\|_{\infty} \quad,\quad C_2:= 2 \|f\|_{\infty}.\]
    If $C_2=0$, then $f=0$ and the lemma is trivial. If $C_1=0$,
Lemma \ref{lm:Ptf-bound} gives $[\rmP_t,\M_f]=0$ for every $t>0$,
and hence $[\L^\alpha,\M_f]=0$. We may therefore assume that
$C_1,C_2>0$.
   For any $T>0$, we have 
   \begin{eqnarray*}
       \int_0^\infty \|[\rmP_t,\M_f]g\|_{2} t^{-1-\alpha} \d t &\le& \|g\|_2\Big (\int_0^TC_1 t^{-\frac{1}{2}-\alpha}\d t + 
       \int_T^{\infty} C_2 t^{-1-\alpha}\d t\Big )\\
       &=& \|g\|_2 \Big (C_1\frac{T^{\frac{1}{2}-\alpha}}{\frac{1}{2}-\alpha}+ C_2 \frac{T^{-\alpha}}{\alpha}\Big) \cdot 
   \end{eqnarray*}
   Choosing $T=(C_2/C_1)^2$ yields
   \[ C_1\frac{T^{\frac{1}{2}-\alpha}}{\frac{1}{2}-\alpha}+ C_2 \frac{T^{-\alpha}}{\alpha} = C_2 T^{-\alpha} \Big (\frac{1}{\frac{1}{2}-\alpha}+\frac{1}{\alpha} \Big )= C_1^{2\alpha}C_2^{1-2\alpha} \frac{1}{\alpha(1-2\alpha)}\cdot\]

   Combining everything, we then have
   \[ \|[\L^{\alpha},\M_f]g \|_2 \le c_{\alpha} \frac{C_1^{2\alpha}C_2^{1-2\alpha}}{\alpha(1-2\alpha)} \|g\|_2, \quad \forall g\in \mathcal{D}(\mathcal{E}).\]
   Since $\mathcal{D}(\mathcal{E})$ is dense in $L^2(\mu)$, $[\L^{\alpha},\M_f]$ then extends to a bounded operator on $L^2(\mu)$ with
   \[ \|[\L^{\alpha},\M_f]\|_{L^2\rightarrow L^2} \le \frac{2}{(1-2\alpha)\gamma(1-\alpha)}\|\Gamma(f)\|_{\infty}^{\alpha} \cdot \|f\|_{\infty}^{1-2\alpha}.\]
    By AM-GM
    \[\|[\L^{\alpha},\M_f]\|_{L^2 \rightarrow L^2} \le \frac{4(\|f\|^2_{\infty} + \|\Gamma(f)\|_{\infty})^{1/2}}{(1-2\alpha)\gamma(1-\alpha)} \cdot \]
    Since the Euler gamma function $\gamma(1-\alpha)$ is uniformly bounded from below by a positive constant for $\alpha \in (0,1/2)$, we have
    \[ \|[\L^{\alpha},\M_f]\|_{L^2 \rightarrow L^2} \le \frac{C}{1-2\alpha}(\|f\|^2_{\infty} + \|\Gamma(f)\|_{\infty})^{1/2}, \qquad \forall \alpha \in \Big (0,\frac{1}{2}\Big )\]
    where $C>0$ is a numerical constant.
    This completes the proof.
\end{proof}

\section{Proofs from Subsection \ref{sec:monte-carlo-manifold}} \label{proof:manifold}
\subsection{Weyl law and local spectral counting functions}

We recall that the eigencounting function of $\L$ is given by
\[n_{\L}(\lambda)
:=
\#\{k\geq 1:\lambda_k\leq\lambda\}, \qquad \lambda \in \R.\]
The celebrated Weyl law describes the asymptotic behavior of $n_{\L}(\lambda)$ as $\lambda\to\infty$ (e.g., see \cite{hormander1968spectral, huang2022pointwise, rivera2022weighted}). In particular, we have the following lemma.

\begin{lemma}[Local Weyl law]\label{lm:loc-weyl-law}
    Let  $m\ge2$ and $\rho$ satisfy \eqref{eq:assumption-rho}. Then
     \[\sup_{x\in \mathcal{M}}\big |\sum_{j: \lambda_j \le \lambda} \phi_j(x)^2\rho(x) - \frac{\omega_m}{(2\pi)^m}\lambda^{m/2}\big | \le C (1+\lambda^{(m-1)/2}), \quad \forall \lambda \ge 0,\]
     where $\omega_m$ denotes the volume of the unit ball in $\R^m$, and $C>0$ is a constant depending only on $(\mathcal{M},g)$ and $\rho$. Moreover, 
    \[ \Big |n_{\L}(\lambda) - \frac{\omega_m}{(2\pi)^m} \lambda^{m/2} \Big |  \le C(1+ \lambda^{(m-1)/2}),\qquad \forall \lambda \ge 0.\]
\end{lemma}

\begin{proof}[Proof of Lemma \ref{lm:loc-weyl-law}]
    Define a linear operator $U:L^2(\mu) \to L^2(\vol_g)$ by $Uf:=\rho^{1/2}f$.
   Then $U\L U^{-1} = -\Delta_g +V_{\rho}$, where
   \[ V_{\rho}:= \frac{1}{2}\Delta_g \log \rho + \frac{1}{4}|\nabla \log \rho|^2_g.\]
   Set $\psi_j:= \rho^{1/2}\phi_j$. Then $\psi_j$'s are eigenfunctions of the Schr\"odinger operator $-\Delta_g+V_{\rho}$; moreover, $\psi_j$'s are orthonormal in $L^2(\vol_g)$.
   Since $\rho \in C^{1,1}(\mathcal{M})$ and is bounded from below by a positive constant, $V_\rho \in L^{\infty}(\vol_g) \subset L^m(\vol_g)$. 
    Applying Theorem 2 in \cite{huang2022pointwise} yields 
   \[\sup_{x\in \mathcal{M}}\big |\sum_{j: \lambda_j \le \lambda} \psi_j(x)^2 - \frac{\omega_m}{(2\pi)^m}\lambda^{m/2}\big | \le C \lambda^{(m-1)/2}, \quad \forall \lambda \ge 1,\]
   for $C>0$ depending only on $(\mathcal{M},g)$ and $\rho$. This yields the first assertion.
   
   For the second assertion, observe that
   \[ \int_{\mathcal{M}} \sum_{j:\lambda_j \le \lambda}\psi_j(x)^2 \vol_g(\d x) = \sum_{j:\lambda_j \le \lambda} \|\psi_j\|^2_{L^2(\vol_g)} = n_{\L}(\lambda).\]
   Therefore, for all $\lambda \ge 1$
   \[\Big |n_{\L}(\lambda) - \frac{\omega_m}{(2\pi)^m} \lambda^{m/2} \Big | \le 
     \sup_{x\in \mathcal{M}}\big |\sum_{j: \lambda_j \le \lambda} \psi_j(x)^2 - \frac{\omega_m}{(2\pi)^m}\lambda^{m/2}\big | \le C \lambda^{(m-1)/2}.\] 
  By enlarging the constant $C$, we have
     \[\Big |n_{\L}(\lambda) - \frac{\omega_m}{(2\pi)^m} \lambda^{m/2} \Big | \le 
     C(1+ \lambda^{(m-1)/2}),\]
  as desired.
\end{proof}

 Lemma~\ref{thm:spectral-counting} below provides the main spectral input needed to apply Theorem~\ref{thm:var-reduce-general} in this setting. 

\begin{lemma} \label{thm:spectral-counting}
     There exists $C>0$ depending only on $(\mathcal{M},g)$ and $\rho$ such that: 
    \[ \mathcal{R}_{\L}(n) \le C n^{1-1/m}, \qquad \forall n\in \N.\]
\end{lemma}

\begin{proof}
  By enlarging the implicit constant if needed, it suffices to prove the claim for $n$ large enough. 
  By Lemma \ref{lm:loc-weyl-law}, we have 
  \begin{eqnarray*}
      \Big |n_{\L}(\lambda) - \frac{\omega_m}{(2\pi)^m} \lambda^{m/2} \Big |
      \le C_0 (1+\lambda^{(m-1)/2})
  \end{eqnarray*}
  for some $C_0>0$  depending only on $(\mathcal{M},g)$ and $\rho$. 
    We note that 
   \[n_{\L}(\lambda_n - \delta) <n \le n_{\L}(\lambda_n),\qquad \forall \delta>0.\] 
   Letting $\delta\downarrow 0$ gives
    \[\big |n - \frac{\omega_m}{(2\pi)^m}\lambda_n^{m/2}\big | \le C_1(1+\lambda_n^{(m-1)/2})\]
    for some $C_1$ depending only on $(\mathcal{M},g)$ and $\rho$. In particular, $\lambda_n \asymp n^{2/m}$ for large $n$.

\medskip

    We recall that
    \[\mathcal R_{\L}(n)
    := \max\Big \{n_{\L}(1);
    \max_{0 \le p \le \lfloor\sqrt{\lambda_n}\rfloor-1} \big \{n_{\L}\big ((\sqrt{\lambda_{n}}-p)^{2} \big ) - n_{\L}\big ((\sqrt{\lambda_{n}}-p-1)^{2} \big )\big \} \Big \}.\]
    For $p\le \lfloor \sqrt{\lambda_n}\rfloor -1$ , one has
    \begin{eqnarray*}
       && n_{\L}\big ((\sqrt{\lambda_{n}}-p)^{2} \big ) - n_{\L}\big ((\sqrt{\lambda_{n}}-p-1)^{2} \big ) \\
       &\le& \frac{\omega_m}{(2\pi)^m}\big ((\sqrt{\lambda_n} -p)^{m} - (\sqrt{\lambda_n} -p-1)^{m}\big ) + 2C_0(1+\lambda_n^{(m-1)/2})\\ 
       &\le& \frac{m\omega_m}{(2\pi)^m} \lambda_n^{(m-1)/2} + 2C_0(1+\lambda_n^{(m-1)/2}).
    \end{eqnarray*} 
   Since $\lambda_n \asymp n^{2/m}$, 
    \[\mathcal R_{\L}(n) \le  Cn^{1-1/m},\]  
    where $C$ depends only on $(\mathcal{M},g)$ and $\rho$ as desired.
\end{proof}

\subsection{Proof of Theorem \ref{thm:var-reduce-manifold}}

\begin{proof}[Proof of Theorem \ref{thm:var-reduce-manifold}] 
    By identifying functions that agree  $\mu$-almost everywhere, we have 
    \[ \lip(\mathcal{M}) \equiv W^{1,\infty}(\mathcal{M}) = \{f\in H^1(\mathcal{M}): \|f\|_{\infty}<\infty, \|\nabla f\|_{\infty}<\infty \}.\]
    Then for Lipschitz $f:\mathcal{M} \rightarrow \R$, one has
    $\|\Gamma(f)\|_{\infty} = \|\nabla f\|^2_{\infty} = \Lip(f)^2$. Applying Lemma \ref{thm:spectral-counting} and Theorem \ref{thm:var-reduce-general} yields
    \begin{eqnarray*}
        \var[\mathcal{S}_n(f)] \le C \big (\|f\|^2_{\infty} + \Lip(f)^2 \big ) \cdot n^{1-1/m}
    \end{eqnarray*}
    where $C$ depends only on $(\mathcal{M},g)$ and $\rho$ as desired. 
\end{proof}

\subsection{Proof of Theorem \ref{thm:monte-carlo-manifold}}

\begin{proof}[Proof of Theorem \ref{thm:monte-carlo-manifold}]
    We first observe that 
    \[ \E [\widehat{\mathcal{L}}_{n}(f)]= \frac{1}{n}\int \frac{n\omega(x)}{\K_n(x,x)}f(x)\K_{n}(x,x)\mu(\d x)= \mathcal{L}_{\omega}(f).\]
    Thus, $\widehat{\mathcal{L}}_{n}(f)$ is an unbiased estimator of $\mathcal{L}_{\omega}(f)$.
    Then
  \[\E[|\widehat{\mathcal{L}}_{n}(f)- \mathcal{L}_{\omega}(f)|^2] 
  = \var[\widehat{\mathcal{L}}_{n}(f)] 
  = \frac{1}{n^2}\var[\mathcal{S}_n(q_n  f \omega)],\]
  where $q_n(x):=n/\K_n(x,x)$. 
  We note that $\K_n(x,x) = \sum_{j=1}^n \phi_j(x)^2$, and $\phi_1=1$. This implies $\K_n(x,x) \ge 1$ for all $x\in \mathcal{M}$, thus the weight $q_n$ is well-defined. 
  
  Since $ \mathcal{S}_n(q_nf\omega) = \mathcal{S}_n(f\omega) + \mathcal{S}_n((q_n-1)f\omega)$, one has
 \begin{eqnarray*}
     \var[ \mathcal{S}_n(q_nf\omega)]^{1/2} &\le& \var[\mathcal{S}_n(f\omega)]^{1/2} +  \var[\mathcal{S}_n((q_n-1)f\omega)]^{1/2} \\
     &\le&  \var[\mathcal{S}_n(f\omega)]^{1/2} +  \E[\mathcal{S}_n(|(q_n-1)f\omega|^2)]^{1/2}\\
     &\le&  \var[\mathcal{S}_n(f\omega)]^{1/2} + n^{1/2} \|q_n-1\|_{\infty}\|f\omega\|_{\infty}.
 \end{eqnarray*}
  Applying Theorem \ref{thm:var-reduce-manifold} to $f\omega$, one has
  \[\frac{1}{n}\var[\mathcal{S}_n(f \omega)]^{1/2} \le C(\|f\omega\|_{\infty}+\lip( f\omega)) n^{-1/2-1/(2m)} \]
  where $C$ depends only on $(\mathcal{M},g)$. 
  
  Thus, it suffices to control $\|q_n-1\|_{\infty}$.
  To this end,
  we claim that there exists a constant $C_1>0$ depending only on $(\mathcal{M},g)$ such that
  \begin{equation} \label{eq:claim-monte-carlo-dpp}
      \sup_{x\in \mathcal{M}}|\K_n(x,x) -n | \le C_1 n^{1-1/m}, \quad \forall n \in \mathbb{N}.
  \end{equation}
  Indeed, let $\Lambda>0$ be such that $n_{\L}(\Lambda) \le n \le n_{\L}(\Lambda+1)$. It suffices to prove the claim for $n$ large enough, so that we can assume $\Lambda \ge 1$.  We observe that
  \[\K_n(x,x) - n \le  \Big (\sum_{j:\lambda_j \le \Lambda+1} \phi_j(x)^2 - n_{\L}(\Lambda+1) \Big ) + \big ( n_{\L}(\Lambda+1) - n_{\L}(\Lambda)\big )\]
  and 
  \[\K_n(x,x) - n \ge  \Big (\sum_{j:\lambda_j \le \Lambda} \phi_j(x)^2 - n_{\L}(\Lambda) \Big ) - \big ( n_{\L}(\Lambda+1) - n_{\L}(\Lambda)\big ).\]
  By Lemma \ref{lm:loc-weyl-law} with $\rho=1$, we have
  \[\sup_{x\in \mathcal{M}} \big | \sum_{j:\lambda_j \le \lambda} \phi_j(x)^2 - n_{\L}(\lambda) \big |  = O(\lambda^{(m-1)/2}), \qquad \text{as } \lambda \to \infty.\]
  Applying this lemma and the mean-value theorem yields the claim \eqref{eq:claim-monte-carlo-dpp}.
  
  By the claim \eqref{eq:claim-monte-carlo-dpp}, one has
  \begin{eqnarray*}
      \|q_n-1\|_{\infty} = \sup_{x\in \mathcal{M}}\Big |\frac{\K_n(x,x)-n}{\K_n(x,x)}\Big | \le C_2 n^{-1/m} 
  \end{eqnarray*}
  for some $C_2>0$ depending only on $(\mathcal{M},g)$ and for $n$ sufficiently large. Combining these estimates, we obtain
  \begin{eqnarray*}
      \E[|\widehat{\mathcal{L}}_{n}(f)- \mathcal{L}_{\omega}(f)|^2]^{1/2} \le  C(\|f\omega\|_{\infty}+\lip( f\omega)) n^{-1/2-1/(2m)} +C_2 \|f\omega\|_{\infty}n^{-1/2-1/m}.
  \end{eqnarray*}
  Enlarging the constant $C$ and using
  \[\|f\omega\|_{\infty}+\lip( f\omega) \le (\|\omega\|_\infty + \|\nabla \omega\|_{\infty})(\|f\|_{\infty} + \lip (f)),\]
  completes the proof.
\end{proof}

\section{Proofs from Subsection \ref{sec:minibatch-lowdim}} \label{proof:geometric-graph}
\subsection{Local spectral counting functionals of geometric graphs}\label{subsec:graph-loc-spec-general}

Let $\L_N$ be the Laplacian on $\mathcal{X}$ defined as in \eqref{eq:rescaled-graph-laplacian-four}.
We list the eigenvalues of \(\L_N\) as
 $0=\lambda_{1}(\L_N)\leq\lambda_{2}(\L_N)\leq\cdots\leq\lambda_{N}(\L_N)$, and define
\[
    n_{\L_N}(\lambda):
    =\#\{j\in \{1,\ldots,N\}:\lambda_{j}(\L_N)\leq\lambda\}.
\]

In this section, we provide the general result to control the local spectral counting functional $\mathcal{R}_{\L_N}(n)$ associated to $\L_N$ (cf. Definition \ref{def:loc-spec-counting}). The idea is that $\L_N$ approximates a continuum operator $\L_{\star}$ on $L^2(\mu)$ for large $N$. This suggests that the spectrum of $\L_N$ could be controlled by the spectrum of $\L_{\star}$ (cf. \cite{calder2022improved, calder2022lipschitz,dunson2021spectral}).

\begin{theorem} \label{thm:abstract-loc-spec-counting-graph}
    We assume that there exists a nonnegative self-adjoint operator \(\L_\star\) on \(L^2(\mu)\) satisfying the following properties:
\begin{itemize}
    \item[(i)] $\L_\star$ has purely discrete spectrum with finite multiplicities. Furthermore, there exist structural constants $C_1,C_2>0$ such that
    \begin{equation} \label{eq:abstract-limiting-weyl}
        \big|n_{\L_\star}(\lambda)-C_1 \lambda^{m/2}\big| \le C_2(1+\lambda^{(m-1)/2}), \qquad \lambda \ge 0,
    \end{equation}
    where $n_{\L_\star}(\lambda):=\#\{j\in \N:\lambda_j(\L_\star)\leq\lambda\}$.
    \item[(ii)] There exist a non-decreasing function $\err:[0,\infty)\to [0,\infty)$, and an event $\mathcal{G}_N$ on which  the following holds
    \begin{equation} \label{eq:abstract-sandwich}
 n_{\L_\star}\!\big(\lambda\big(1-4\err(\lambda)\big)\big) \leq n_{\L_N}(\lambda) \leq n_{\L_\star}\!\big(
     \lambda \big (1+4\err(\lambda)\big)
     \big), 
\end{equation}
for all $\lambda > 0$ such that $\err(\lambda)\le 1/8$.
\end{itemize}
   Then, on the event $\mathcal{G}_N$, 
    we have 
    \begin{equation*}
       n \asymp \lambda_n(\L_N)^{m/2} ,\qquad \mathcal{R}_{\L_N}(n) \lesssim n^{1-1/m}+n\err(\lambda_n(\L_N)),
    \end{equation*} 
    for every $n$ such that $\lambda_n(\L_N)\ge 1$ and $\err(\lambda_n(\L_N)) \le 1/8$.
\end{theorem}

\begin{remark}
    In our examples below, the error function $\err$ will be given explicitly, and decays fast in $N$.
    In the suitable scaling regime, one has $\err(\lambda_n(\L_N)) = O(n^{-1/m})$. This yields the rate $\mathcal{R}_{\L_N}(n) = O(n^{1-1/m})$, which is the same rate we have observed in Lemma \ref{thm:spectral-counting}. 
\end{remark}

\begin{proof}[Proof of Theorem \ref{thm:abstract-loc-spec-counting-graph}]
    For simplicity of notation, set $\Lambda:= \lambda_n(\L_N)$. We assume that $\Lambda \ge 1$ and $\err(\Lambda)\le 1/8$. On the event $\mathcal{G}_N$ we have
    \[ n \le n_{\L_N}(\Lambda) \le n_{\L_*}(\Lambda(1+4\err(\Lambda))) \le n_{\L_*}(3\Lambda/2) \lesssim \Lambda^{m/2}.\]
    For the reverse inequality, note that $\Lambda/2<\Lambda$ so that $n_{\L_N}(\Lambda/2) \le n-1$. Thus
    \[ n \ge n_{\L_N}(\Lambda/2) +1 \ge n_{\L_*}(\Lambda/4 )+1 \gtrsim \Lambda^{m/2}.\]
    Therefore, $n\asymp \Lambda^{m/2}$.

\medskip 
    Let
    $\psi(s):= n_{\L_N}(s^{2}) - n_{\L_N}((s-1)^{2})$, where $s \in [1, \Lambda^{1/2}]$. Then 
    \[\mathcal{R}_{\L_N}(n) \le \max\{n_{\L_N}(1), \sup_{s\in [1,\Lambda^{1/2}]} \psi(s) \}.\]
    We note that
    $\psi(1) \le n_{\L_N}(1) \le n_{\L_*}(3/2)=O(1)$.
    Thus, by enlarging the implicit constant if necessary, it is enough to show that
    \begin{equation} \label{eq:claim-psi}
        \psi(s) \lesssim \Lambda^{(m-1)/2} + \Lambda^{m/2}\err(\Lambda), \qquad \forall s\in (1,\Lambda^{1/2}]
    \end{equation}
    where the implicit constant is uniform in $s$. To this end, observe that
    \begin{eqnarray*}
        \psi(s) &\le& n_{\L_\star}\big(s^{2}(1+4\err(s^{2}))\big) - n_{\L_\star}\big((s-1)^{2}(1-4\err((s-1)^{2}))\big) \\
        &\le&  n_{\L_\star}\big(s^{2}(1+4\err(s^{2}))\big) - n_{\L_\star}\big((s-1)^{2}(1-4\err(s^{2}))\big)\\
        &\le& C_1\Big [\big(s^{2}(1+4\err(s^{2}))\big)^{m/2} - \big((s-1)^{2}(1-4\err(s^{2}))\big)^{m/2}\Big ] \\
        &+&2 C_2\Big [\big(s^{2}(1+4\err(s^{2}))\big)^{(m-1)/2} + 1\Big ].
    \end{eqnarray*}
    By the mean value theorem, we have
    \begin{eqnarray*}
       && \big(s^{2}(1+4\err(s^{2}))\big)^{m/2} - \big((s-1)^{2}(1-4\err(s^{2}))\big)^{m/2}\\
       &\lesssim& \big(s^{2}(1+4\err(s^{2}))\big)^{(m-2)/2}\Big (s^{2}(1+4\err(s^{2}))-(s-1)^{2}(1-4\err(s^{2}))\Big )\\
       &\lesssim& s^{(m-2)}\Big (s + s^2 \err(s^{2})\Big ) \le \Lambda^{(m-2)/2}\big (\Lambda^{1/2} + \Lambda \err(\Lambda)\big),
    \end{eqnarray*}
    where the implicit constant is independent of $s$. Here we used the fact that $s\le \Lambda^{1/2}$ and $\err(\lambda)$ is non-decreasing.

    On the other hand, since $s\le \Lambda^{1/2}$, we then have
    \[\big(s^{2}(1+4\err(s^{2}))\big)^{(m-1)/2}+1 \le 2^{(m-1)/2}\Lambda^{(m-1)/2}+1\lesssim \Lambda^{(m-1)/2}. \]
    Combining all ingredients, the claim \eqref{eq:claim-psi} follows.
    Therefore,
    \[\mathcal{R}_{\L_N}(n) \lesssim  \Lambda^{(m-1)/2} + \Lambda^{m/2}\err(\Lambda) \asymp n^{1-1/m}+n \err(\Lambda).\]
     This completes the proof.
\end{proof}

\subsection{Proofs for $\eps$-graphs}

\subsubsection{Proof of Theorem \ref{thm:var-reduce-eps}} \label{subsec:graph-proof-eps}

\paragraph{Local spectral counting functions of $\eps$-graphs}

In the setting of $\eps$-graphs, the continuum operator for comparison is $\L_{\star}=\sigma_{\eta}\Delta_{\rho}$, where
\begin{equation*}
    \Delta_\rho f
    :=-\frac{1}{2\rho}\operatorname{div}_g(\rho^2\nabla_g f), \qquad f\in C^{\infty}(\mathcal{M}).
\end{equation*}
Since $\rho$ is smooth and bounded from above and below by positive constants, the Weyl law (cf. \cite{hormander1968spectral}) applies to $\L_\star$. Thus, condition \eqref{eq:abstract-limiting-weyl} is satisfied.

\begin{lemma}[Local spectral counting functions of $\eps$-graphs] \label{thm:loc-spec-eps}
    For $\eps$-graphs, let $\eps, \theta$ and $\tilde{\delta}$ be as in Theorem \ref{thm:var-reduce-eps}. Let $\L_{\star}=\sigma_{\eta}\Delta_{\rho}$. Then, there exist structural constants $C_1, C_2$, and an event $\mathcal{G}_N^{\eps}$ such that 
    \[ \mathbb{P}(\mathcal{G}_N^{\eps}) \ge 1-C_1N \exp(-C_1^{-1}N\theta^2\tilde{\delta}^m),\]
    and on $\mathcal{G}_N^{\eps}$, the following holds  for all $\lambda > 0$ such that $\err(\lambda)\le 1/8$
     \[ n_{\L_\star}\!\big(\lambda\big(1-4\err(\lambda)\big)\big) \leq n_{\L_N^{\eps}}(\lambda) \leq n_{\L_\star}\!\big(
     \lambda \big (1+4\err(\lambda)\big)
     \big), \]
     where
    \[ \err(\lambda):= C_2 \Big (\frac{\tilde{\delta}}{\eps} + \theta + \eps + \eps \sqrt{\lambda}\Big ), \qquad \lambda \ge 0.\]
    Consequently, on $\mathcal{G}_N^{\eps}$ we have $n \asymp \lambda_n(\L_N^{\eps})^{m/2}$ and 
    \begin{equation*} 
      \mathcal{R}_{\L_N^{\eps}}(n) \lesssim n^{1-1/m}+n \Big (\frac{\tilde{\delta}}{\eps} + \eps + \theta + \eps n^{1/m}\Big ),
    \end{equation*} 
    for every $n$ such that $\lambda_n(\L_N^{\eps})\ge 1$ and $\err(\lambda_n(\L_N^{\eps})) \le 1/8$.
\end{lemma}

To prove Lemma \ref{thm:loc-spec-eps}, we use discrete-to-continuum comparison results for graph Laplacians with diffusion operators in \cite{calder2022improved}.
To this end, let
    \[b_{\eps}(u):=2\langle \L_N^{\eps} u ,u\rangle_{L^2(\mu_N)},\qquad u\in L^2(\mu_N)\] 
    and 
    \[ D_2(f):= \int_{\mathcal{M}}|\nabla_g f(x)|^2 \rho^2(x) \vol_g(\d x), \qquad f\in H^1(\mathcal{M}).\]
By integration by parts, one can check that $\sigma_{\eta} D_2(f) = 2\langle f, \L_{\star} f\rangle_{L^2(\mu)}$ for $f\in H^2(\mathcal{M})$.
Thus, these energy forms describe the eigenvalues of $\L_N^{\eps}$ and $\L_{\star}$ variationally via the Courant-Fischer minimax principle. We cite the following result in \cite{calder2022improved}.

\begin{proposition}[Propositions 4.1 and 4.2 in \cite{calder2022improved}]\label{prop:comparison}
    Let $\eps, \theta$, and $\tilde{\delta}$ be as in Theorem \ref{thm:var-reduce-eps}. 
    There exists an event $\mathcal{G}_{\ref{prop:comparison}}$ with probability at least 
    $1-CN\exp(-C^{-1}N\theta^2 \tilde{\delta}^m)$, such that on $\mathcal{G}_{\ref{prop:comparison}}$, there are linear maps $\tilde{P}:L^2(\mu)\to L^2(\mu_N)$ and $\tilde{\mathcal{I}}: L^2(\mu_N) \to \lip(\mathcal{M})$ satisfying:
    \begin{enumerate}
        \item[(i)] For every $f\in \lip(\mathcal{M})$,
        \begin{equation*}
            b_{\eps}(\tilde{P}f) \le \Big (1+C\Big (\frac{\tilde{\delta}}{\eps} + \eps + \theta \Big ) \Big ) \sigma_{\eta}D_2(f),
        \end{equation*}
        and
        \begin{equation*}
            \big | \|f\|^2_{L^2(\mu)}-\|\tilde{P}f\|^2_{L^2(\mu_N)}\big | 
            \le C \tilde{\delta} \|f\|_{L^2(\mu)}\sqrt{D_2(f)} + C(\theta+\tilde{\delta}  )\|f\|^2_{L^2(\mu)}.
        \end{equation*}
        \item[(ii)] For every $u\in L^2(\mu_N)$,
        \begin{equation*}
            \sigma_{\eta}D_2(\tilde{\mathcal{I}}u) \le \Big (1+C\Big (\frac{\tilde{\delta}}{\eps} + \eps + \theta \Big ) \Big ) b_{\eps}(u),
        \end{equation*}
        and
        \begin{equation*}
            \big | \|u\|^2_{L^2(\mu_N)}-\|\tilde{\mathcal{I}}u\|^2_{L^2(\mu)}\big | 
            \le C \eps \|u\|_{L^2(\mu_N)}\sqrt{b_{\eps}(u)} + C(\theta + \tilde{\delta}  ) \|u\|^2_{L^2(\mu_N)}.
        \end{equation*}
    \end{enumerate}
\end{proposition}

Given Proposition \ref{prop:comparison}, we now prove Lemma \ref{thm:loc-spec-eps}.

\begin{proof}[Proof of Lemma \ref{thm:loc-spec-eps}]
    Let $\mathcal{G}_N^{\eps} = \mathcal{G}_{\ref{prop:comparison}}$ be the event in Proposition \ref{prop:comparison}.
    In what follows, we restrict to the event $\mathcal{G}_N^{\eps}$, so that
    all estimates in Proposition \ref{prop:comparison} are valid on this event. 

    For $\lambda \ge 0$, let $V_N^{\eps}(\lambda)$ (respectively, $V(\lambda)$) be the subspace of $L^2(\mu_N)$ (respectively, $L^2(\mu)$) spanned by eigenfunctions of $\L_N^{\eps}$ (respectively, $\L_{\star}$) with eigenvalues at most $\lambda$. For $u\in V_N^{\eps}(\lambda)$ and $f\in V(\lambda)$, one has
    \[b_{\eps}(u)\le 2\lambda \|u\|_{L^2(\mu_N)}^2,\qquad \sigma_{\eta}D_2(f) \le 2\lambda \|f\|_{L^2(\mu)}^2.\]
    We set $C_2:= \sqrt{2}(1+\sigma_{\eta}^{-1/2}) \cdot C$, and let
    \[ \err(\lambda) := C_2 \Big ( \frac{\tilde{\delta}}{\eps} + \theta + \eps + \eps \sqrt{\lambda} \big ), \qquad \lambda \ge 0.\]
    This choice of $C_2$ ensures that for every $f\in V(\lambda)$
    \[ b_{\eps}(\tilde{P} f) \le (1+\err(\lambda)) \sigma_{\eta}D_2(f),\qquad  \big | \|f\|^2_{L^2(\mu)}-\|\tilde{P}f\|^2_{L^2(\mu_N)}\big | \le \err(\lambda) \|f\|_{L^2(\mu)}^2,\]
    and for every $u\in V_N^{\eps}(\lambda)$
    \[  \sigma_{\eta}D_2(\tilde{\mathcal{I}}u) \le (1+\err(\lambda)) b_{\eps}(u), \qquad  \big | \|u\|^2_{L^2(\mu_N)}-\|\tilde{\mathcal{I}}u\|^2_{L^2(\mu)}\big | \le \err(\lambda)\|u\|^2_{L^2(\mu_N)}.\]
    Let $\lambda>0$ be such that $\err(\lambda)\le 1/8$. We show that
    \begin{equation} \label{eq:sandwich-eps}
         n_{\L_\star}\!\big(\lambda\big(1-4\err(\lambda)\big)\big) \leq n_{\L_N^{\eps}}(\lambda) \leq n_{\L_\star}\!\big(
     \lambda \big (1+4\err(\lambda)\big)
     \big). 
    \end{equation}
\medskip
      \noindent
     \textbf{Upper bound.} We set $\ell:= n_{\L_N^{\eps}}(\lambda)$. We first observe that 
      \[\|\tilde{\mathcal{I}}u\|^2_{L^2(\mu)} \ge (1-\err(\lambda)) \|u\|_{L^2(\mu_N)}^2, \qquad \forall u \in V_N^{\eps}(\lambda).\] 
       Since $\err(\lambda)<1$, it follows that $\tilde{\mathcal{I}}|_{V_N^{\eps}(\lambda)}$ is injective, hence an isomorphism onto its range $\tilde{\mathcal{I}}(V_N^{\eps}(\lambda))$. In particular, $\dim \tilde{\mathcal{I}}(V_N^{\eps}(\lambda)) = \ell$. 

       \medskip
       
    Let $f= \tilde{\mathcal{I}}u \in \tilde{\mathcal{I}}(V_N^{\eps}(\lambda)) \subset H^1(\mathcal{M})$.
    We have
    \[\sigma_{\eta}D_2(f) 
        \le \Big (1+C\Big (\frac{\tilde{\delta}}{\eps} + \eps + \theta \Big )\Big ) b_{\eps}(u)\le  \big (1+\err(\lambda) \big ) \cdot 2\lambda \|u\|^2_{L^2(\mu_N)}.\]
   Combining with $\|f\|^2_{L^2(\mu)} = \|\tilde{\mathcal{I}}u\|^2_{L^2(\mu)} \ge (1-\err(\lambda)) \|u\|_{L^2(\mu_N)}^2$, one has
   \[
   \frac{\sigma_{\eta}D_2(f)}{2\|f\|^2_{L^2(\mu)}}
   \le\lambda\Big (\frac{1+\err(\lambda)}{1-\err(\lambda)}\Big) \le  \lambda(1+4\err(\lambda)),\qquad \forall f\in  \tilde{\mathcal{I}}(V_N^{\eps}(\lambda)) \setminus \{0\},\]
   where we used the inequality $(1+x)/(1-x)\le 1+4x$, which holds for $0\le x\le 1/8$. 
   Since $\tilde{\mathcal{I}}(V_N^{\eps}(\lambda))$ is an $\ell$-dimensional subspace of $H^1(\mathcal{M})$,
   by the variational characterization of $\lambda_{\ell}(\L_\star)$, one has
    \[\lambda_{\ell}(\L_\star) \le \lambda(1+4\err(\lambda)).\] 
   Hence,
    \[n_{\L_N^{\eps}}(\lambda) = \ell \le n_{\L_{\star}}(\lambda_{\ell}(\L_\star))
        \le  n_{\L_{\star}} (\lambda (1+4\err(\lambda) ) ).\]

\medskip
   \noindent
   \textbf{Lower bound.} Let $k:= n_{\L_{\star}}( \lambda(1-4\err(\lambda)))$ and let $W:= V(\lambda(1-4\err(\lambda)))$. By the same argument, 
     $\|\tilde P f\|^2_{L^2(\mu_N)} \ge (1-\err(\lambda)) \|f\|^2_{L^2(\mu)}$ for all $f\in W$. Thus, 
   $\tilde P|_W$ is bijective from $W$ to its range $\tilde{P}(W)$, and hence $\dim \tilde{P}(W) = k$. For $f\in W$, we have
   \[ b_{\eps}(\tilde{P}f) \le (1+\err(\lambda)) \sigma_{\eta} D_2(f) \le 2\lambda (1+\err(\lambda))(1-4\err(\lambda)) \|f\|^2_{L^2(\mu)}.\]
   Consequently, for every $u=\tilde P f \in \tilde{P}(W)\setminus\{0\}$
   \[\frac{\langle u, \L_N^\eps u\rangle_{\mu_N} }{\|u\|^2_{L^2(\mu_N)}} = \frac{b_{\eps}(\tilde{P}f)}{2\|\tilde P f\|^2_{L^2(\mu_N)}} \le \lambda \frac{(1+\err(\lambda))(1-4\err(\lambda))}{1-\err(\lambda)} \le \lambda,\]
   since $\err(\lambda)<1/4$. By the variational characterization of $\lambda_k(\L_N^{\eps})$, we have $\lambda_k(\L_N^{\eps}) \le \lambda$. Thus, 
   \[ n_{\L_N^{\eps}}(\lambda) \ge k =  n_{\L_\star}( \lambda(1-4\err(\lambda) )).\]
   Hence, the claim \eqref{eq:sandwich-eps} follows.

   We now apply Theorem \ref{thm:abstract-loc-spec-counting-graph} to $\L_N^{\eps}$. Then, on the event $\mathcal{G}_N^{\eps}$, we have $n \asymp \lambda_n(\L_N^{\eps})^{m/2}$ and 
   \[\mathcal{R}_{\L_N^{\eps}}(n) \lesssim n^{1-1/m}+n\err(\lambda_n(\L_N^{\eps})),\]
  for every $n$ such that $\lambda_n(\L_N^{\eps})\ge 1$ and $\err(\lambda_n(\L_N^{\eps})) \le 1/8$. 
  We note that $\err(\lambda) \lesssim \tilde{\delta}/\eps + \theta +\eps + \eps \sqrt{\lambda}$, and $\lambda_n(\L_N^{\eps}) \asymp n^{2/m}$. 
  Therefore, on the event $\mathcal{G}_N^{\eps}$:
   \[\mathcal{R}_{\L_N^{\eps}}(n) \lesssim n^{1-1/m}+n \Big (\frac{\tilde{\delta}}{\eps} + \eps + \theta + \eps n^{1/m}\Big ), \]
   for every $n$ such that $\lambda_n(\L_N^{\eps})\ge 1$ and $\err(\lambda_n(\L_N^{\eps})) \le 1/8$.
\end{proof}

\paragraph{Carr\'e du champ of $\eps$-graphs}

\begin{lemma}\label{lm:eps-carreduchamp}
    Let $\eps, \theta$ and $\tilde{\delta}$ be as in Theorem \ref{thm:var-reduce-eps}. 
    Then, there exist structural constants $C_1,C_2$ such that: with probability at least $1-C_1N \exp(-C_1^{-1}N\theta^2\tilde{\delta}^m)$, one has
    \[ \|\Gamma_{\L_N^{\eps}}(f)\|_{L^\infty(\mathcal{X})} \le C_2 \lip^2_{\mathcal{M}}(f),\]
    for every Lipschitz function $f:\mathcal{M}\to \R$ with the Lipschitz constant $\lip_{\mathcal{M}}(f)$. 
\end{lemma}

\begin{proof}[Proof of Lemma \ref{lm:eps-carreduchamp}]
     We first note that for $x_i \in \mathcal{X}$,
     \begin{eqnarray*}
        \Gamma_{\L_N^{\eps}}(f)(x_i) =  \frac{1}{2N\eps^{m+2}}
   \sum_{j=1}^Nw_{ij}^{\eps}\bigl(f(x_i)-f(x_j)\bigr)^2 \le  \frac{\lip_{\mathcal{M}}^2(f)}{2N \eps^{m+2}}\sum_{j=1}^N w_{ij}^{\eps}\d_{\mathcal{M}}(x_i,x_j)^2. 
     \end{eqnarray*}
     Recall that in the construction of $\eps$-graph on $\mathcal{X}$, $w_{ij}^{\eps} \ne 0$ only if $|x_i-x_j|\le \eps$.  We have the following property (cf. Equation (3.3) in \cite{calder2022improved}): for $x,y \in \mathcal{M}$ such that $|x-y|\le R_{\mathcal{M}}/2$, one has
     \begin{equation} \label{eq:distance-M-comparison}
         |x-y| \le \d_{\mathcal{M}}(x,y) \le |x-y|+\frac{8}{R_{\mathcal{M}}^2}|x-y|^3.
     \end{equation}
     Since $\eps \le R_{\mathcal{M}}/4$, by \eqref{eq:distance-M-comparison} it follows that $\d_{\mathcal{M}}(x,y) \lesssim \eps$ for $|x-y|\le \eps$.
     Thus, 
    \begin{eqnarray*}
        \Gamma_{\L_N^{\eps}}(f)(x_i) 
        &\lesssim &  \lip_{\mathcal{M}}^2(f)\Big (\frac{1}{N \eps^{m}}\sum_{j=1}^N \eta\Big(\frac{|x_i-x_j|}{\eps} \Big )\Big ) \\
        &=& \lip_{\mathcal{M}}^2(f)\Big (\frac{1}{ \eps^{m}}\int_{\mathcal{X}}\eta\Big(\frac{|x_i-x|}{\eps} \Big )\mu_N(\d x)\Big ).
    \end{eqnarray*}
    
    By \cite[Proposition 2.12]{calder2022improved}, with probability at least $1-N \exp(-CN\theta^2 \tilde{\delta}^m)$, there exists a density function $\tilde{\rho}_N$ on $\mathcal{M}$ with $\tilde{\mu}_N:=\tilde{\rho}_N\vol_g$ satisfying
    \[ \min_{T_{\#}\tilde{\mu}_N = \mu_N} \sup_{x\in \mathcal{M}} \d_{\mathcal{M}}(x,T(x))\le \tilde{\delta}, \quad \|\rho-\tilde{\rho}_N\|_{\infty} \le C(\theta + \tilde{\delta}).\]
    Here $\d_{\mathcal{M}}$ denotes the geodesic distance on $\mathcal{M}$. 
    We denote such an optimal transport map by $\tilde{T}: \mathcal{M}\to \mathcal{X}$. Then
    \begin{eqnarray*}
       \int_{\mathcal{X}}\eta\Big(\frac{|x_i-x|}{\eps} \Big )\mu_N(\d x) &=& \int_{\mathcal{M}}\eta\Big(\frac{|x_i-\tilde{T}(x)|}{\eps} \Big )\tilde{\mu}_N(\d x) \\
       &\le& \eta(0) \tilde{\mu}_N\Big (\big \{x\in \mathcal{M}: |x_i-\tilde{T}(x)| \le \eps \big \}\Big )\\
        &\le& \eta(0) \|\tilde{\rho}_N\|_{\infty} \vol_g\Big (\big \{x\in \mathcal{M}: |x_i-\tilde{T}(x)| \le \eps \big \}\Big ),
    \end{eqnarray*}
    where we used the fact that $\eta(t)$ is non-increasing and supported on $[0,1]$, and $\tilde{\mu}_N=\tilde{\rho}_N\vol_g$. Note that 
    $ \|\rho-\tilde{\rho}_N\|_{\infty} \le C(\theta + \tilde{\delta}) \lesssim \rho_{\min}$, which implies $\|\tilde{\rho}_N\|_{\infty} \lesssim  \rho_{\max}+\rho_{\min}$. 
    Therefore,
    \[ \Gamma_{\L_N^{\eps}}(f)(x_i) \lesssim \lip_{\mathcal{M}}^2(f) \cdot\eps^{-m} \cdot\vol_g\Big (\big \{x\in \mathcal{M}: |x_i-\tilde{T}(x)| \le \eps \big \}\Big ).\]

    We now control $\vol_g(\{x\in \mathcal{M}: |x_i-\tilde{T}(x)| \le \eps\} )$.
    Let $x\in \mathcal{M}$ be such that $|x_i - \tilde{T}(x)|\le \eps$. One has
    \begin{eqnarray*}
        \d_{\mathcal{M}}(x,x_i) &\le& \d_{\mathcal{M}}(x,\tilde{T}(x)) + \d_{\mathcal{M}}(\tilde{T}(x),x_i) \\
        &\le& \tilde{\delta} + |x_i-\tilde{T}(x)| + \frac{8}{R_{\mathcal{M}}^2} |x_i-\tilde{T}(x)|^3 
        \le \tilde{\delta} + \eps +\frac{8}{R_{\mathcal{M}}^2}\eps^3
    \end{eqnarray*}
    where we used \eqref{eq:distance-M-comparison}.
    Since $\eps \le R_{\mathcal{M}}/4$ and $\tilde{\delta} \le \eps/4$, we then have $\d_{\mathcal{M}}(x,x_i)\le 7\eps /4$.
    Let $B_{\mathcal{M}}(x,r)$  denote the geodesic ball in $\mathcal{M}$ of radius $r$ centered at $x$. Then
    \begin{eqnarray*}
        \vol_g\big (\big \{x\in \mathcal{M}: |x_i-\tilde{T}(x)| \le \eps \big \}\big) \le\vol_g\big (B_{\mathcal{M}}\big (x_i, \frac{7\eps}{4}\big )\big)\cdot
    \end{eqnarray*}
    Note that $|\vol_g(B_{\mathcal{M}}(x,r))-\omega_mr^m|\lesssim K_{\mathcal{M}}r^{m+2}$ (cf. Equation 3.2 in \cite{calder2022improved}). Applying this for $x=x_i$ and $r=7\eps/4$, and noting that $\eps < K_{\mathcal{M}}^{-1/2}/2$, one has
    \[   \vol_g\big (\big \{x\in \mathcal{M}: |x_i-\tilde{T}(x)| \le \eps \big \}\big) \lesssim \eps^m.\]
    The lemma follows.
\end{proof}

\paragraph{Proof of Theorem \ref{thm:var-reduce-eps}.}

\begin{proof}
    Applying Theorem \ref{thm:var-reduce-general} gives
\begin{equation} \label{eq:proof-var-reduc-eps}
    \var[\mathcal{S}^{\eps}_{n,N}(f)|\mathcal{X}] \lesssim (\|f\|^2_{L^{\infty}(\mathcal{X})} + \|\Gamma_{\L_N^{\eps}}(f)\|_{L^{\infty}(\mathcal{X})}) \cdot \mathcal{R}_{\L^{\eps}_N}(n).
\end{equation}
 By Lemma \ref{thm:loc-spec-eps}, there are structural constants $\tilde{C}_1,\tilde{C}_2$ such that with probability at least $1-\tilde{C}_1N \exp(-\tilde{C}_1^{-1}N\theta^2\tilde{\delta}^m)$, we have
\[\mathcal{R}_{\L_N^{\eps}}(n) \lesssim n^{1-1/m}+n \Big (\frac{\tilde{\delta}}{\eps} + \eps + \theta + \eps n^{1/m}\Big ),\]
 for every $n$ such that $\lambda_n(\L_N^{\eps}) \ge 1$ and $\err(\lambda_n(\L_N^{\eps})) \le 1/8$. 
 Here \[ \err(\lambda):= \tilde{C}_2 \Big (\frac{\tilde{\delta}}{\eps} + \theta + \eps + \eps \sqrt{\lambda}\Big ), \qquad \lambda \ge 0.\]

By Lemma \ref{lm:eps-carreduchamp}, we have $\|\Gamma_{\L_N^{\eps}}(f)\|_{L^\infty(\mathcal{X})} \lesssim \lip^2_{\mathcal{M}}(f)$ with probability at least $1-N \exp(-CN\theta^2 \tilde{\delta}^m)$.
Plugging into \eqref{eq:proof-var-reduc-eps}, and choosing $C_1>0$ large enough so that the intersection of the events in Lemma \ref{lm:eps-carreduchamp} and Lemma \ref{thm:loc-spec-eps} occurs with probability at least $1-C_1N\exp(-C_1^{-1}N \theta^2 \tilde{\delta}^m)$, we have
\[\frac{1}{n^2} \var[\mathcal{S}^{\eps}_{n,N}(f)|\mathcal{X}] \lesssim (\|f\|^2_{L^{\infty}(\mathcal{X})} +  \lip^2_{\mathcal{M}}(f))n^{-1-1/m} \Big [1+\Big (\frac{\tilde{\delta}}{\eps} + \eps + \theta\Big )n^{1/m} + \eps n^{2/m}\Big ]\]
  for every $n$ such that $\lambda_n(\L_N^{\eps}) \ge 1$ and $E_N(\lambda_n(\L_N^{\eps})) \le 1/8$.

 Finally, we claim that there exists a structural constant $C_2>0$ such that
 \[\frac{\tilde{\delta}}{\eps} + \eps + \theta + \eps n^{1/m} \le C_2 \quad \text{ implies } \quad \err(\lambda_n(\L_N^{\eps})) \le 1/8.\]
 Indeed, let $B>0$ be a large enough structural constant such that $\lambda_j(\L_{\star}) \le B \cdot j^{2/m}$ for every $j\in \N$. We choose a structural constant $C_2>0$ small enough such that 
 \[\frac{\tilde{\delta}}{\eps} + \eps + \theta + \eps n^{1/m} \le C_2 \quad \text{ implies } \quad \err(2B \cdot n^{2/m}) \le 1/8.\]
  Since $\err(2B \cdot n^{2/m}) \le 1/8$, we can apply Lemma \ref{thm:loc-spec-eps} for $\lambda = 2 B \cdot n^{2/m}$, which yields
\[ n_{\L_N^{\eps}}(2B \cdot n^{2/m}) \ge n_{\L_{\star}}(B \cdot n^{2/m}) \ge n.\]
This implies $\lambda_n(\L_N^{\eps}) \le 2B \cdot n^{2/m}$. Hence, $\err(\lambda_n(\L_N^{\eps})) \le \err(2B\cdot n^{2/m}) \le 1/8$ as desired.
This completes the proof.
\end{proof}

\subsubsection{Proof of Corollary \ref{cor:epsilon-concrete-scaling}}

\begin{proof}
Fix \(r>0\), and set
$
    \varepsilon_N
    :=
    \left(\frac{\log N}{N}\right)^{\frac{1}{2m+2}},
    \
    \widetilde\delta_N:=\varepsilon_N^2,
    \
    \theta_N:=A_r\varepsilon_N,
$
where \(A_r>0\) will be chosen below.

We first verify the assumptions of Theorem~\ref{thm:var-reduce-eps}. Since
\[
    \frac{\widetilde\delta_N}{N^{-1/m}}
    =
    \frac{\varepsilon_N^2}{N^{-1/m}}
    =
    (\log N)^{\frac1{m+1}}
    N^{\frac{1}{m(m+1)}}
    \longrightarrow \infty,
\]
we have \(N^{-1/m}\leq \widetilde\delta_N\) for all sufficiently
large \(N\). Moreover,
\[
    \frac{\widetilde\delta_N}{\varepsilon_N}
    =\varepsilon_N\longrightarrow0,
    \qquad
    \theta_N+\widetilde\delta_N
    =A_r\varepsilon_N+\varepsilon_N^2
    \longrightarrow0.
\]
Thus, after increasing \(N_0(r)\) if necessary,
$
    \widetilde\delta_N\leq \frac{\varepsilon_N}{4},
    \
    C(\theta_N+\widetilde\delta_N)\leq \frac{\rho_{\min}}2,
$
and the geometric restriction on \(\varepsilon_N\) in Theorem~\ref{thm:var-reduce-eps}
is also satisfied.
\\
The failure probability in Theorem~\ref{thm:var-reduce-eps} is bounded by
$
    C_1N
    \exp(
        -C_1^{-1}N\theta_N^2\widetilde\delta_N^{\,m}
    ).
$
By the choice of parameters,
$
    N\theta_N^2\widetilde\delta_N^{\,m}
    =
    A_r^2N\varepsilon_N^{2m+2}
    =
    A_r^2\log N.
$
Hence
\[
    C_1N
    \exp\!\left(
        -C_1^{-1}A_r^2\log N
    \right)
    =
    C_1N^{1-C_1^{-1}A_r^2}.
\]
Choosing \(A_r\) sufficiently large, depending only on \(r\) and
the structural constants, this is bounded by \(C_rN^{-r}\).

It remains to verify the restrictions on \(n\). Let \(\L^\star\)
denote the continuum operator appearing in Lemma~\ref{thm:loc-spec-eps}. Since
\[
    E_N(1)
    \lesssim
    \frac{\widetilde\delta_N}{\varepsilon_N}
    +\theta_N+\varepsilon_N
    \longrightarrow0,
\]
we may assume that \(E_N(1)\leq 1/8\). On the event of
Lemma~\ref{thm:loc-spec-eps},
\[
    n_{\L_N^{\varepsilon_N}}(1)
    \leq
    n_{\L^\star}\!\left(1+4E_N(1)\right)
    \leq
    n_{\L^\star}(3/2).
\]
Consequently, if
$
    n_0:=1+n_{\L^\star}(3/2),
$
then \(n_0\) is a structural integer and
$
    \lambda_n(\L_N^{\varepsilon_N})>1$ for  $n\geq n_0$.

Now suppose that
$
    n_0\leq n\leq \varepsilon_N^{-m/2},
$
so that
$
    \varepsilon_N n^{1/m}
    \leq \varepsilon_N^{1/2}.
$
Then
\[
    \frac{\widetilde\delta_N}{\varepsilon_N}
    +\varepsilon_N+\theta_N
    +\varepsilon_N n^{1/m}
    \leq
    (A_r+2)\varepsilon_N+\varepsilon_N^{1/2}
    \longrightarrow0.
\]
Thus the remaining hypothesis of Theorem~\ref{thm:var-reduce-eps} holds for all
sufficiently large \(N\).
Applying Theorem~\ref{thm:var-reduce-eps} gives
\[
\begin{aligned}
    \frac1{n^2}
    \Var\!\left[
        \mathcal S_{n,N}^{\varepsilon_N}(f)\mid \mathcal X
    \right]
    &\lesssim
    \bigl(
        \|f\|_{L^\infty(\mathcal X)}^2+\Lip_{\mathcal M}(f)^2
    \bigr)
    n^{-1-1/m}
\\
    &\qquad\times
    \left[
        1+
        \left(
            \frac{\widetilde\delta_N}{\varepsilon_N}
            +\varepsilon_N+\theta_N
        \right)n^{1/m}
        +\varepsilon_N n^{2/m}
    \right].
\end{aligned}
\]
For \(n\leq\varepsilon_N^{-m/2}\),
\[
    \left(
        \frac{\widetilde\delta_N}{\varepsilon_N}
        +\varepsilon_N+\theta_N
    \right)n^{1/m}
    \leq
    (A_r+2)\varepsilon_N^{1/2}.
\]
Hence the quantity in square brackets is bounded by a constant
depending only on \(r\) and the structural data. We obtain
\[
    \frac1{n^2}
    \Var\!\left[
        \mathcal S_{n,N}^{\varepsilon_N}(f)\mid \mathcal X
    \right]
    \lesssim_r
    \bigl(
        \|f\|_{L^\infty(\mathcal X)}^2+\Lip_{\mathcal M}(f)^2
    \bigr)
    \cdot n^{-1-1/m},
\]
simultaneously for every Lipschitz \(f:\mathcal M\to\mathbb R\) and every
$
    n_0\leq n\leq \varepsilon_N^{-m/2}.
$
Finally,
$
    \varepsilon_N^{-m/2}
    =
    \left(\frac{N}{\log N}\right)^{\frac{m}{4m+4}},
$
which proves the result.
\end{proof}


\subsection{Proofs for $k$-NN graphs}

\subsubsection{Proof of Theorem \ref{thm:var-reduce-knn}}\label{subsec:graph-proof-knn}

\paragraph{Local spectral counting of $k$-NN graphs.}

In the setting of $k$-NN, the continuum operator is $\L_{\star} = \sigma_{\eta} \Delta_{\rho}^{\knn}$, where
\begin{equation*}
    \Delta_\rho^\knn f
    :=-\frac{1}{2\rho}\operatorname{div}_g(\rho^{1-2/m}\nabla_g f), \qquad f\in C^{\infty}(\mathcal{M}).
\end{equation*}

\begin{lemma}[Local spectral counting functions of $k$-NN graphs] \label{thm:loc-spec-knn}
    For $k$-NN graphs, let $k, \theta$ and $\tilde{\delta}$ be as in Theorem \ref{thm:var-reduce-knn}. Let $\L_{\star}=\sigma_{\eta}\Delta_{\rho}^{\knn}$. Then, there exist structural constants $C_1, C_2$, and an event $\mathcal{G}_N^{k}$ such that 
    \[ \mathbb{P}(\mathcal{G}_N^{k}) \ge 1-C_1N \exp(-C_1^{-1}N\theta^2\tilde{\delta}^m),\]
    and on $\mathcal{G}_N^{k}$, the following holds for all $\lambda > 0$ such that $\err(\lambda)\le 1/8$
     \[ n_{\L_\star}\!\big(\lambda\big(1-4\err(\lambda)\big)\big) \leq n_{\L_N^{k}}(\lambda) \leq n_{\L_\star}\!\big(
     \lambda \big (1+4\err(\lambda)\big)
     \big), \]
     where
    \[ \err(\lambda):= C_2 \Big (\frac{\tilde{\delta}}{(k/N)^{1/m}} + \theta + (k/N)^{1/m} + (k/N)^{1/m} \sqrt{\lambda}\Big ), \qquad \lambda \ge 0.\]
    Consequently, on $\mathcal{G}_N^k$ we have $n \asymp \lambda_n(\L_N^{k})^{m/2}$ and 
    \begin{equation*} 
      \mathcal{R}_{\L_N^{k}}(n) \lesssim n^{1-1/m}+n \Big (\frac{\tilde{\delta}}{(k/N)^{1/m}} + (k/N)^{1/m} + \theta + (k/N)^{1/m} n^{1/m}\Big ),
    \end{equation*} 
    for every $n$ such that $\lambda_n(\L_N^{k})\ge 1$ and $\err(\lambda_n(\L_N^{k})) \le 1/8$.
\end{lemma}

The proof follows exactly the same argument as the case of $\eps$-graphs. Hence we omit the proof. We remark that the discrete-to-continuum comparison result used in the case of $\eps$-graphs (i.e. Proposition \ref{prop:comparison}) is replaced by the following result:

\begin{proposition}[Propositions 4.4 and 4.5 in \cite{calder2022improved}]\label{prop:comparison-knn}
    Let $k, \theta$ and $\tilde{\delta}$  be as in Theorem \ref{thm:var-reduce-knn}.
    Let 
    \[b_{k}(u):=2\langle \L_N^{k} u ,u\rangle_{L^2(\mu_N)},\qquad u\in L^2(\mu_N)\] 
    and 
    \[ D_{1-2/m}(f):= \int_{\mathcal{M}}|\nabla_g f(x)|^2 \rho^{1-2/m}(x) \vol_g(\d x), \qquad f\in H^1(\mathcal{M}).\]
    There exists an event $\mathcal{G}_{\ref{prop:comparison-knn}}$ with probability at least 
    $1-CN\exp(-C^{-1}N\theta^2 \tilde{\delta}^m)$, such that on $\mathcal{G}_{\ref{prop:comparison-knn}}$, there are linear maps $\tilde{P}:L^2(\mu)\to L^2(\mu_N)$ and $\tilde{\mathcal{I}}: L^2(\mu_N) \to \lip(\mathcal{M})$ satisfying:
    \begin{enumerate}
        \item[(i)] For every $f\in \lip(\mathcal{M})$,
        \begin{equation*}
            b_{k}(\tilde{P}f) \le \Big (1+C\Big (\frac{\tilde{\delta}}{(k/N)^{1/m}} + (k/N)^{1/m} + \theta \Big ) \Big ) \sigma_{\eta}D_{1-2/m}(f),
        \end{equation*}
        and
        \begin{equation*}
            \big | \|f\|^2_{L^2(\mu)}-\|\tilde{P}f\|^2_{L^2(\mu_N)}\big | \le C\tilde{\delta}\|f\|_{L^2(\mu)} \sqrt{D_{1-2/m}(f)} + C(\theta+\tilde{\delta})\|f\|^2_{L^2(\mu)}.
        \end{equation*}
        \item[(ii)] For every $u\in L^2(\mu_N)$,
        \begin{equation*}
            \sigma_{\eta}D_{1-2/m}(\tilde{\mathcal{I}}u) \le \Big (1+C\Big (\frac{\tilde{\delta}}{(k/N)^{1/m}} + (k/N)^{1/m} + \theta \Big ) \Big ) b_{k}(u),
        \end{equation*}
        and
        \begin{equation*}
            \big | \|u\|^2_{L^2(\mu_N)}-\|\tilde{\mathcal{I}}u\|^2_{L^2(\mu)}\big | \le C\Big (\frac{k}{N}\Big )^{1/m}\|u\|_{L^2(\mu_N)} \sqrt{b_{k}(u)} + C(\theta + \tilde{\delta}) \|u\|^2_{L^2(\mu_N)}.
        \end{equation*}
    \end{enumerate}
\end{proposition}

\paragraph{Carr\'e du champ of $k$-NN graphs.}

\begin{lemma}\label{lm:knn-carreduchamp}
    Let $k, \theta$ and $\tilde{\delta}$ be as in Theorem \ref{thm:var-reduce-knn}. 
     Then, there exist structural constants $C_1,C_2$ such that: with probability at least $1 -C_1N  \exp(-C_1^{-1}N\theta^2\tilde{\delta}^m)$,
     one has
    \[ \|\Gamma_{\L_N^{k}}(f)\|_{L^\infty(\mathcal{X})} \le C_2 \lip^2_{\mathcal{M}}(f),\]
    for every Lipschitz function $f:\mathcal{M}\to \R$ with the Lipschitz constant $\lip_{\mathcal{M}}(f)$. 
\end{lemma}

\begin{proof}
    We recall that 
    \begin{eqnarray*}
 \Gamma_{\L_N^{k}}(f)(x_i)
 &:=&\frac{\omega_m^{1+2/m}}{2N}\left(\frac{N}{k}\right)^{1+2/m}
   \sum_{j=1}^N
   \eta\!\left(\frac{|x_i-x_j|}{r_k(x_i,x_j)}\right)
   \bigl(f(x_i)-f(x_j)\bigr)^2 \\
   &\lesssim& \lip_{\mathcal{M}}(f)^2 \Big (\frac{k}{N}\Big )^{-\frac{m+2}{m}} \frac{1}{N}\sum_{j=1}^N
   \eta\!\left(\frac{|x_i-x_j|}{r_k(x_i,x_j)}\right)
   \d_{\mathcal{M}}(x_i,x_j)^2,
\end{eqnarray*}
where $r_k(x_i,x_j):=\max\{\eps_k(x_i),\eps_k(x_j)\}$, and \(\varepsilon_k(x_i)\) denotes the Euclidean distance from \(x_i\) to its \(k\)-th nearest neighbour.

By \cite[Lemma 3.9]{calder2022improved}, for some structural constants $c_0, C_0>0$, we have
\[ \mathbb P \big (|\omega_m \rho(x_i)N\eps_k(x_i)^m - k| \ge c_0k\big ) \le 4 \exp(-C_0 k), \qquad \forall x_i \in \mathcal{X}. \]
Thus, there exists an event $\mathcal{G}$ with probability at least $1-4N\exp(-C_0k)$ on which we have
\[ \eps_k(x_i)^m \rho(x_i) \le \frac{(1+c_0)k}{\omega_m N}, \qquad \forall x_i \in \mathcal{X}.\]
Since $\rho(x_i) \ge \rho_{\min}>0$, this implies $\eps_k(x_i) \le c (k/N)^{1/m}$ for some structural constant $c>0$. Thus, on this event, $r_k(x_i,x_j) \le c (k/N)^{1/m}$ for all $x_i,x_j$. Finally, the conditions on $k, \theta, \tilde{\delta}$ in Theorem \ref{thm:var-reduce-knn} imply that $k\gtrsim N \theta^2 \tilde{\delta}^m$. Thus, by choosing a structural constant $C_1>0$ large enough, we have $\mathbb{P}(\mathcal{G}) \ge 1-C_1N\exp(-C_1^{-1}N\theta^2 \tilde{\delta}^m)$.

Recall that by the construction of the symmetric $k$-NN graph, we join $x_i$ and $x_j$ only if $|x_i-x_j|\le r_k(x_i,x_j)$.
On $\mathcal{G}$, one has $r_k(x_i,x_j) \le c(k/N)^{1/m}$. 
We choose the constants so that $r_k(x_i,x_j) \le c(k/N)^{1/m}\le R_{\mathcal{M}}/2$.
Then by \eqref{eq:distance-M-comparison}, one has $\d_{\mathcal{M}}(x_i,x_j) \lesssim r_k(x_i,x_j)$ for every $i,j$ such that $|x_i-x_j| \le r_k(x_i,x_j)$. 
Set $h:= c (k/N)^{1/m}$. Since $\eta$ is non-increasing, on $\mathcal{G}$ one has
\[ \Gamma_{\L_N^{k}}(f)(x_i) \lesssim \lip_{\mathcal{M}}(f)^2 \cdot h^{-m}\frac{1}{N}\sum_{j=1}^N \eta\!\left(\frac{|x_i-x_j|}{h}\right).\]
Following the same argument as the proof of Lemma \ref{lm:eps-carreduchamp} and taking the intersection of the events yields the result.
\end{proof}

\paragraph{Proof of Theorem \ref{thm:var-reduce-knn}}

The argument is similar to the proof of Theorem \ref{thm:var-reduce-eps}. We omit the details.

\subsubsection{Proof of Corollary \ref{cor:knn-concrete-scaling}}

\begin{proof}
Fix \(r>0\), and set
$
    k_N
    :=
    \left\lceil
        N^{\frac{m+2}{2m+2}}
        (\log N)^{\frac{m}{2m+2}}
    \right\rceil,
    \
    h_N:=\left(\frac{k_N}{N}\right)^{1/m},
    \
    \widetilde\delta_N:=h_N^2,
    \
    \theta_N:=A_rh_N,
$
where \(A_r>0\) will be chosen below. By construction,
\[
    h_N
    \asymp
    \left(\frac{\log N}{N}\right)^{\frac{1}{2m+2}}.
\]

We first verify the assumptions of Theorem~\ref{thm:var-reduce-knn}. As in the proof of
Corollary~\ref{cor:epsilon-concrete-scaling},
\[
    \frac{\widetilde\delta_N}{N^{-1/m}}
    =
    \frac{h_N^2}{N^{-1/m}}
    \longrightarrow\infty,
    \qquad
    \frac{\widetilde\delta_N}{h_N}
    =h_N\longrightarrow0,
    \qquad
    \theta_N+\widetilde\delta_N
    =A_rh_N+h_N^2
    \longrightarrow0.
\]
Thus, after increasing \(N_0(r)\) if necessary,
$
    \widetilde\delta_N\leq c_\rho h_N,
    \
    C(\theta_N+\widetilde\delta_N)\leq \frac{\rho_{\min}}2,
$
and the geometric restriction on \(h_N=(k_N/N)^{1/m}\) in
Theorem~\ref{thm:var-reduce-knn} is also satisfied.

The failure probability in Theorem~\ref{thm:var-reduce-knn} is bounded by
$
    C_1N
    \exp(
        -C_1^{-1}N\theta_N^2\widetilde\delta_N^{\,m}
    ).
$
By the choice of parameters,
\[
    N\theta_N^2\widetilde\delta_N^{\,m}
    =
    A_r^2Nh_N^{2m+2}
    \gtrsim
    A_r^2\log N.
\]
Hence, exactly as in the proof of Corollary~\ref{cor:epsilon-concrete-scaling}, choosing \(A_r\)
sufficiently large, depending only on \(r\) and the structural
constants, the failure probability is bounded by \(C_rN^{-r}\).

It remains to verify the restrictions on \(n\). Let \(\L^\star\)
denote the continuum operator appearing in Lemma~\ref{thm:loc-spec-knn}.
Since
\[
    E_N(1)
    \lesssim
    \frac{\widetilde\delta_N}{h_N}
    +\theta_N+h_N
    \longrightarrow0,
\]
the same argument as in the proof of Corollary~\ref{cor:epsilon-concrete-scaling} shows that there
exists a structural integer \(n_0\geq1\) such that
$
    \lambda_n(\L_N^{k_N})>1$
for $n\geq n_0$.

Now suppose that
$
    n_0\leq n\leq h_N^{-m/2},
$
so that
$
    h_Nn^{1/m}
    \leq h_N^{1/2}.
$
Then 
\[
    \frac{\widetilde\delta_N}{h_N}
    +h_N+\theta_N
    +h_Nn^{1/m}
    \leq
    (A_r+2)h_N+h_N^{1/2}
    \longrightarrow0.
\]
Thus the remaining hypothesis of Theorem~\ref{thm:var-reduce-knn} holds for all
sufficiently large \(N\).
Applying Theorem~\ref{thm:var-reduce-knn} gives
\[
\begin{aligned}
    \frac1{n^2}
    \Var\!\left[
       \mathcal S_{n,N}^{k_N}(f)\mid 
       \mathcal X
    \right]
    &\lesssim
    \bigl(
        \|f\|_{L^\infty(\mathcal X)}^2+\Lip_{\mathcal M}(f)^2
    \bigr)
    n^{-1-1/m}
\\
    &\qquad\times
    \left[
        1+
        \left(
            \frac{\widetilde\delta_N}{h_N}
            +h_N+\theta_N
        \right)n^{1/m}
        +h_Nn^{2/m}
    \right].
\end{aligned}
\]
For \(n\leq h_N^{-m/2}\),
\[
    \left(
        \frac{\widetilde\delta_N}{h_N}
        +h_N+\theta_N
    \right)n^{1/m}
    \leq
    (A_r+2)h_N^{1/2}.
\]
Hence the quantity in square brackets is bounded by a constant
depending only on \(r\) and the structural data. We obtain
\[
    \frac1{n^2}
    \Var\!\left[
        \mathcal S_{n,N}^{k_N}(f)\mid \mathcal X
    \right]
    \lesssim_r
    \bigl(
        \|f\|_{L^\infty(\mathcal X)}^2+\Lip_{\mathcal M}(f)^2
    \bigr)
    \cdot n^{-1-1/m},
\]
simultaneously for every Lipschitz \(f:\mathcal M\to\mathbb R\) and every
$
    n_0\leq n\leq h_N^{-m/2}.
$
Finally,
\[
    h_N^{-m/2}
    \asymp
    \left(\frac{N}{\log N}\right)^{\frac{m}{4m+4}},
\]
which proves the result.
\end{proof}

\subsection{Application to minibatch sampling: Theorems \ref{thm:eps-minibatch} and \ref{thm:knn-minibatch}}

\subsubsection{Technical lemmas}
In this section, let $\L_N$ be either $\L_N^{\eps}$ or $\L_N^k$. Let $\L_{\star}$ denote the corresponding continuum operator; in particular,
\[
    \L_\star f:=
    \begin{cases}
        -\dfrac{\sigma_\eta}{2\rho}
        \operatorname{div}_g(\rho^2\nabla_g f),
        &\text{for the $\varepsilon$-graph},\\
        -\dfrac{\sigma_\eta}{2\rho}
        \operatorname{div}_g(\rho^{1-2/m}\nabla_g f),
        &\text{for the symmetric $k$-NN graph}.
    \end{cases}
\]
We denote the ordered eigenvalues of $\L_{\star}$ and $\L_N$ respectively by
\[
    0=\lambda_1^{\star}\le\lambda_2^{\star}\le\cdots,
    \qquad
    0=\lambda_{1,N}\le\lambda_{2,N}\le\cdots\le\lambda_{N,N},
\]
and corresponding real-valued orthonormal eigenbases
$\{\phi_j^{\star}\}_{j\ge1}$ and $\{\phi_{j,N}\}_{j=1}^N$.
Since $\rho$ is smooth, we can choose $\phi_j^{\star}$ to be smooth. We choose $\phi_1^{\star}$ to be the constant function $1$ on $\mathcal{M}$, and $\phi_{1,N}$ to be the constant function $1$ on $\mathcal{X}$.
For $n \in \mathbb{N}$, we set
\[
    \K_n^{\star}(x,y):=
    \sum_{j=1}^n\phi_j^{\star}(x)\phi_j^{\star}(y).
\]

\begin{lemma}
\label{lm:continuum-kernel-local-weyl}
There exists a smooth, strictly positive function $w_\star:\mathcal{M} \to \R$ with  
$\int_{\mathcal M}w_\star\d\mu=1$, depending only on $(\mathcal{M},g)$ and $\rho$, such that the following holds.
There exist structural constants $C,c>0$ such that, for every
integer $n\ge1$,
\begin{equation*}
    \sup_{x\in\mathcal M}
    \left|\K_n^\star(x,x)-n w_\star(x)\right|
    \le C n^{1-1/m}.
\end{equation*}
Moreover,
\begin{equation*}
    cn\le\K_n^\star(x,x)\le Cn,
    \qquad x\in\mathcal M,\quad n\ge1.
\end{equation*}
\end{lemma}

\begin{proof}
For $\Lambda\ge0$, let $ e_\star(\Lambda,x):=\sum_{\lambda_j^\star\le\Lambda}|\phi_j^\star(x)|^2$. Since $\rho$ is smooth and is bounded above and away from zero, both
operators are uniformly elliptic with smooth coefficients.
The local Weyl law (see \cite{hormander1968spectral} and \cite[Theorem~2.4 and Remark~2.5]{rivera2022weighted}) gives,
\begin{equation*}
    e_\star(\Lambda,x)
    =a_\star(x)\Lambda^{m/2}
    +O\bigl(\Lambda^{(m-1)/2}\bigr),
    \qquad \Lambda\longrightarrow\infty;
\end{equation*}
uniformly in $x\in\mathcal M$, where $a_{\star}:\mathcal{M} \to \R$ is smooth and strictly positive, depending only on $(\mathcal{M},g)$, the density $\rho$ and $\sigma_{\eta}$.

Let $e_\star^-(\Lambda,x):=\sum_{\lambda_j^\star<\Lambda}|\phi_j^\star(x)|^2$. Then taking limits yields
\[
    e_\star^-(\Lambda,x)
    =a_\star(x)\Lambda^{m/2}
    +O\bigl(\Lambda^{(m-1)/2}\bigr),
\]
uniformly in $x$, for large $\Lambda$. Since $e_\star^-(\lambda_n^\star,x)
    \le\K_n^\star(x,x)
    \le e_\star(\lambda_n^\star,x),$
we obtain
\begin{equation*}
    \K_n^\star(x,x)
    =a_\star(x)(\lambda_n^\star)^{m/2}
    +O\bigl((\lambda_n^\star)^{(m-1)/2}\bigr) \qquad \text{uniformly in $x$, as $n\to\infty$}.
\end{equation*}
Integrating against $\mu$ and using orthonormality yields
\begin{equation*}
    n=A_\star(\lambda_n^\star)^{m/2}
    +O\bigl((\lambda_n^\star)^{(m-1)/2}\bigr),
\end{equation*}
where $A_{\star}:= \int a_{\star} \d \mu$.
Since $A_\star>0$ and $\lambda_n^\star\to\infty$, this implies
$(\lambda_n^\star)^{m/2}\asymp n$ for large $n$. In particular,
\[
    (\lambda_n^\star)^{(m-1)/2}=O(n^{1-1/m}).
\]
Therefore 
\begin{equation*}
    \K_n^\star(x,x)
    =n\frac{a_\star(x)}{A_\star}+O(n^{1-1/m}) =: nw_{\star}(x) + O(n^{1-1/m}),
\end{equation*}
uniformly in $x$ as desired. Note that $w_{\star}(x)$ is strictly positive and smooth on $\mathcal{M}$, the second assertion follows easily.
\end{proof}

\begin{remark} \label{rmk:density-eps-knn}
    By direct computations using principal symbols
    \[
w_{\star}(x)=
\begin{cases}
    \displaystyle
    \frac{\rho(x)^{-1-m/2}}
    {\int_{\mathcal M}\rho^{-1-m/2}\,\mathrm d\mu},
    & \qquad\text{for the $\varepsilon$-graph},\\[1.2em]
    1,
    & \qquad\text{for the symmetric $k$-NN graph}.
\end{cases}
\]
    Thus, the leading-order
    diagonal weight is independent of $\rho$ for symmetric $k$-NN
    graphs, whereas it depends explicitly on $\rho$ for
    $\varepsilon$-graphs. This suggests that the $k$-NN construction
    is more robust to variations in the sampling density.
\end{remark}



For the Gram matrix $\mathbf G_{n,N}
    :=
    \bigl[
        \langle\phi_j^{\star},\phi_k^{\star}\rangle_{L^2(\mu_N)}
    \bigr]_{1\le j,k\le n}$,
define
\begin{equation} \label{eq:eta-minibatch}
     \eta_{n,N}:=
    \|\mathbf G_{n,N}-\mathbf I_n\|_{\mathrm{op}}.
\end{equation}
    
\begin{lemma}
\label{lm:empirical-gram-concentration}

There exists a structural constant $C_0\ge1$ such that, for every $s>0$,
\begin{equation}
    \mathbb P\bigl(\eta_{n,N}\ge s\bigr)
    \le
    2n\exp\left(
        -\frac{Ns^2}{2C_0n(1+s/3)}
    \right).
    \label{eq:empirical-gram-tail}
\end{equation}
Consequently, for every $\delta\in(0,1)$, with probability
at least $1-\delta$,
\begin{equation}
    \eta_{n,N}
    \le
    \sqrt{\frac{2C_0n\log(2n/\delta)}{N}}
    +
    \frac{2C_0n\log(2n/\delta)}{3N}.
    \label{eq:empirical-gram-confidence}
\end{equation}
In particular, for every $r>0$, $N\ge2$ and $1\le n \le N$, with probability
at least $1-N^{-r}$,
\begin{equation}
    \eta_{n,N}
    \lesssim_r
    \left(
        \sqrt{\frac{n\log N}{N}}
        +
        \frac{n\log N}{N}
    \right).
    \label{eq:empirical-gram-polynomial}
\end{equation}
\end{lemma}

\begin{proof}

We recall that $\mathcal{X}=\{x_1,\ldots,x_N\}$ where $x_i \sim^{\iid} \mu$. Set
\[
    z(x_i)
    :=
    \bigl(\phi_1^\star(x_i),\ldots,\phi_n^\star(x_i)\bigr)^\top
    \in\mathbb R^n,
    \qquad
    Y_i:=z(x_i)z(x_i)^\top-\mathbf I_n.
\]
The matrices $Y_1,\ldots,Y_N$ are independent and symmetric.
By orthonormality in $L^2(\mu)$,
\[
    \mathbb E[z(x_i)z(x_i)^\top]=\mathbf I_n,
    \qquad
    \mathbb E[Y_i]=0,
    \qquad
     \mathbf G_{n,N}-\mathbf I_n = \frac1N\sum_{i=1}^N Y_i.
\]
Let
\[
    B_n:=\sup_{x\in\mathcal M}\K_n^\star(x,x)
    =
    \sup_{x\in\mathcal M}\sum_{j=1}^n|\phi_j^\star(x)|^2.
\]
Since $\mu$ is a probability measure, $ B_n
    \ge
    \int_{\mathcal M}\K_n^\star(x,x)\,d\mu(x)
    =n\ge1.$
Moreover, $\|z(x_i)\|_2^2 = \K_n^{\star}(x_i,x_i) \le B_n$, so
\[
    \|Y_i\|_{\mathrm{op}}
    \le \max\{B_n-1,1\}
    \le B_n.
\]
Using
$(z(x_i)z(x_i)^\top)^2
=\|z(x_i)\|_2^2z(x_i)z(x_i)^\top$,
we also obtain
   \[ \mathbb E[Y_i^2]=
    \mathbb E\bigl[
        \|z(x_i)\|_2^2z(x_i)z(x_i)^\top
    \bigr]-\mathbf I_n
    \preceq
    B_n\,\mathbb E[z(x_i)z(x_i)^\top]-\mathbf I_n
    \preceq B_n\mathbf I_n.\]
Consequently, $ \|
        \sum_{i=1}^N\mathbb E[Y_i^2]
    \|_{\mathrm{op}}
    \le NB_n$.
Applying the matrix Bernstein inequality
\cite[Theorem~1.4]{tropp2012user} yields
\[
    \mathbb P\Big(
        \big\|\sum_{i=1}^N Y_i\big\|_{\mathrm{op}}
        \ge Ns
    \Big)
    \le
    2n\exp\Big(
        -\frac{N^2s^2}{2NB_n+\frac23B_nNs}
    \Big), \qquad \forall s\ge 0.
\]
By Lemma \ref{lm:continuum-kernel-local-weyl}, $B_n \le C_0 n$ for some structural constant $C_0\ge 1$. This proves \eqref{eq:empirical-gram-tail}.

For $\delta \in (0,1)$, let
\[
    a:=\frac{C_0n\log(2n/\delta)}{N},
    \qquad
    s_\delta:=\sqrt{2a}+\frac{2a}{3}.
\]
A direct calculation gives
$s_\delta^2-2a(1+s_\delta/3)
    =
    \frac{2a}{3}\sqrt{2a}\ge0$.
Hence
\[
    \frac{Ns_\delta^2}{2C_0n(1+s_\delta/3)}\ge \log(2n/\delta),
\]
and \eqref{eq:empirical-gram-tail} implies
\[
    \mathbb P(\eta_{n,N}\ge s_\delta)
    \le 2ne^{-\log(2n/\delta)}=\delta.
\]
This proves \eqref{eq:empirical-gram-confidence}.

Finally, taking $\delta=N^{-r}$ and using $n\le N$ and $N\ge2$,
we have
\[
    \log(2n/\delta)
    =
    \log(2nN^r)
    \le (r+2)\log N.
\]
Substitution into \eqref{eq:empirical-gram-confidence} proves
\eqref{eq:empirical-gram-polynomial}.
\end{proof}


Let $ R_j:=
    \L_N(\phi_j^{\star}|_{\mathcal X})
    -\lambda_j^{\star}\phi_j^{\star}|_{\mathcal X}$ for $1\le j \le N$. We define
\begin{equation}
   \mathsf R_{n,N}:=
    \Big(
        \frac1n\sum_{j=1}^n\|R_j\|_{L^2(\mu_N)}^2
    \Big)^{1/2}.    
\end{equation}

\begin{lemma}
\label{lm:averaged-residual-calder}
We set
\[
    h:=
    \begin{cases}
        \varepsilon,
            &\text{for the $\varepsilon$-graph},\\
        (k/N)^{1/m},
            &\text{for the symmetric $k$-NN graph}.
    \end{cases}
\]
There exist structural constants $C,c>0$, and
$h_{\mathrm{geom}}\in(0,1]$, such that the following
holds. 
Let $\delta\in(0,1)$. If $h \le h_{\mathrm{geom}}$ and
$\log (CNn/\delta) \le cN h ^{m}$,
then with probability at least $1-\delta$,
\[
    \mathsf R_{n,N}
       \lesssim n^{1/2+3/m}
          \Big( h+ \sqrt{\frac{\log (Nn/\delta)}{N h ^{m+2}}} \Big).
\]
\end{lemma}

\begin{proof}
By \cite[Theorems~3.3 and~3.7]{calder2022improved},
there exist structural constants $C \ge 1,c>0, a\ge1$, and $0<h_{\mathrm{geom}} \le \min\{1,a^{-1/2}\}$ such that
\begin{equation}
    \mathbb P\Big (
       \max_{1\le i\le N}
       \left|
          \L_N(f|_{\mathcal X})(x_i)-\L_\star f(x_i)
       \right|
       > C\|f\|_{C^3(\mathcal M)}t \Big )\le
       C N\exp\left(-cNh^{m+2}t^2\right),
\label{eq:calder-scalar-consistency}
\end{equation}
for every $f\in C^3(\mathcal M)$, provided that $a h\le t\le h^{-1}$ and $h \le h_{\mathrm{geom}}$. Here
 \[ \|f\|_{C^3(\mathcal M)}
       :=\sum_{q=0}^3
          \sup_{x\in\mathcal M}|\nabla_g^qf(x)|.\]
Since $\L_\star\phi_j^\star=\lambda_j^\star\phi_j^\star$, applying \eqref{eq:calder-scalar-consistency} to
$f=\phi_j^\star$ for $1\le j\le n$ and
a union bound gives
\[\mathbb P\Big (
       \exists j\le n:
       \max_{i\le N}|R_j(x_i)|
        > C \|\phi_j^\star\|_{C^3(\mathcal M)}t
    \Big )\le 
       C Nn\exp\left(-cN h ^{m+2}t^2\right)\]
for $ah \le t \le h^{-1}$ and $h \le h_{\mathrm{geom}}$. 
Let $\delta \in (0,1)$. Set
\[   t(\delta):=
       \max\Big\{
          ah,\,
          \sqrt{\frac{\log (CNn/\delta)}{cN h ^{m+2}}} \Big\}. \]
Then
\[ C Nn\exp\left(-cN h ^{m+2}t(\delta)^2\right) \le CNn\exp (-\log (CNn/\delta)) = \delta. \]
If $\log (CNn/\delta) \le cN h ^{m}$, then $ah \le t(\delta) \le h^{-1}$. Hence
\[\mathbb P\Big (
       \exists j\le n:
       \max_{i\le N}|R_j(x_i)|
        > C \|\phi_j^\star\|_{C^3(\mathcal M)}t(\delta)
    \Big ) \le \delta.\]
On the complementary event, we have
\[
    \mathsf R_{n,N}^2 =
       \frac1n\sum_{j=1}^n
       \frac1N\sum_{i=1}^N|R_j(x_i)|^2 \le
       \frac{C^2t(\delta)^2}{n}
       \sum_{j=1}^n
          \|\phi_j^\star\|_{C^3(\mathcal M)}^2.
\]
Finally, note that
$t(\delta) \lesssim h + \sqrt{\frac{\log (Nn/\delta)}{N h ^{m+2}}}$,
which implies 
\[
    \mathsf R_{n,N}
       \lesssim 
          \Big( h+ \sqrt{\frac{\log (Nn/\delta)}{N h ^{m+2}}} \Big) \Big (\frac{1}{n}\sum_{j=1}^n
          \|\phi_j^\star\|_{C^3(\mathcal M)}^2\Big )^{1/2}.
\]

It remains to show that $\frac{1}{n}\sum_{j=1}^n
          \|\phi_j^\star\|_{C^3(\mathcal M)}^2 \lesssim n^{1+6/m}$. 
Since $\L_\star$ is smooth, uniformly elliptic,
and of order two, the differentiated local Weyl law gives
 (see \cite[Theorem~2.4 and Remark~2.5]{rivera2022weighted})
\[
    \sup_{x\in\mathcal M}
    \sum_{\lambda_\nu^\star\le\Lambda}
       |\nabla_g^q\phi_\nu^\star(x)|^2
       \le C_q(1+\Lambda)^{m/2+q},
    \qquad \Lambda\ge0, \text{ for $q=0,1,2,3$}.
\]
Taking $\Lambda=\lambda_j^\star$ gives
\[
    \sup_x|\nabla_g^q\phi_j^\star(x)|^2
       \le C_q(1+\lambda_j^\star)^{m/2+q}.
\]
By the definition of the $C^3$ norm and Cauchy--Schwarz,
\[
    \|\phi_j^\star\|_{C^3(\mathcal M)}^2
       \le4\sum_{q=0}^3
              \sup_x|\nabla_g^q\phi_j^\star(x)|^2
       \le C(1+\lambda_j^\star)^{m/2+3}.
\]
The global Weyl law implies
$\lambda_j^\star\le Cj^{2/m}$ for all $j\ge1$,
after increasing $C$ to cover the finitely many small
indices. Hence $\|\phi_j^\star\|_{C^3(\mathcal M)}^2\le Cj^{1+6/m}$,
and therefore
\[
\frac1n\sum_{j=1}^n
          \|\phi_j^\star\|_{C^3(\mathcal M)}^2
       \le\frac Cn\sum_{j=1}^n j^{1+6/m}
       \le Cn^{1+6/m}.
\]
\end{proof}


\begin{lemma}
\label{lm:sqrt-kernel-comparison}
Fix $1\le n<N$ and an integer $0\le b<N-n$. Suppose that
\[
    t_{n,N,b}:=\lambda_{n+b+1,N}-\lambda_n^{\star}>0.
\]
Then
\begin{eqnarray}
    \big\|
        \sqrt{\K_{n,N}|_{\mathrm{diag}}}
        -
        \sqrt{\K_n^{\star}|_{\mathrm{diag}}}
    \big\|_{L^2(\mu_N)}
    \le 3\sqrt{n}\Big (
    \frac{\mathsf R_{n,N}}{t_{n,N,b}}
    +\eta_{n,N}
    +\sqrt{\frac{b}{n}}\Big ).
    \label{eq:sqrt-kernel-buffer-bound}
\end{eqnarray}
Here $\K_n^{\star}|_{\mathrm{diag}}(x):=\K_n^{\star}(x,x)$, and similarly for
$\K_{n,N}$.
\end{lemma}

\begin{proof}
For simplicity, we write $ u_j:=\phi_{j,N}$ and $f_j:=\phi_j^{\star}|_{\mathcal X}$ for $1\le j\le N$.
Then 
\[ \K_{n,N}(x,x) = \sum_{j=1}^n u_j(x)^2, \qquad \K_n^{\star}(x,x) = \sum_{j=1}^n f_j(x)^2, \qquad \forall x\in \mathcal{X}.\]
Set $H_N:=L^2(\mu_N)$, and let $e_1,\ldots,e_n$ be the standard
orthonormal basis of $\mathbb R^n$. Define linear maps $U,F:\mathbb R^n\longrightarrow H_N$
by
\[
   Ue_j:=u_j,
    \qquad
 Fe_j:=f_j.
\]
Thus, $U^*U=\mathbf I_n$ and $F^* F=[\langle f_i, f_j\rangle_{H_N}]_{1\le i,j \le n} = \mathbf G_{n,N}$, where the adjoints use the Euclidean inner product on $\mathbb R^n$
and the empirical inner product on $H_N$. We recall that for a linear map $T:\mathbb R^n\to H_N$, its
Hilbert--Schmidt norm is given by
\begin{equation*}
    \|T\|_{\mathrm{HS}}^2
    =
    \sum_{j=1}^n\|T e_j\|_{H_N}^2
    =
    \frac1N\sum_{i=1}^N\sum_{j=1}^n
    |( T e_j)(x_i)|^2.
\end{equation*}

Let $V:\mathbb R^n\to H_N$ be such that
   $V^*V=\mathbf I_n$ and $F=V\mathbf G_{n,N}^{1/2}$.
Let $\gamma_1,\ldots,\gamma_n\ge 0$ be eigenvalues of $\mathbf{G}_{n,N}$. Since the Hilbert--Schmidt norm is invariant under changes
of orthonormal basis, we have
\begin{eqnarray*}
    \|F-V\|_{\mathrm{HS}}^2
    =
    \sum_{j=1}^n(\sqrt{\gamma_j}-1)^2
    \le
    \sum_{j=1}^n(\gamma_j-1)^2
    \le
    n\eta_{n,N}^2.
\end{eqnarray*}
Therefore
\begin{equation}
    \|F-V\|_{\mathrm{HS}}
    \le
    \sqrt n\,\eta_{n,N}.
    \label{eq:sqrt-isometry-error}
\end{equation}

Let $P:= UU^*$ be the orthogonal projection onto
$\operatorname{span}\{u_1,\ldots,u_n\}$. For $b\ge 1$, we define $Q^{\mathrm{near}}$
as the orthogonal projection onto
$\operatorname{span}\{u_{n+1},\ldots,u_{n+b}\}$. When $b=0$, we set $Q^{\mathrm{near}}:=0$. 
Let $I_{H_N}$ denote the identity operator on $H_N$, and $Q^{\mathrm{far}}:= I_{H_N}-P-Q^{\mathrm{near}}$.
Because the vectors $Ve_j$ are orthonormal,
Bessel's inequality gives
\begin{eqnarray}
    \|Q^{\mathrm{near}} V\|_{\mathrm{HS}}^2
    =
    \sum_{k=n+1}^{n+b}\sum_{j=1}^n
    |\langle Ve_j,u_k\rangle_{H_N}|^2
    \le
    \sum_{k=n+1}^{n+b}\|u_k\|_{H_N}^2
    =
    b.
    \label{eq:sqrt-near-bound}
\end{eqnarray}
For $j\le n$ and $k\ge n+b+1$, one has
\begin{eqnarray*}
    \langle R_j,u_k\rangle_{H_N}=
    \langle\L_Nf_j-\lambda_j^{\star}f_j,u_k\rangle_{H_N}
    =
    (\lambda_{k,N}-\lambda_j^{\star})
    \langle f_j,u_k\rangle_{H_N}.
\end{eqnarray*}
Since the eigenvalues are ordered, $\lambda_{k,N}-\lambda_j^{\star}\ge\lambda_{n+b+1,N}-\lambda_n^{\star}=t_{n,N,b}$.
Consequently,
\begin{eqnarray*}
    \|Q^{\mathrm{far}}f_j\|_{H_N}^2=
    \sum_{k=n+b+1}^N
    |\langle f_j,u_k\rangle_{H_N}|^2
    \le
    \frac1{t_{n,N,b}^2}
    \sum_{k=n+b+1}^N
    |\langle R_j,u_k\rangle_{H_N}|^2
    \le
    \frac{\|R_j\|_{H_N}^2}{t_{n,N,b}^2}.
\end{eqnarray*}
Summing over $j$ yields
\begin{equation*}
    \| Q^{\mathrm{far}} F\|_{\mathrm{HS}}^2 
    = \sum_{j=1}^n \| Q^{\mathrm{far}} F e_j\|_{H_N}^2 
    = \sum_{j=1}^n \| Q^{\mathrm{far}} f_j\|_{H_N}^2 
    \le
    n\frac{ \mathsf R_{n,N}^2}{t_{n,N,b}^2}.
\end{equation*}
Thus,
\begin{eqnarray*}
    \| Q^{\mathrm{far}} V\|_{\mathrm{HS}}
    \le
    \| Q^{\mathrm{far}} F\|_{\mathrm{HS}}
    +
    \| V-F\|_{\mathrm{HS}}
    \le
   \sqrt n\Big ( \frac{\mathsf R_{n,N}}{t_{n,N,b}}
    +
   \eta_{n,N} \Big ).
\end{eqnarray*}
Combining this with \eqref{eq:sqrt-near-bound}, we obtain
\begin{equation}
    \|(I_{H_N}- P) V\|_{\mathrm{HS}} \le  \|Q^{\mathrm{near}} V\|_{\mathrm{HS}} +  \|Q^{\mathrm{far}} V\|_{\mathrm{HS}}
    \le
    \sqrt b
    +
    \frac{\sqrt n\,\mathsf R_{n,N}}{t_{n,N,b}}
    +
    \sqrt n\,\eta_{n,N}.
    \label{eq:sqrt-leakage-bound}
\end{equation}

Take a singular value decomposition of $U^* V\in \mathbb{R}^{n\times n}$
\[
    U^*V
    =
    \mathbf A\boldsymbol\Sigma\mathbf B^{\top},
    \qquad
    \boldsymbol\Sigma
    =
    \operatorname{diag}(\sigma_1,\ldots,\sigma_n).
\]
Since both maps $U$ and $V$ are isometries, $0\le\sigma_j\le1$. Set $\mathbf C:=\mathbf A\mathbf B^{\top}$. A direct expansion gives
\begin{eqnarray}
    \|U\mathbf C- V\|_{\mathrm{HS}}^2
    =
    2n
    -
    2\operatorname{tr}
    \left(
        \mathbf C^{\top} U^* V
    \right) =
    2\sum_{j=1}^n(1-\sigma_j).
    \label{eq:sqrt-alignment-identity}
\end{eqnarray}
On the other hand,
\begin{eqnarray*}
    \|(I_{H_N}- P) V\|_{\mathrm{HS}}^2
    =\|(I_{H_N}- UU^*) V\|_{\mathrm{HS}}^2=
    n-\|U^* V\|_{\mathrm F}^2
    =
    \sum_{j=1}^n(1-\sigma_j^2),
\end{eqnarray*}
where $\|\cdot\|_{\mathrm F}$ denotes the Frobenius norm.
Since $1-\sigma_j\le1-\sigma_j^2$,
\eqref{eq:sqrt-alignment-identity} implies
\[
    \| U \mathbf C- V\|_{\mathrm{HS}}
    \le
    \sqrt2\,
    \|(I_{H_N}- P) V\|_{\mathrm{HS}}.
\]
Using \eqref{eq:sqrt-isometry-error}
and \eqref{eq:sqrt-leakage-bound}, we conclude that
\begin{eqnarray}
    \| U \mathbf C-F\|_{\mathrm{HS}}
    \le
    \| U \mathbf C- V\|_{\mathrm{HS}}
    +
    \| V- F\|_{\mathrm{HS}}
    \le 3\sqrt n \Big (
    \sqrt{\frac{b}{n}}
    +
    \frac{\mathsf R_{n,N}}{t_{n,N,b}}
    +
    \eta_{n,N} \Big ).
    \label{eq:sqrt-frame-bound}
\end{eqnarray}

Finally, write $\mathbf{C}=[c_{jk}]_{1\le j,k \le n}$ and set $w_j:= U\mathbf{C}e_j = \sum_{k=1}^n c_{kj}u_k$ for $1\le j \le n$. Since $\mathbf C$ is orthogonal, we have  
\begin{eqnarray*}
     \sum_{j=1}^n w_j(x)^2 =  \sum_{j=1}^n\Big (\sum_{k=1}^n c_{kj}u_k(x)\Big )^2 = \sum_{k=1}^n u_k(x)^2 = \K_{n,N}(x,x), \qquad \forall x\in \mathcal{X}.
\end{eqnarray*}
Thus, 
\begin{eqnarray*}
    \Big|
        \sqrt{\K_{n,N}(x,x)}
        -
        \sqrt{\K_n^{\star}(x,x)}
    \Big|
    &=&
    \Big|
        \Big (
            \sum_{j=1}^n|w_j(x)|^2
        \Big)^{1/2}
        -
        \Big(
            \sum_{j=1}^n|f_j(x)|^2
        \Big)^{1/2}
    \Big|
    \\
    &\le&
    \Big(
        \sum_{j=1}^n|w_j(x)-f_j(x)|^2
    \Big)^{1/2}.
\end{eqnarray*}
Squaring and integrating against $\mu_N$, we obtain
\begin{eqnarray*}
    \left\|
        \sqrt{\K_{n,N}|_{\mathrm{diag}}}
        -
        \sqrt{\K_n^{\star}|_{\mathrm{diag}}}
    \right\|_{L^2(\mu_N)}^2
    \le
    \sum_{j=1}^n\|w_j-f_j\|_{H_N}^2=
    \| U\mathbf C- F\|_{\mathrm{HS}}^2.
\end{eqnarray*}
Combining this with \eqref{eq:sqrt-frame-bound} proves
\eqref{eq:sqrt-kernel-buffer-bound}.
\end{proof}

\subsubsection{Proofs of Theorems \ref{thm:eps-minibatch} and \ref{thm:knn-minibatch}}

\begin{proof}
The argument is similar for both graph constructions. We first prove the result for the case of $\eps$-graphs.

Fix $r>0$ and $\kappa>0$. The choices in Theorem \ref{thm:eps-minibatch} give
\begin{equation*}
 \frac{\log N}{N}=\eps_N^{2m+2},\qquad
 n\le\eps_N^{-2m/(m+5)}\le\eps_N^{-m/2},\qquad
 \eps_N n^{(m+5)/(2m)}\le1
 \quad(n\in I_N),
\end{equation*}
for sufficiently large $N$. 
We choose $n_0, A_r$ sufficiently large, and set
$\widetilde\delta_N=\eps_N^2$ and $\theta_N=A_r\eps_N$ as in
Corollary \ref{cor:epsilon-concrete-scaling}. These parameters
satisfy the hypotheses of Theorem \ref{thm:var-reduce-eps} for large $N$, since
\[
 \frac{\widetilde\delta_N}{N^{-1/m}}
 =N^{1/(m(m+1))}(\log N)^{1/(m+1)}\to\infty,
 \qquad \frac{\widetilde\delta_N}{\eps_N}=\eps_N\to 0,
 \qquad \theta_N+\widetilde\delta_N\to 0.
\]
In particular, the variance estimate in Corollary \ref{cor:epsilon-concrete-scaling} holds with probability at least $1-C_rN^{-r}$.

By a union bound and Lemma \ref{lm:empirical-gram-concentration}, we have the following with probability at least $1-N^{-r}$
\[ \eta_{n,N} \lesssim_r \sqrt{\frac{n \log N}{N}} + \frac{n \log N}{N} \lesssim_r n^{-1/(2m)},\qquad \forall n\in I_N.\]

Note that for $h=\eps_N$, the condition in Lemma \ref{lm:averaged-residual-calder} is satisfied uniformly for large $N$.
Applying Lemma \ref{lm:averaged-residual-calder} with $\delta = N^{-r-1}$ and a union bound, we have the following with probability at least $1-N^{-r}$
\[ \mathsf R_{n,N} \lesssim_r n^{1/2+3/m} \Big (\eps_N + \sqrt{\frac{\log N}{N \eps_N^{m+2}}}\Big )\lesssim_r  n^{1/2+3/m} \Big (\frac{\log N}{N}\Big )^{1/(2m+2)}, \qquad \forall n \in I_N.\]

We remark that the error function
from Lemma \ref{thm:loc-spec-eps} then satisfies
\begin{equation}\label{eq:countingerror}
 E_N(\lambda)\lesssim_r\eps_N(1+\sqrt\lambda), \qquad \lambda \ge 0.
\end{equation}
By Lemma \ref{thm:loc-spec-eps}, with probability at least $1-C_rN^{-r}$, we have
\[ n_{\L_\star}\big (\lambda (1-4E_N(\lambda))\big ) \le n_{\L_N^{\eps}}(\lambda) \le n_{\L_\star}\big (\lambda (1+4E_N(\lambda))\big )\]
for all $\lambda>0$ such that $E_N(\lambda) \le 1/8$.

Intersect the preceding events with the diagonal lower-bound event in Theorem $\ref{thm:eps-minibatch}$. By a union bound, this intersection has probability at least $1-C_rN^{-r}-p_{N,\kappa}$. We work throughout on this event, on which the following estimates hold simultaneously for every $n\in I_N$.

We now apply Lemma \ref{lm:sqrt-kernel-comparison}. The Weyl estimates used in the proof
of Lemma \ref{lm:continuum-kernel-local-weyl} give
\begin{equation*}
 \lambda_n^\star\asymp n^{2/m},\qquad
 n=A_\star(\lambda_n^\star)^{m/2}+O(n^{1-1/m}),\qquad
 n_{\L_\star}(3\lambda_n^\star)\lesssim n.
\end{equation*}
Using $n \le \eps_N^{-m/2}$ and \eqref{eq:countingerror}, we have
\begin{equation*}
 E_N(2\lambda_n^\star)
 \lesssim_r\eps_N(1+n^{1/m})
 \lesssim_r n^{-1/m}.
\end{equation*}
Moreover, since $n^{1/m} \le \eps_N^{-1/2}$, we have $E_N(2\lambda_n^\star) \lesssim_r \eps_N +\sqrt{\eps_N} \to 0$ as $N \to \infty$. Thus, for $N$ large enough, $E_N(2\lambda_n^\star) \le 1/8$ for every $n \in I_N$.

Fix $n\in I_N$ and choose $\lambda: =(1+n^{-1/m})\lambda_n^\star$.
Since $\lambda\le2\lambda_n^\star$, $E_N(\lambda)\le 1/8$. Lemma \ref{thm:loc-spec-eps} yields
\begin{equation*}
 n_{\L_N^{\eps_N}}(\lambda)
 \le n_{\L_\star}\bigl(\lambda(1+4E_N(\lambda))\bigr)
 \le n_{\L_\star}(3\lambda_n^\star)\lesssim n.
\end{equation*}
Let 
$ b:=\max\{n_{\L_N^{\eps_N}}(\lambda)-n,0\}$.
Because $n/N\to0$ uniformly for $n\in I_N$, it follows that
$n+b=\max\{n,n_{\L_N^{\eps_N}}(\lambda)\}<N$
for large $N$. 
Furthermore, the definition of $b$ gives
$\lambda_{n+b+1,N}>\lambda$. Hence
\begin{equation*}
 t_{n,N,b}=\lambda_{n+b+1,N}-\lambda_n^\star> \lambda -\lambda_n^{\star}
 =n^{-1/m}\lambda_n^\star\gtrsim n^{1/m}.
\end{equation*}
Therefore,
\begin{equation*}
 \frac{\mathsf R_{n,N}}{t_{n,N,b}}
 \lesssim_r\eps_N n^{1/2+2/m}
 =n^{-1/(2m)}\bigl(\eps_Nn^{(m+5)/(2m)}\bigr)
 \lesssim_r n^{-1/(2m)}.
\end{equation*}
On the other hand, by subtracting the leading Weyl terms and using $E_N(\lambda) \lesssim_r n^{-1/m}$, we have
\[
b\le n_{\L_\star}\bigl(\lambda(1+4E_N(\lambda))\bigr)-n
 \lesssim_r n^{1-1/m}.
\]

Combining all ingredients, Lemma \ref{lm:sqrt-kernel-comparison} implies
\begin{equation}\label{eq:rootcomparison}
\left\|\sqrt{\K_{n,N}^{\eps_N}|_{\diag}}
           -\sqrt{\K_n^\star|_{\diag}}\right\|_{L^2(\mu_N)}
\le3\sqrt n\left(
       \frac{\mathsf R_{n,N}}{t_{n,N,b}}+\eta_{n,N}+\sqrt{\frac bn}\right)
 \lesssim_r n^{1/2-1/(2m)}.
\end{equation}
On the other hand,
by Lemma \ref{lm:continuum-kernel-local-weyl},
\[
 \left\|\sqrt{\K_n^\star|_{\diag}}-\sqrt{nw_\star}
       \right\|_{L^\infty(\mathcal M)}
 =\left\|\frac{\K_n^\star|_{\diag}-nw_\star}
                 {\sqrt{\K_n^\star|_{\diag}}+\sqrt{nw_\star}}
       \right\|_{L^\infty(\mathcal M)}
 \lesssim n^{1/2-1/m}.
\]
Combining this with \eqref{eq:rootcomparison} yields
\begin{equation}\label{eq:smoothdiagonal}
 \left\|\sqrt{\K_{n,N}^{\eps_N}|_{\diag}}
             -\sqrt{nw_\star}\right\|_{L^2(\mu_N)}^2
 \lesssim_r n^{1-1/m}.
\end{equation}

We now turn to the estimator. Since $\phi_{1,N}\equiv1$,
its denominator is positive on every realization. The inclusion
probabilities give
\[
 \E[\mathcal L_{n,N}^{\eps_N}(f)\mid\mathcal X]
 =\frac1N\sum_{i=1}^N
 \frac{f(x_i)}{\K_{n,N}^{\eps_N}(x_i,x_i)}
                    \K_{n,N}^{\eps_N}(x_i,x_i)
 =\mathcal L_N(f).
\]
Observe that
\[
 \mathcal L_{n,N}^{\eps_N}(f)
 =\frac1n\mathcal{S}_{n,N}^{\eps_N}(f/w_\star)
 +\mathcal{S}_{n,N}^{\eps_N}
 \Big (
 f\Big [\frac1{\K_{n,N}^{\eps_N}|_{\diag}}
                         -\frac1{nw_\star}\Big]\Big).
\]
We then have
\begin{eqnarray} \label{eq:variancesplit}
 \E[|\mathcal L_{n,N}^{\eps_N}(f)-\L_N(f)|^2\mid\mathcal X]
 &\le&\frac2{n^2}\Var[\mathcal{S}_{n,N}^{\eps_N}(f/w_\star)\mid\mathcal X] \notag \\
 &+& 2\Var\Big[\mathcal{S}_{n,N}^{\eps_N}
 \Big (
 f\Big [\frac1{\K_{n,N}^{\eps_N}|_{\diag}}
                         -\frac1{nw_\star}\Big]\Big) \Big | \mathcal{X}\Big].
\end{eqnarray}

Since $w_{\star}$ is strictly positive and smooth, we have
\[
 \|f/w_\star\|_{L^\infty(\X)}^2+\Lip(f/w_\star)^2
 \lesssim\|f\|_{L^\infty(\X)}^2+\Lip(f)^2.
\]
Corollary \ref{cor:epsilon-concrete-scaling} therefore gives
\begin{equation}\label{eq:mainvariance}
 \frac1{n^2}\Var[\mathcal{S}_{n,N}^{\eps_N}(f/w_\star)\mid\X]
 \lesssim_r\bigl(\|f\|_{L^\infty(\X)}^2+\Lip(f)^2\bigr)
                     \cdot n^{-1-1/m}.
\end{equation}

On the other hand, for every real function $g$ on $\X$, the projection-DPP variance
formula gives
\[
 \Var[\mathcal{S}_{n,N}^{\eps_N}(g)\mid\X]
 \le\int_{\X}g(x)^2\K_{n,N}^{\eps_N}(x,x)\d\mu_N(x).
\]
On the event under consideration,
$\K_{n,N}^{\eps_N}(x,x)\ge\kappa n$ for all $x\in\X$. Thus
\[
 \K_{n,N}^{\eps_N}|_{\diag}
       \Big(\frac1{\K_{n,N}^{\eps_N}|_{\diag}}-\frac1{nw_\star}\Big)^2\lesssim_\kappa\frac1{n^2}
 \left(\sqrt{\K_{n,N}^{\eps_N}|_{\diag}}-\sqrt{nw_\star}\right)^2.
\]
Consequently, by \eqref{eq:smoothdiagonal},
\begin{equation}\label{eq:errorvariance}
 \begin{split}
 \Var\Big[\mathcal{S}_{n,N}^{\eps_N}\Big(
 f\Big[\frac1{\K_{n,N}^{\eps_N}|_{\diag}}
                         -\frac1{nw_\star}\Big]\Big)\mid\X\Big]
 &\quad\lesssim_\kappa\frac{\|f\|_{L^\infty(\X)}^2}{n^2}
 \left\|\sqrt{\K_{n,N}^{\eps_N}|_{\diag}}-\sqrt{nw_\star}
       \right\|_{L^2(\mu_N)}^2\\
 &\quad\lesssim_{r,\kappa}
                    \|f\|_{L^\infty(\X)}^2n^{-1-1/m}.
 \end{split}
\end{equation}
Substituting \eqref{eq:mainvariance} and \eqref{eq:errorvariance} into
\eqref{eq:variancesplit} gives
\[
 \E[|\mathcal L_{n,N}^{\eps_N}(f)-\mathcal L_N(f)|^2\mid\X]
 \lesssim_{r,\kappa}
 \bigl(\|f\|_{L^\infty(\X)}^2+\Lip(f)^2\bigr)
 \cdot n^{-1-1/m}.
\]
Theorem \ref{thm:eps-minibatch} follows.

For $k$-NN graphs, take $\widetilde\delta_N=h_N^2$ and
$\theta_N=A_rh_N$, and use Corollary \ref{cor:knn-concrete-scaling} and Lemma \ref{thm:loc-spec-knn} in place
of Corollary \ref{cor:epsilon-concrete-scaling} and Lemma \ref{thm:loc-spec-eps}. Since
$Nh_N^{2m+2}\ge\log N$, the same probability and residual
estimates hold with $h_N$ replacing $\eps_N$. Moreover,
\[
n\le h_N^{-2m/(m+5)}\le h_N^{-m/2},
\qquad
h_Nn^{(m+5)/(2m)}\le1.
\]
The preceding kernel comparison and variance estimates therefore
apply unchanged. Intersecting with the diagonal lower-bound event
gives the claimed conclusion with probability at least
$1-C_rN^{-r}-p_{N,\kappa}$.
\end{proof}

\section{Proofs from Subsection \ref{sec:large-network}}\label{proof:network}
In this section, we prove Theorem \ref{thm:var-reduce-network}. In what follows, we assume the graph $G=(\mathcal{X},E)$ satisfies the assumptions in Theorem \ref{thm:var-reduce-network}.
For $\alpha>0$, we recall that
     \[\mathcal R_{\L_R}^{(\alpha)}(n):= \max\Big \{n_{\L_R}(1),\max_{0 \le p \le \lfloor\lambda_{R,n}^{\alpha}\rfloor-1} \big \{n_{\L_R}\big ((\lambda_{R,n}^{\alpha}-p)^{1/\alpha} \big ) - n_{\L_R}\big ((\lambda_{R,n}^{\alpha}-p-1)^{1/\alpha} \big )\big \} \Big \}.\]

\begin{lemma} \label{thm:loc-spec-infinite-graph}
    Let $(v_k)_{k\ge 1}$ be the nondecreasing sequence of values of $V$ on $\mathcal{X}$, repeated according
    to multiplicity.
    Then, for every $R\in \N$, one has
    \begin{equation} \label{eq:inf-graph-eigen}
        v_k \le \lambda_{R,k} \le v_k +2D, \qquad \forall\: 1\le k \le N_R.
    \end{equation} 
    Moreover, there exists $C>0$ independent of $R$ and $n$ for every $0<\alpha \le \min \big\{1,\frac{m}{\beta} \big\}$, one has
   \[ \mathcal{R}_{\L_R}^{(\alpha)}(n) \le \frac{C}{\alpha}\big (n^{1-\alpha\beta/m} + n^{1-\gamma/m} \big ), \qquad \forall 1\le n \le N_R.\]
\end{lemma}

\begin{proof}
    Let $A_R$ denote the adjacency operator of $G_R$. The uniform degree bound implies $
    \|A_R\|_{L^2(\mu_R)\to L^2(\mu_R)} \le D$. Therefore, $0\le \Delta_R \le 2DI$, and consequently 
    \[ \M_{V|_{B_R}} \le \L_R \le \M_{V|_{B_R}}  + 2 DI\]
    in the sense of quadratic forms. We note that the values of $V$ on $B_R$ are exactly the first $N_R$ terms of the sequence $(v_k)_{k\ge 1}$. By the min-max principle, we obtain \eqref{eq:inf-graph-eigen}.

    For the second assertion, let $\alpha \in (0,1]$.
     It suffices to prove the claim for large enough $n$, so that $\lambda_{R,n}>1$.
      Let $0 \le p \le \lfloor \lambda_{R,n}^{\alpha} \rfloor -1$, we set $u:= \lambda_{R,n}^{\alpha}-p$. 
       Note that for $0<\alpha \le 1$, $|a^{\alpha}-b^{\alpha}|\le |a-b|^{\alpha}$ for all $a,b \ge 0$. Thus 
    \[ 0 \le \lambda_{R,k}^{\alpha} - v_k^{\alpha} \le (2D)^{\alpha}, \qquad \forall 1\le k \le N_R.\]
    Hence, for $k$ such that $u-1 <\lambda_{R,k}^{\alpha} \le u$, one has $u-1-(2D)^{\alpha} < v_k^{\alpha} \le u$. 
    Let 
    \[ F_{\alpha}(u) := \# \{x\in \mathcal{X}: V(x)^{\alpha} \le u \}, \qquad u\ge 0\]
    with the convention that $F_{\alpha}(u) = 0$ for $u <0$. 
    Then we have
    \[ n_{\L_R}(u^{1/\alpha}) - n_{\L_R}((u-1)^{1/\alpha}) \le F_{\alpha}(u) - F_{\alpha}(u-1-(2D)^{\alpha}).\]
    The volume expansion implies
    \[ F_{\alpha}(u) = C_{\vol} u^{m/(\alpha\beta)} + O(u^{(m-\gamma)/(\alpha\beta)} ) \qquad \text{as } u \to \infty.\]
    For $\alpha \le m/\beta$, one has $m/(\alpha \beta)\ge 1$. 
    Applying the mean-value theorem gives
    \[ F_{\alpha}(u)- F_{\alpha}(u-1-(2D)^{\alpha}) \le \frac{C}{\alpha} \big ( \lambda_{R,n}^{m/\beta - \alpha} + \lambda_{R,n}^{(m-\gamma)/\beta}  \big ).\]
    
   Furthermore, one has $\lambda_{R,n} \le v_n + 2D$. 
   Enumerate the vertices $(x_k)$ of $\mathcal{X}$ in a way such that $r(x_k)$ is non-decreasing. Then $v_n = r(x_n)^{\beta}$ and 
   \[ |B_{r(x_n)-1}| < n \le |B_{r(x_n)}|.\]
   By the inversion of volume expansion, one has $r(x_n) \lesssim n^{1/m}$ and thus $v_n \lesssim n^{\beta/m}$. Thus,
    \[ n_{\L_R}(u^{1/\alpha}) - n_{\L_R}((u-1)^{1/\alpha}) \lesssim \frac{1}{\alpha} \big (n^{1-\alpha\beta/m} + n^{1-\gamma/m} \big ).\]
   By enlarging the implicit constant if necessary, for every $0<\alpha \le \min \big\{1,\frac{m}{\beta} \big\}$, one has
   \[ \mathcal{R}_{\L_R}^{(\alpha)}(n) \le \frac{C}{\alpha}\big (n^{1-\alpha\beta/m} + n^{1-\gamma/m} \big ), \qquad \forall 1\le n \le N_R.\]
   
\end{proof}

In the setting of graphs of bounded degree, we have stronger commutator estimates:
\begin{lemma} \label{lm:improve-commutator}
    For every $f:B_R \to \R$ and for every $0<\alpha\le 1$
    \[ \|[\L_R^{\alpha}, \M_f]\|_{L^2(\mu_R) \to L^2(\mu_R)} \le C_{\alpha,D} \|f\|_{\infty},\]
    where $C_{\alpha,D}$ is independent of $R$ and $f$.
\end{lemma}

\begin{proof}
    We first consider $\alpha =1$. Since multiplication operators commute, one has $[\L_R, \M_f] = \M_f A_R -A_R\M_f$. The bound $\|A_R\|_{L^2(\mu_R)\to L^2(\mu_R)} \le D$ yields the claim for $\alpha =1$.

    For $0<\alpha<1$, we first note that by the Duhamel formula and $L^2$-contractivity of $e^{-t\L_R}$
    \[ \|[e^{-t\L_R},\M_f]\|_{L^2(\mu_R)\to L^2(\mu_R)}  \le \min \big \{2 \|f\|_{\infty} , t \|[\L_R, \M_f]\|_{L^2(\mu_R) \to L^2(\mu_R)} \big \}, \qquad \forall t>0.\]
    For $0<\alpha<1$, one has 
    \[[\L_R^{\alpha},\M_f] = - \frac{\alpha}{\gamma(1-\alpha)} \int_0^{\infty} [e^{-t\L_R},\M_f] \frac{\d t}{t^{1+\alpha}},\]
    where $\gamma$ denotes Euler’s gamma function. Substituting the above estimate yields
    \[  \|[\L_R^{\alpha}, \M_f]\|_{L^2(\mu_R) \to L^2(\mu_R)} \le C_{\alpha} \|f\|_{\infty}^{1-\alpha} \|[\L_R, \M_f]\|_{L^2(\mu_R) \to L^2(\mu_R)}^{\alpha}. \]
    The lemma follows.
\end{proof}

We conclude this section with the proof of Theorem \ref{thm:var-reduce-network}.
\begin{proof}[Proof of Theorem \ref{thm:var-reduce-network}]
    Since $\gamma \le 1$, one can choose $\alpha = \min \{1, \gamma/\beta \}$ in Lemma \ref{thm:loc-spec-infinite-graph} to get 
    \[ \mathcal{R}_{\L_R}^{(\alpha)}(n) \lesssim n^{1- \min \{\beta, \gamma\}/m}.\]
  Applying Theorem \ref{thm:var-reduce-general-alpha}, one has
    \[  \var[\mathcal{S}_{n,R}(f)] \le \big ( \|f\|_{\infty}^2 +\frac{\pi^2}{6}  \|[\L_R^{\alpha}, \M_f]\|_{L^2(\mu_R) \to L^2(\mu_R)}^2 \big ) \mathcal{R}_{\L_R}^{(\alpha)}(n). \]
    Applying Lemma \ref{lm:improve-commutator} and Lemma \ref{thm:loc-spec-infinite-graph} gives the theorem.
\end{proof}

\end{document}